\documentclass[12pt,a4]{article}
\usepackage{geometry}
\usepackage{amsmath}
\numberwithin{equation}{section}
\usepackage{amsthm}
\usepackage{amssymb}
\usepackage{graphicx}
\usepackage{hyperref}
\usepackage{zymacros}
\newcommand{\LogLn}{\left(\ln m\right)}
\providecommand{\keywords}[1]
{
  \small	
  \textbf{\textit{Keywords---}} #1
}

\begin{document}
\title{Towards an Understanding of Residual Networks Using Neural Tangent Hierarchy~(NTH)}
\author{Yuqing Li, Tao Luo, Nung Kwan Yip\\
Department of Mathematics, Purdue University, IN, 47907,
USA}

\date{\today}
\maketitle
 
\begin{abstract}
   
 Gradient descent yields zero training loss in polynomial time for deep neural networks despite   non-convex nature of the objective function. The behavior of  network  in the infinite width limit trained by gradient descent can be described by the Neural Tangent Kernel (NTK) introduced in~\cite {Jacot2018Neural}. In this paper, we study  dynamics of the NTK for finite width  Deep Residual Network~(ResNet) using  the neural tangent hierarchy~(NTH) proposed in~\cite{Huang2019Dynamics}. For a  ResNet with smooth and Lipschitz activation function, we reduce the requirement on the layer width $m$ with respect to the number of training samples $n$ from quartic to cubic. Our analysis suggests strongly that the particular skip-connection structure of  ResNet is the main reason for its triumph over fully-connected network.
\end{abstract}

\section{Introduction}\label{Section...Introduction}
Deep neural networks have achieved transcendent performance in a wide range of tasks such as speech recognition~\cite{Dahl2011context}, computer vision~\cite{Rastegari2016xnor}, and natural language processing \cite{Collobert2008unified}. There are various methods to train  neural networks, such as  first-order gradient
based methods like Gradient Descent~(\textbf{GD}) and Stochastic Gradient Descent~(\textbf{SGD}), which have been proven to achieve satisfactory results \cite{Goodfellow2016deep}. Experiments~in~\cite{Zhang2018understanding} established that, even though with a random labeling of the training images, if one trains the state-of-the-art convolutional network for image classification using SGD, the network is still  able to fit them well. There are numerous works trying to demystify such phenomenon theoretically. Du et al.  \cite{Du2018Gradient,Du2018Gradienta} proved that  GD can obtain zero training loss for deep and shallow neural networks, and Zou et al.~\cite{zou2018stochastic} analyzed the convergence of SGD   on  networks assembled with Rectified Linear Unit (\textbf{ReLU}) activation function.  All these results are built upon the overparameterized regime, and it is widely accepted that overparameterization enables the neural network to fit all training data and bring no harm to the power of its generalization \cite{Zhang2018understanding}.  In particular, the deep neural networks that evaluated positions and selected moves for the well-known program  AlphaGo are highly overparameterized~\cite{Silver2016mastering,Silver2017mastering}.

Another advance is the outstanding performance of Deep Residual Network~(\textbf{ResNet}), initially proposed by He et al.~\cite{He2016deep}. ResNet is arguably the most groundbreaking work in deep learning, in that it can train up to hundreds or even thousands of layers and still achieves compelling performance~\cite{Huang2017Densely}. Recent works have shown that ResNet can utilize the features in transfer learning with better efficiency, and its residual link structure enables faster convergence of the training loss \cite{Zagoruyko8wide,Szegedy2016inception}. Theoretically, Hardt and Ma  \cite{Hardt2016identity} proved that for any residual linear networks with arbitrary depth, there are no  spurious local optima. Du et al.~\cite{Du2018Gradient}  showed that in the scope of the convergence of GD  via overparameterization for different networks, training ResNet requires weaker conditions compared with fully-connected networks. Apart from that, the advantages of using residual connections remain to be discovered. 

In this paper, we contribute to the further understanding of the above two aspects and make improvements in the analysis of their performance. We use the same ResNet structure as in \cite{Du2018Gradient}. (Details of the network structure are provided in Section \ref{subsection....Problem Setup}.) The ResNet has $L$ layers with width $m.$ We will assume that the $n$ data points are not parallel with each other. Such an assumption holds in general for a standard dataset~\cite{Du2018Gradienta}. We focus on the empirical risk minimization problem given by the quadratic loss and   the activation function is $1$-Lipschitz and analytic. We show that if $m= \Omega\left( {n}^{3}L^2\right),$ then the empricial risk $R_S(\vtheta_t)$ under GD decays exponentially. More precisely, 
\begin{equation*} 
     R_S(\vtheta_t)\leq  R_S(\vtheta_0) \exp\left(-\frac{\lambda t}{n}\right),
\end{equation*}
where $\lambda$ is the least eigenvalue of $\mK^{[L+1]},$ definition of which can be found in \eqref{definition...definition of the th Gram Matrix L+1}.

It is worth noticing that
\begin{itemize}
    \item Given identical ResNet architectures, for the convergence of randomly initialized GD, our results improve upon~\cite{Du2018Gradient} in the required number of width per layer from $m=\Omega(n^4L^2)$ to $m=\Omega(n^3L^2)$~(Corollary \ref{corollary....for thm}).
    \item For fully-connected network,   the required amount of overparametrization in~\cite{Huang2019Dynamics} is $m=\Omega\left(n^32^{\fO(L)}\right).$  We are able to reproduce the result of Du~et~al.~\cite{Du2018Gradient},  showing that the exponential dependence of $m$ on the number of layers $L$ can be eliminated for ResNet.
\end{itemize}
 
Our work is mainly motivated by the framework proposed by Huang and  Yau~\cite{Huang2019Dynamics}, in which an infinite hierarchy of ordinary differential equations, the neural tangent hierarchy~(NTH) is derived. Huang and Yau applied NTH to a fully-connected feedforward network and showed that  it is possible for us to directly study the change of the neural tangent kernel~(NTK)~\cite{Jacot2018Neural},  and NTH outperforms kernel regressions using the corresponding limiting NTK.

Different from Huang and Yau's work in analyzing the fully-connected network,  ResNet is investigated in our paper. We exploit the benefits of using ResNet architecture for training and the advantage of choosing NTH over kernel regression. In Section \ref{section....Techinique}, an  of our technique is provided.

The organization of the paper is listed  as follows. In Section \ref{section....related wotks}, we  discuss some related works. In Section \ref{section....Preliminaries}, we give  some preliminary introductions to our problem. In Section \ref{section...main result}, we state our main results for ResNet using NTH. In Section \ref{section....Techinique}, we give out an outline of our approach. We give some conclusions and future direction in Section \ref{section....Discussion}. 
All the details of the proof are deferred to the Appendix.

\section{Related Works}\label{section....related wotks}
In this section, we survey some  previous works on aspects related to  optimization aspect of neural networks. 

Due to the non-convex nature of optimizing a neural network, it is challenging to locate the global optima. A popular way to analyze such optimization problems is to identify the geometric properties of each critical point. Some recent works have shown that   for the set of  functions satisfying: (i) all local minima are global and (ii)  every saddle point possesses a negative curvature~(i.e. it is non-degenerate), then GD can find  a global optima \cite{Du2017gradient,jin2017escape,Ge2015Escaping,lee2016gradient}. The objective functions of some shallow networks are in  such set \cite{Hardt2016identity,du2018power,nguyen2017loss,zhou2017critical}. The work~\cite{kawaguchi2016deep} indicates   that even for a three-layer linear network, there exists degenerate saddle points without negative curvature. So it is doubtful that  global convergence of first-order methods can be guaranteed for deep neural networks. 

Here we directly study the dynamics of the GD for a specific neural network architecture. This is another approach widely taken to obtain convergence results. Recently, it has been shown that if the network is over-parameterized, the SGD is able to find a global optima for two-layer networks~\cite{brutzkus2017globally,du2017gradientaa,ge2017learning,li2017convergence,ma2019comparative,Du2018Gradienta}, deep linear networks~\cite{arora2018convergence,Hardt2016identity,bartlett2019gradient} and ResNet~\cite{Du2018Gradient,Allen2018convergence}. Jacot et al. \cite{Jacot2018Neural} established that in the infinite width limit, the full batch GD corresponds to kernel regression predictor using the limiting NTK. Consequently, the convergence of GD for any `infinite-width' neural network can be characterized by a fixed kernel~\cite{Arora2019exact,Jacot2018Neural}.
This is the cornerstone upon which rests the compelling performance of over-parameterization . In the regime of finite width, many works have suggested that the network can  reduce training loss at exponential rate using GD~\cite{Du2018Gradient,Du2018Gradienta,Huang2019Dynamics,ma2019comparative,arora2018convergence}. As the width increases, there are going to be  small changes in the parameters during the whole training process~\cite{daniely2016toward,zou2018stochastic}. Such a variation of the parameters is crucial to the results we present, where the NTK of our ResNet behaves linearly in terms of its parameters throughout training~(Theorem~\ref{theorem...vary at 1/m}). Specifically, we use the results concerning the stability of the Gram matrices  in~\cite{Du2018Gradient} to demonstrate the benefits of choosing ResNet over fully-connected networks~(Proposition \ref{proposition......eigenvalue of order higher}).

Some other works used optimal transport theory to analyze the mean field SGD dynamics of training neural networks in the large-width limit~\cite{sirignano2018mean,rotskoff2018parameters,chizat2018global,mei2018mean}. However, their results are limited to one hidden layer networks, and their normalization factor $1/m$ is different from our $1/\sqrt{m}$ which is commonly employed in modern networks~\cite{He2016deep,Glorot2010Understanding}.

\section{Preliminaries}\label{section....Preliminaries}
\subsection{Notations}\label{subsection...Notations}
We begin this section by introducing some notations that will be used in the rest of this paper.  We set $n$ for the number of input samples and $m$ for the width of the neural network, and a special vector $(1,1,1,\dots,1)^\T\in\sR^m$ by $\vone:=(1,1,1,\dots,1)^\T.$ We denote vector $L^2$ norm as $\Norm{\cdot}_2$, vector or function $L_{\infty}$ norm as $\Norm{\cdot}_{\infty}$, matrix spectral~(operator) norm as $\Norm{\cdot}_{2\to 2}$,  matrix Frobenius norm as $\Norm{\cdot}_{\mathrm{F}},$ matrix infinity norm as $\Norm{\cdot}_{\infty\to \infty},$ and a special matrix norm, matrix $2$ to infinity norm as $\Norm{\cdot}_{2\to \infty},$ which was shown to be useful in \cite{Du2018Gradienta}. For a semi-positive-definite  matrix  $\mA,$ we denote its smallest eigenvalue by $\lambda_{\min}(\mA).$ We use $\fO(\cdot)$ and $\Omega(\cdot)$ for the standard Big-O and Big-Omega notations. We take $C$ and $c$ for some universal constants, which might vary from line to line. 

Next we introduce a notion of high probability events that was also used in Huang and Yau~\cite[Section 1.3]{Huang2019Dynamics}. We say that an event holds with high probability if the probability of the  event is at least $1-\exp\left(-  m^{\eps}\right)$ for some constant $\eps>0.$ Since for a deep neural network in practice,  we always have $m\lesssim \mathrm{poly}(n)$ and $n\lesssim \mathrm{poly}(m)$ \cite{kawaguchi2019gradient,Allen2018convergence}, then the intersection of a collection of  many high probability events still has the same property as long as  the number of events is at most polynomial in $m$ and $n.$ This terminology is also used by  Huang and Yau~ \cite[Section 1.3]{Huang2019Dynamics}.

\subsection{Problem Setup}\label{subsection....Problem Setup}
 We shall focus on the empirical risk minimization problem given by quadratic loss:
 \begin{equation}\label{eq for definition...Loss function with order 1/n}
\min_{\vtheta}R_S(\vtheta)=\frac{1}{2n}\sum_{\alpha=1}^n\Norm{f(\vx_\alpha,\vtheta)-y_\alpha}_2^2.
\end{equation}
In the above $\{ \vx_\alpha\}_{\alpha=1}^n$ are the training inputs, 
$\{ y_\alpha\}_{\alpha=1}^n$ are the labels, $f(\vx_\alpha,\vtheta)$ is the prediction function, and $\vtheta$ are the parameters to be optimized, and their dependence is modeled by ResNet with $L$ hidden layers, each of which has $m$ neurons. Let $\vx\in \sR^d$ be an input sample, then the network has $d$ input nodes. Let $\vx^{[l]}$ be the output of layer~$l$ with $\vx^{[0]}=\vx.$ We consider the ResNet given below:
\begin{equation}\label{eq for definition....the h-th layer for Resnet}
 \begin{aligned}
 \vx^{[1]}&=\sqrt{\frac{c_{\sigma}}{m}} \sigma(\mW^{[1]}\vx), \\
 \vx^{[l]}&= \vx^{[l-1]}+\frac{c_{\mathrm{res}}}{L\sqrt{m}}\sigma(\mW^{[l]}\vx^{[l-1]}), \ \  \text{for}  \ 2\leq l\leq L, 
 \end{aligned} 
\end{equation}
 where $\sigma(\cdot)$ is the activation function applied coordinate-wisely to its input. We assume that $\sigma(\cdot)$  is $1$-Lipschitz and smooth. The constant $c_{\sigma}=\left(\Exp_{x\sim \fN(0,1)} \left[\sigma(x)^2\right]\right)^{-1}$ is a scaling factor serving for the purpose of normalization, and $0<c_{\mathrm{res}}<1$ is a small constant. Moreover, we have a series of weight matrices $\left\{\mW^{[l]}\right\}_{l=1}^L$. Note that 
$\mW^{[l]}\in \sR^{m\times d}$ for $l=1$, and $\mW^{[l]}\in \sR^{m\times m} $ for $2\leq l\leq L$. The output function of ResNet is 
\begin{equation}\label{eq for definition...the output of REsNet}
      f_{\mathrm{res}}(\vx,\vtheta) =\va^\T \vx^{[L]},
\end{equation}
where $\va\in \sR^m$ is the weight vector of the output layer. We denote the vector containing all parameters by $\vtheta= \left(\mathrm{vec} \left(\mW^{[L]}\right),\mathrm{vec} \left(\mW^{[L-1]}\right),\dots,\mathrm{vec} \left(\mW^{[1]}\right),\va\right).$ Such a~parameterization has been employed widely, see \cite{Du2018Gradient,Du2018Gradienta,Lee2019Wide}. We shall initialize the parameter vector $\vtheta_0$ following the adopted Xavier initialization scheme \cite{Glorot2010Understanding}, i.e.,~$W_{i,j}^{[l]}\sim \fN(0, 1), a_k\sim \fN(0,1)$, where $\fN(0,1)$ denotes the standard Gaussian distribution. Applying  the continuous time GD fot the loss function \eqref{eq for definition...Loss function with order 1/n}, we have for any time $t\geq 0$:
\begin{align}
    \partial_t \mW^{[l]}_t&=-\partial_{\mW^{[l]}}R_S(\vtheta_t), \ l=1,2,\cdots, L, \label{eq... augumented GD for Weight matrices}\\
    \partial_t \va_t&=-\partial_{\va}R_S(\vtheta_t).\label{eq... augumented GD for Weight vectors}
\end{align}
We use $\fX=\{\vx_1,\vx_2,...,\vx_n\}$ for the set of input samples, $\sigma\left(\vW^{[l]}\vx_{\alpha}^{[l-1]}\right)$ as $\sigma_{[l]}(\vx_{\alpha})$, and the diagonal matrix generated by the  $r$-th derivatives of  $\sigma_{[l]}(\vx_{\alpha})$, i.e., $\mathrm{diag} \left(\sigma^{(r)}(\mW^{[l]}\vx_{\alpha}^{[h-1]})\right)$ by $\vsigma^{(r)}_{[l]}(\vx_{\alpha}),$ where $r\geq 1.$ We also write the output function $f_{\mathrm{res}}(\vx_{\alpha},\vtheta_t)$ as $f_{\alpha}(t).$ Moreover, we shall define a series of special matrices. Using $\mI_{{m}}$ to signify the identity matrix in $\sR^{m\times m}$, we define for $2\leq l\leq L:$
\begin{equation}\label{def of matrices...Big E matrix for I+res}
    \mE_{t,\alpha}^{[l]}:=\left(\mI_{{m}}+\frac{c_{\mathrm{res}}}{L}\vsigma^{(1)}_{[l]}(\vx_{\alpha})\frac{\mW_t^{[l]}}{\sqrt{m}}\right).
\end{equation}

\noindent The above matrices are termed   \emph{skip-connection matrices}. Given $\left\{\mE_{t,\alpha}^{[l]}\right\}_{l=2}^L$, we let $\mE_{t,\alpha}^{[2:L]}$ be the direct parameterization of the end-to-end  mapping realized by the group of skip-connection matrices, i.e., $\mE_{t,\alpha}^{[2:L]}:= \mE_{t,\alpha}^{[L]}\mE_{t,\alpha}^{[L-1]}\cdots\mE_{t,\alpha}^{[2]},$ where we set $\mE_{t,\alpha}^{[i:j]}:=\mI_{m}, i>j$ for completeness.

With the above notations, the continuous time GD dynamics \eqref{eq... augumented GD for Weight matrices} and \eqref{eq... augumented GD for Weight vectors} can be written as 
\begin{align}
\partial_t \va_t&=-\frac{1}{n}\sum_{\beta=1}^n\vx_{\beta}^{[L]}(f_{\beta}(t)-y_{\beta}), \label{eqgroup...dynamic for parameter a_t}\\
\partial_t \mW^{[L]}_t&=-\frac{1}{n}\sum_{\beta=1}^n\frac{c_{\mathrm{res}}}{L\sqrt{m}} \mathrm{diag}\left(\vsigma^{(1)}_{[L]}(\vx_{\beta}) \va_t \right)\vone \otimes  (\vx_{\beta}^{[L-1]})^\T (f_{\beta}(t)-y_{\beta}), \label{eqgroup...dynamic for parameter W_H}\\
\partial_t \mW^{[l]}_t&=-\frac{1}{n}\sum_{\beta=1}^n\frac{c_{\mathrm{res}}}{L\sqrt{m}}\mathrm{diag}\left(\vsigma^{(1)}_{[l]}(\vx_{\beta})  \left(\mE_{t,\beta}^{[(l+1):L]}\right)^{\T}       \va_t \right)\vone \otimes  (\vx_{\beta}^{[l-1]})^\T (f_{\beta}(t)-y_{\beta}), \label{eqgroup...dynamic for parameter W_l}\\
\text{for}  \ l&=2,3,\cdots, L-1,\nonumber\\
\partial_t \mW^{[1]}_t&=-\frac{1}{n}\sum_{\beta=1}^n\sqrt{\frac{c_{\sigma}}{m}} \mathrm{diag}\left(\vsigma^{(1)}_{[1]}(\vx_{\beta})\left(\mE_{t,\beta}^{[2:L]}\right)^{\T}       \va_t\right)\vone  \otimes  (\vx_{\beta})^\T (f_{\beta}(t)-y_{\beta}).\label{eqgroup...original dynamic for parameter W_1}
\end{align}

\subsection{Neural Tangent Kernel}\label{subsection....NTK}
The Neural Tangent Kernel~(NTK) is introduced in Jacot et al.~\cite{Jacot2018Neural}. For any parametrized function $f(\vx,\vtheta_t),$ it is defined as: $$\fK_{\vtheta_t}(\vx_{\alpha},\vx_{\beta})=\left<\nabla_{\vtheta} f(\vx_{\alpha},\vtheta_t),\nabla_{\vtheta} f(\vx_{\beta},\vtheta_t) \right>.$$ 
In the situations where $f(\vx,\vtheta_t)$ is the output of a fully-connected feedforward network with appropriate scaling factor $1/\sqrt{m}$ for the parameters, there is an infinite width limit ($m\to\infty$) of $\fK_{\vtheta_t}(\vx_{\alpha},\vx_{\beta}),$ denoted by $\fK_{\infty}(\vx_{\alpha},\vx_{\beta}).$
This result allows them to capture the behavior of fully-connected feedforward network trained by GD in the infinite width limit. More precisely, the output function $f(\vx,\vtheta_t)$~evolves as a linear differential equation:
\begin{equation}\label{eq for sidenote...evolve as a linear equation}
    \partial_t f(\vx,\vtheta_t)= -\frac{1}{n}\sum_{\beta=1}^n\fK_{\infty}(\vx,\vx_{\beta}) (f(\vx_{\beta},\vtheta_t)-y_{\beta}).
\end{equation}
Note that the training dynamic is identical to the dynamics of kernel regression under gradient flow.  Also  we note that $\fK_{\infty}(\cdot)$  only depends on the training inputs. More importantly, $\fK_{\infty}(\cdot)$ is independent of the neural network parameters $\vtheta$~\cite{Du2018Gradient,Du2018Gradienta,Arora2019exact}. Similar result holds for our ResNet structure.
 
The finding above is groundbreaking in that it provides us an analytically~tractable equation to predict the behavior of GD. However,  the convergence~$\fK_{\vtheta_t}(\vx_{\alpha},\vx_{\beta})$ to $\fK_{\infty}(\vx_{\alpha},\vx_{\beta})$ is proved in the regime of infinite width.~This is  unrealistic in nature. Some concurrent works concerning various network structures \cite{Li2018Learning,Du2018Gradient,song2019quadratic,Du2018Gradienta,Arora2019Fine,arora2018convergence} have extended the result in~\cite{Jacot2018Neural} to the regime of finite width. For a two-layer network with ReLU,  the required width $m$ in Song and Yang \cite{song2019quadratic} is $m=\Omega(n^2 \mathrm{poly}(\log(n))) $ under some strong assumptions on the input data. For fully-connected feedforward network, Huang and Yau requires width $m=\Omega\left(n^3\log(n)2^{\fO(L)}\right).$ Finally, for ResNet which is the main focus of our paper, the required width for Du et al.~\cite{Du2018Gradient} is $m=\Omega(n^4L^2)$ with iteration complexity $T=\Omega(n^2 \log\left(\frac{1}{\eps}\right)).$ Our  Corollary \ref{corollary....for thm} only requires  $m=\Omega(n^3L^2)$ and $T=\Omega\left(n \log\left(\frac{1}{\eps}\right)\right).$

We now write out the NTK for ResNet:
\begin{equation}\label{eq for derivation...neural tangent kernel}
\begin{aligned}
    \partial_t(f_\alpha(t)-y_\alpha)&=-\nabla_{\vtheta} f_\alpha(t)\cdot\nabla_{\vtheta}R_S(\vtheta_t)\\
    &=-\frac{1}{n} \nabla_{\vtheta} f_\alpha(t)\cdot\sum_{\beta=1}^n\nabla_{\vtheta} f_\beta(t)(f_\beta(t)-y_{\beta})\\
    &=-\frac{1}{n}\sum_{\beta=1}^n \fK_{\vtheta_t}(\vx_{\alpha},\vx_{\beta})(f_\beta(t)-y_{\beta}),
\end{aligned}
\end{equation}
 using equations \eqref{eqgroup...dynamic for parameter a_t}, \eqref{eqgroup...dynamic for parameter W_H}, \eqref{eqgroup...dynamic for parameter W_l} and \eqref{eqgroup...original dynamic for parameter W_1}, the NTK $\fK_{\vtheta_t}(\cdot)$ is given below
\begin{equation}\label{eq for derivation...NTK written into separate sums of kernels}
    \fK_{\vtheta_t}(\vx_{\alpha},\vx_{\beta})=\left<\nabla_{\vtheta} f_\alpha(t),\nabla_{\vtheta} f_\beta(t)\right>=\sum_{l=1}^{L+1} \fG_t^{[l]}(\vx_\alpha,\vx_\beta),
\end{equation}
where
\begin{align}
    &\fG_t^{[1]}(\vx_\alpha,\vx_\beta)=\left<\partial_{\mW^{[1]}}f_\alpha(t),\partial_{\mW^{[1]}}f_\beta(t)\right> \nonumber\\
    &=\left<\sqrt{\frac{c_{\sigma}}{m}}\vsigma^{(1)}_{[1]}(\vx_{\alpha})\left(\mE_{t,\alpha}^{[2:L]}\right)^{\T}\va_t ,\sqrt{\frac{c_{\sigma}}{m}}\vsigma^{(1)}_{[1]}(\vx_{\beta})\left(\mE_{t,\beta}^{[2:L]}\right)^{\T}\va_t \right>\left<\vx_{\alpha},\vx_{\beta}\right>,\label{eq for definition....definition for kernel G1}
\end{align}
for $2\leq l \leq L,$
\begin{align}
    &\fG_t^{[l]}(\vx_\alpha,\vx_\beta)=\left<\partial_{\mW^{[l]}}f_\alpha(t),\partial_{\mW^{[l]}}f_\beta(t)\right> \nonumber\\
    &=\left<\frac{c_{\mathrm{res}}}{L\sqrt{m}}\vsigma^{(1)}_{[l]}(\vx_{\alpha})  \left(\mE_{t,\alpha}^{[(l+1):L]}\right)^{\T}       \va_t ,\frac{c_{\mathrm{res}}}{L\sqrt{m}}\vsigma^{(1)}_{[l]}(\vx_{\beta})  \left(\mE_{t,\beta}^{[(l+1):L]}\right)^{\T}       \va_t \right>\left<\vx^{[l-1]}_{\alpha},\vx^{[l-1]}_{\beta}\right>,\label{eq for definition....definition for kernel G L}
\end{align}
and finally
\begin{equation}\label{eq for definition....definition for kernel G L+1}
    \fG_t^{[L+1]}(\vx_\alpha,\vx_\beta)=\left<\partial_{\va}f_\alpha(t),\partial_{\va}f_\beta(t)\right>=\left< \vx_{\alpha}^{[L]},\vx_{\beta}^{[L]} \right>.
\end{equation}
We note that all the $\fG_t^{[l]}$ depends on $\vtheta_t$ but for simplicity it is not explicitly written.
\section{Main Results}\label{section...main result}
\subsection{Activation function and input samples}\label{subsection....Activation function and input samples}
In this paper, we will impose some following  technical conditions  on the activation function and input samples.
\begin{assump}\label{Assump...Assumption on activation functions}
The activation function $\sigma(\cdot)$ is smooth, and there exists a universal constant $0<C_L\leq 1$   such that for any $r\geq 1,$ its $r$-th  derivative and the function value at $0$ satisfy 
\begin{equation}\label{eq for assumption...the uniform constant on activation function +derivatives}
    \Abs{\sigma(0)},\Norm{\sigma^{(r)}(\cdot)}_{\infty}\leq C_L.
\end{equation}
\end{assump}

\noindent Note that Assumption \ref{Assump...Assumption on activation functions} can be satisfied by using the softplus activation: $$\sigma(x)=\ln(1+\exp(x)).$$ Some other functions also satisfy this assumption, for instance, the sigmoid activation:$$\sigma(x)=\frac{1}{1+\exp(-x)}.$$
\begin{assump}\label{Assump... ont he imput of the  samples}
The training inputs and labels satisfy $\Norm{\vx_{\alpha}}_2=1, \Abs{y_{\alpha}}\leq 1$, for any $ {\vx_\alpha}\in \fX$. All training inputs are non-parallel with each other, i.e.,~$\vx_{\alpha_1}\nparallel \vx_{\alpha_2},$ for any $\alpha_1\neq \alpha_2$. 
\end{assump}
Assumption \ref{Assump... ont he imput of the  samples} guarantees that some of the Gram matrices defined in Section \ref{subsection....gram matrices} are strictly positive definite.
\subsection{Gram Matrices}\label{subsection....gram matrices}
Recent works~\cite{Du2018Gradienta,zhang2019fast,song2019quadratic} have shown that the convergence of the outputs of neural networks are determined by the spectral property of Gram matrices. Here we define the key Gram matrices $\left\{\mK^{[l]} \right\}_{l=1}^{L+1}$ below. We more or less follow the definition of the Gram matrices partially from  \cite[Definition 6.1]{Du2018Gradient}. Also we note that the Gram matrices depends on the series of matrices~$\left\{\widetilde{\mK}^{[l]} \right\}_{l=1}^L,$ $\left\{\widetilde{\mA}^{[l]}\right\}_{l=1}^{L+1},$ and the series of vectors $\left\{\widetilde{\vb}^{[l]} \right\}_{l=1}^L,$ which are listed out as follows, for $2\leq l \leq L$
\begin{align}
\widetilde{\mK}^{[0]}_{ij}&=\left<\vx_i,\vx_j\right>,\nonumber\\
\widetilde{\mK}^{[1]}_{ij}&=\Exp_{(u,v)^{\T}\sim \fN\left(\vzero, \begin{pmatrix}\widetilde{\mK}_{ii}^{[0]}&\widetilde{\mK}_{ij}^{[0]}\\
\widetilde{\mK}_{ji}^{[0]}&\widetilde{\mK}_{jj}^{[0]}\end{pmatrix} \right)} c_{\sigma}\sigma(u)\sigma(v),\nonumber\\
\widetilde{\vb}^{[1]}_i&=\sqrt{c_{\sigma}}\Exp_{u\sim \fN(0,\widetilde{\mK}_{ii}^{[0]})} \left[\sigma(u)\right], \nonumber\\
\widetilde{\mA}^{[l]}_{ij}&=\begin{pmatrix}\widetilde{\mK}_{ii}^{[l-1]}&\widetilde{\mK}_{ij}^{[l-1]}\\
\widetilde{\mK}_{ji}^{[l-1]}&\widetilde{\mK}_{jj}^{[l-1]}\end{pmatrix},\nonumber\\
\widetilde{\mK}^{[l]}_{ij}&=\widetilde{\mK}_{ij}^{[l-1]}+\Exp_{(u,v)^{\T}\sim \fN\left(\vzero, \widetilde{\mA}^{[l]}_{ij}\right)}\left[\frac{c_{\mathrm{res}}\widetilde{\vb}_{i}^{[l-1]}\sigma(v) }{L} +\frac{c_{\mathrm{res}}\widetilde{\vb}_{j}^{[l-1]}\sigma(u) }{L}+\frac{c_{\mathrm{res}}^2\sigma(u)\sigma(v)}{L^2} \right],\nonumber\\
\widetilde{\vb}^{[l]}_i&=\widetilde{\vb}_{i}^{[l-1]}+\frac{c_{\mathrm{res}}}{L}\Exp_{u\sim \fN(0,\widetilde{\mK}_{ii}^{[l-1]})} \left[\sigma(u)\right], \nonumber\\
\widetilde{\mA}^{[L+1]}_{ij}&=
\begin{pmatrix}\widetilde{\mK}_{ii}^{[L]}&\widetilde{\mK}_{ij}^{[L]}\\
\widetilde{\mK}_{ji}^{[L]}&\widetilde{\mK}_{jj}^{[L]}\end{pmatrix},\nonumber
\end{align}
then we may proceed to the definitions of Gram matrices for $l=L+1$ and $L$.
\begin{defi}\label{def...Gram Matrix of order L+1}
Given the input samples $\fX=\{\vx_1,\vx_2,...,\vx_n\}$, the Gram matrix $\mK^{[L+1]}\in~\sR^{n\times n}$ is recursively defined as follows, for $1\leq i,j \leq n,  $
\begin{align}
\mK^{[L+1]}_{ij}&=\widetilde{\mK}_{ij}^{[L]}+\Exp_{(u,v)^{\T}\sim \fN\left(\vzero, \widetilde{\mA}^{[L+1]}_{ij}\right)}\left[\frac{c_{\mathrm{res}}\widetilde{\vb}_{i}^{[L]}\sigma(v) }{L} +\frac{c_{\mathrm{res}}\widetilde{\vb}_{j}^{[L]}\sigma(u) }{L}+\frac{c_{\mathrm{res}}^2\sigma(u)\sigma(v)}{L^2} \right].\label{definition...definition of the th Gram Matrix L+1}
\end{align}
\end{defi}
 
\begin{defi}\label{def...Gram Matrix of order L....}
Gram matrix $\mK^{[L]}\in \sR^{n\times n}$ is defined as follows,~for~$1\leq i,j\leq n$,
\begin{equation}\label{definition...definition of the Lth Gram Matrix} 
{\mK}^{[L]}_{ij}= \frac{c_{\mathrm{res}}^2}{L^2} \widetilde{\mK}_{ij}^{[L-1]}\Exp_{(u,v)^{\T}\sim \fN\left(\vzero, \widetilde{\mA}^{[L]}_{ij}\right)}\left[ \sigma^{(1)}(u)\sigma^{(1)}(v)\right]. 
\end{equation}
\end{defi}
Note that matrix ${\mK}^{[L]}$ coincides with ${\mK}^{[H]}$ given by \cite[Definition 6.1]{Du2018Gradient}. Now that given the definition of $\mK^{[L+1]}$ and $\mK^{[L]}$, we need to move forward to the definition of other Gram matrices $\left\{\mK^{[l]} \right\}_{l=1}^{L-1}.$ Since it is challenging to give   an explicit formula for the series of matrices $\left\{\mK^{[l]} \right\}_{l=1}^{L-1},$ we shall use a slightly different approach to write out the definitions for these matrices. 
\begin{defi}\label{definition....defintion of the rest of gram matrices}
Gram matrices $\mK^{[l]}\in \sR^{n\times n}$ are defined as follows, for $1\leq i,j\leq n, 2\leq l\leq L-1,$
\begin{equation}\label{definition....for the gram matrix K H-1...limiting case+m to infty}
    {\mK}^{[l]}_{ij}=\frac{c_{\mathrm{res}}^2}{L^2}\widetilde{\mK}_{ij}^{[l-1]} \lim_{m\to\infty}\frac{1}{m}\left<\vsigma^{(1)}_{[l]}(\vx_{i})  \left(\mE_{0,i}^{[(l+1):L]}\right)^{\T} \va_0,\vsigma^{(1)}_{[l]}(\vx_{j})  \left(\mE_{0,j}^{[(l+1):L]}\right)^{\T} \va_0\right>,
\end{equation}
and for $1\leq i,j \leq n,  l=1$,
\begin{equation}\label{definition....for the gram matrix K one ...limiting case+m to infty}
    {\mK}^{[1]}_{ij}=c_{\sigma}\widetilde{\mK}_{ij}^{[0]} \lim_{m\to\infty}\frac{1}{m}\left<\vsigma^{(1)}_{[1]}(\vx_{i})  \left(\mE_{0,i}^{[2:L]}\right)^{\T} \va_0,\vsigma^{(1)}_{[1]}(\vx_{j})  \left(\mE_{0,j}^{[2:L]}\right)^{\T} \va_0\right>.
\end{equation}
\end{defi}
\begin{rmk}
     Thanks to the Strong Law of Large Numbers, the above limit exists~\cite{Arora2019exact}. Since we send $m\to\infty$, the gram matrices only depend on the input samples and the activation patterns.
\end{rmk}
Moreover, in Section \ref{appendix section...least eigenvalue for Gram matrixces}, we show that under Assumption \ref{Assump... ont he imput of the  samples} and width $m\sim n^2$, $\mK^{[L+1]}$ and  $\mK^{[L]}$ are strictly positive definite.

\subsection{Convergence of Gradient Descent}\label{subsection.....convergence rate of GD}
Here  we state our main theorems for the NTH of ResNet.
\begin{thm}\label{thm......infinite family}
Under Assumption \ref{Assump...Assumption on activation functions} and \ref{Assump... ont he imput of the  samples}, there exists an infinite family of operators $\fK_t^{(r)}: \fX^r\to \sR, r\geq 2$ that describes the continuous time GD: 
\begin{equation}\label{eq for thm...neural tangent kernel .... order 2}
\partial_t(f_\alpha(t)-y_\alpha)=-\frac{1}{n}\sum_{\beta=1}^n \fK_t^{(2)}(\vx_{\alpha},\vx_{\beta})(f_\beta(t)-y_{\beta}),
\end{equation}
and for $r\geq 2,$ we have
\begin{equation}\label{eq for thm...neural tangent kernel .... higher order 2}
\partial_t\fK_t^{(r)}(\vx_{\alpha_1},\vx_{\alpha_2},\cdots, \vx_{\alpha_r})=-\frac{1}{n}\sum_{\beta=1}^n \fK_t^{(r+1)}(\vx_{\alpha_1},\vx_{\alpha_2},\cdots, \vx_{\alpha_r},\vx_{\beta})(f_\beta(t)-y_{\beta}).
\end{equation}
Moreover, with high probability w.r.t random initialization, for time $0\leq t \leq \sqrt{m}/{\LogLn}^{C'},$ the following holds: 
\begin{equation}\label{eq for thm...uniform estimate....order2}
    \Norm{\fK_t^{(2)}\left(\cdot\right)}_{\infty}\lesssim 1,
\end{equation}
and for $r\geq 3,$
\begin{equation}\label{eq for thm...uniform estimate....higher rder2}
    \Norm{\fK_t^{(r)}\left(\cdot\right)}_{\infty}\lesssim\frac{{\LogLn}^{C}}{m^{r/2-1}}.
\end{equation}
where the constant $C$   in general depends on  $r.$
\end{thm}
\begin{rmk}
The operator $\fK_t^{(2)}(\cdot)$ by definition is the same as the NTK $\fK_{\vtheta_t}(\cdot)$   derived in~\eqref{eq for derivation...neural tangent kernel}.
\end{rmk}

We note that as $r$ increases, the pre-factor in \eqref{eq for thm...uniform estimate....higher rder2} explodes exponentially fast in $r.$ However, this will not  significantly affect   the convergence of GD. Firstly, only some lower order kernels need to be analyzed. As is shown in the proof of Corollary \ref{corollary....for thm}, only the  kernels up to order $r=4$ will be used.  Secondly, we shall recall the NTK $\fK_t^{(2)}(\cdot)$ derived in \eqref{eq for derivation...NTK written into separate sums of kernels} :
\begin{equation}\label{eq for derivation...NTK written into separate sums of kernels...replicate}
    \fK_t^{(2)}(\vx_{\alpha},\vx_{\beta})=\left<\nabla_{\vtheta} f_\alpha(t),\nabla_{\vtheta} f_\beta(t)\right>=\sum_{l=1}^{L+1} \fG_t^{[l]}(\vx_\alpha,\vx_\beta),
\end{equation}
in the case of Huang and Yau \cite{Huang2019Dynamics}, for a fully-connected feedforward network, since all those  kernels $\fG_t^{[l]}$ are positive definite, then the sum of the least eigenvalue of all  the  kernels $\fG_t^{[l]}$ is much larger than the counterpart of a single kernel, i.e.,
 $$\lambda_{\min}\left[\fK_t^{(2)}\left(\vx_{\alpha},\vx_{\beta}\right)\right]_{1\leq \alpha,\beta\leq n}\gg \lambda_{\min}\left[\fG_t^{[l]}\left(\vx_{\alpha},\vx_{\beta}\right)\right]_{1\leq \alpha,\beta\leq n}.$$
However,  adding up all the kernels will not give substantial increase to the least eigenvalue for the ResNet.  Since there exists a scaling factor~$\frac{1}{L^2}$ for the  kernels $\fG_t^{[l]},$ where $2\leq l \leq L,$ then heuristically, the gap of the least eigenvalues between $\fK_t^{(2)}(\cdot)$ and $\fG_t^{[L+1]}(\cdot)+\fG_t^{[1]}(\cdot)$  is at most of order $\fO\left(\frac{L-1}{L^2}\right)=\fO\left(\frac{1}{L}\right).$ Hence for ResNet, we shall see that even if the depth $L$ gets larger, the least eigenvalue of the NTK is still concentrated on the kernels $\fG_t^{[L+1]}(\cdot)$ and $\fG_t^{[1]}(\cdot).$ Thanks to that observation, we only need to bring the kernel  $\fG_t^{[L+1]}(\cdot)$ to the spotlight. We omit the analysis of  $\fG_t^{[1]}(\cdot)$ because it is not needed in our proof.

It was proven in Theorem \ref{thm......infinite family} and other literatures \cite{Du2018Gradient,yang2019scaling,Arora2019exact} that the  change  of NTK during the dynamics for Deep Neural Network is bounded  by $\fO\left(\frac{1}{\sqrt{m}}\right).$ However, it was observed by Lee et al.~\cite{Lee2019Wide} that the  time variation of the NTK is closer to $\fO\left(\frac{1}{m}\right),$ indicating that there exists a performance gap between the kernel regression using the limiting NTK and neural networks. Such an observation has been confirmed by Huang and Yau~\cite{Huang2019Dynamics} and listed out as Corollary $\mathrm{2.4.}$  in their paper. We use a different approach to obtain similar results and state them as Theorem \ref{theorem...vary at 1/m}.
\begin{thm}\label{theorem...vary at 1/m}
Under Assumption \ref{Assump...Assumption on activation functions} and \ref{Assump... ont he imput of the  samples}, with high probability w.r.t random initialization,  for time $0\leq t \leq\sqrt{m}/{\LogLn}^{C'},$ the following holds:

\begin{equation}\label{eq for thm...uniform estimate....order3}
    \Norm{\partial_t \fG_t^{[L+1]}\left(\cdot\right)}_{\infty}\lesssim\frac{\left(1+t\right){\LogLn}^{C}}{m},
\end{equation}
where   the constant $C$   is independent of the depth $L.$ Moreover, the pre-factor in \eqref{eq for thm...uniform estimate....order3} is at most of order $\fO\left({L^2}\right).$
\end{thm}

As a direct consequence of Theorem \ref{thm......infinite family}, for the ResNet defined in \eqref{eq for definition....the h-th layer for Resnet}, with  width  $m\sim n^3,$ the GD converges to zero training loss at a linear rate. The precise statement is given in the following.
\begin{cor}\label{corollary....for thm}
Under Assumption \ref{Assump...Assumption on activation functions} and \ref{Assump... ont he imput of the  samples},  with $\mK^{[L+1]}$ defined in \eqref{definition...definition of the th Gram Matrix L+1}, we have that for some $\lambda_0>0,$ $ \lambda_{\min}\left(\mK^{[L+1]}\right)> \lambda_0.$ Equipped with this, we have the following two statements.

There exists  a small constant $\gamma_1>0$, such that for   $m=\Omega\left( \left(\frac{n}{\lambda_0}\right)^{2+\gamma_1}\right) ,$ with high probability w.r.t random initialization, the following holds:
\begin{equation}\label{least eigenvalue}
    \lambda_{\min}\left[\fK_0^{(2)}\left(\vx_{\alpha},\vx_{\beta}\right)\right]_{1\leq \alpha,\beta\leq n}\geq\frac{3}{4}\lambda_0.
\end{equation}
Furthermore, there exists  a small constant $\gamma_2>0$, such that for  $m=\Omega\left( \left(\frac{n}{\lambda_0}\right)^{3+\gamma_2}L^2\ln \left(\frac{1}{\eps}\right)^2\right), $
where $\eps>0$ is the desired accuracy for $R_S(\vtheta_t),$ then the training loss  $R_S(\vtheta_t)$ decays exponentially w.r.t time $t,$
\begin{equation}\label{training loss equation}
     R_S(\vtheta_t)\leq  R_S(\vtheta_0) \exp\left(-\frac{\lambda t}{n}\right).
\end{equation}
\end{cor}
For convenience, we summarize the above statement  in the following manner. If
\begin{equation}\label{requirement for width.........all}
    m=\max\left\{\Omega\left( \left(\frac{n}{\lambda}\right)^{2+\gamma_1}\right), \Omega\left( \left(\frac{n}{\lambda}\right)^{3+\gamma_2}L^2\ln \left(\frac{1}{\eps}\right)^2\right)\right\},
\end{equation}
then the continuous GD converges exponentially and reaches the training accuracy $\eps$ with time complexity 
\begin{equation}
    T=\fO\left(\frac{n}{\lambda} \ln \left(\frac{1}{\eps}\right)\right).
\end{equation}

Before we end this section, we present a fair comparison of our result with others. First of all, Du et al.~\cite[Theorem 6.1.]{Du2018Gradient} required $m=\Omega\left(\frac{n^4}{\lambda_{\min}\left(\mK^{[L]}\right)^4L^6}\right).$ Since there is a  scaling factor $\frac{1}{L^2}$ in $\lambda_{\min}\left(\mK^{[L]}\right),$  this leads to $m=\Omega\left({n^4}L^2\right).$ Then their GD converges with iteration complexity $T=\Omega\left(n^2L^2 \ln \left(\frac{1}{\eps}\right)\right).$ Our Corollary \ref{corollary....for thm} improves this result in two ways: (i) The quartic dependence on $n$ is reduced directly to cubic dependence. (ii) A faster convergence of the training process of GD.

Second, our work serves as an extension of the NTH proposed by Huang and Yau~\cite{Huang2019Dynamics}, which captures the GD dynamics for a fully-connected feedforward network. We  show that not only it is possible to  study directly the time variation of NTK  for ResNet  using NTH, but that  ResNet possesses more stability in many aspects than fully-connected network. In particular, we improve their results in three aspects: (i) With ResNet architecture, the dependency of the amount of over-parameterization on the depth $L$ can be reduced from their $2^{\fO(L)}$ to $L^2.$ (ii) While the  time interval for the result  in \cite{Huang2019Dynamics} takes the form $0\leq t \leq m^{\frac{p}{2(p+1)}}/{\LogLn}^{C'}$ for some $p\geq 2$, we extend the interval to $0\leq t \leq \sqrt{m}/{\LogLn}^{C'}.$ Moreover, we are able to show even further that the results hold true for $t\to \infty$ using techniques from~\cite{ma2019analysis}. (iii) In the proof of Corollary $\mathrm{2.5.}$ in \cite{Huang2019Dynamics},  a further assumption  on the least eigenvalue of the NTK $\fK_t^{(2)}(\cdot)$ has been imposed directly, we show in  Appendix \ref{appendix section,,,,,random gram} that the least eigenvalue of the NTK $\fK_t^{(2)}(\cdot)$  can be guaranteed with high probability as long as the width $m$ satisfies $m=\Omega(n^2).$

\section{Technique Overview}\label{section....Techinique}
In this part we first describe some technical tools and present the sketch of proofs for Theorem \ref{thm......infinite family} and \ref{theorem...vary at 1/m} and Corollary \ref{corollary....for thm}. 
\subsection{Replacement Rules}\label{subsection....replacemetn rules}
We revisit the NTK \eqref{eq for derivation...neural tangent kernel}  derived in Section \ref{subsection....NTK}, 
\begin{equation}\label{eq for derivation...NTK written revisit}
    \fK_{\vtheta_t}(\vx_{\alpha},\vx_{\beta})=\sum_{l=1}^{L+1} \fG_t^{[l]}(\vx_\alpha,\vx_\beta),
\end{equation}
Notice that $\fK_{\vtheta_t}(\cdot)$ coincides with $\fK_{t}^{(2)}(\cdot)$ in \eqref{eq for derivation...neural tangent kernel}, and $\fK_{t}^{(2)}(\cdot)$ is the sum of $L+1$ terms, with each term being the inner product of vectors containing the quantities $\va_t,  \vx_{\alpha}^{[l]}, \mE_{t,\alpha}^{[l]}$ and $\vsigma^{(1)}_{[l]}(\vx_{\alpha}).$ We are able to write down the dynamics of  $\va_t,  \vx_{\alpha}^{[l]}, \mE_{t,\alpha}^{[l]}$ and $\vsigma^{(1)}_{[l]}(\vx_{\alpha})$ following GD, using equation \eqref{eqgroup...dynamic for parameter a_t}, \eqref{eqgroup...dynamic for parameter W_H} \eqref{eqgroup...dynamic for parameter W_l}, \eqref{eqgroup...original dynamic for parameter W_1} and chain rules. In order to shorten the space, we perform a similar  replacement rule as in Huang and Yau~\cite{Huang2019Dynamics}. For instance, the dynamics of  $\va_t$  is written as 
\begin{equation}
   \partial_t \va_t=-\frac{1}{n}\sum_{\beta=1}^n\frac{1}{\sqrt{m}}\sqrt{m}\vx_{\beta}^{[L]}(f_{\beta}(t)-y_{\beta}).\label{eq..a_tdynamics}
\end{equation}
For simplicity, we symbolize the dynamics \eqref{eq..a_tdynamics} as $\va_t\to\frac{1}{\sqrt{m}}\sqrt{m}\vx_{\beta}^{[L]}.$
Similarly, for the dynamics of $\vx_{\alpha}^{[l]}, 2\leq l \leq L$, we have
\begin{align*}
\sqrt{m}\vx_{\alpha}^{[1]}&\to
{\frac{c_\sigma}{\sqrt{m}}} \mathrm{diag}\left( \vsigma^{(1)}_{[1]}(\vx_{\alpha})\vsigma^{(1)}_{[1]}(\vx_{\beta})\left(\mE_{t,\beta}^{[2:L]}\right)^{\T}\va_t \right)\vone \left<\vx_\alpha,\vx_\beta  \right>,\\
\sqrt{m}\vx_{\alpha}^{[l]}&\to
{\frac{c_\sigma}{\sqrt{m}}} \mathrm{diag}\left(\mE_{t,\alpha}^{[2:l]} \vsigma^{(1)}_{[1]}(\vx_{\alpha})\vsigma^{(1)}_{[1]}(\vx_{\beta})\left(\mE_{t,\beta}^{[2:L]}\right)^{\T}\va_t \right)\vone \left<\vx_\alpha,\vx_\beta  \right>\\
&+\sum_{k=2}^{l} \frac{c_{\mathrm{res}}^2}{L^2 \sqrt {m}} \mathrm{diag}\left(\mE_{t,\alpha}^{[(k+1):l]}\vsigma^{(1)}_{[k]}(\vx_{\alpha})\vsigma^{(1)}_{[k]}(\vx_{\beta})\left(\mE_{t,\beta}^{[(k+1):L]}\right)^{\T}\va_t  \right) \vone \left<\vx_\alpha^{[k-1]},\vx_\beta^{[k-1]}  \right>,
\end{align*}
and of $\vsigma^{(1)}_{[l]}(\vx_{\alpha})$, for $2\leq l \leq L-1, r\geq 1 $
\begin{align*}
\vsigma_{[1]}^{(r)}(\vx_\alpha)&\to \sqrt{\frac{c_\sigma}{m}}\vsigma_{[1]}^{(r+1)}(\vx_\alpha)\mathrm{diag}\left( \vsigma^{(1)}_{[1]}(\vx_{\beta})\left(\mE_{t,\beta}^{[2:L]}\right)^{\T}\va_t \right) \left<\vx_\alpha,\vx_\beta  \right>,\\
\vsigma_{[2]}^{(r)}(\vx_\alpha)&\to\frac{c_{\mathrm{res}}}{L\sqrt {m}}\vsigma_{[2]}^{(r+1)}(\vx_\alpha)\mathrm{diag}\left( \vsigma^{(1)}_{[2]}(\vx_{\beta})\left(\mE_{t,\beta}^{[3:L]}\right)^{\T}\va_t \right) \left<\vx_\alpha^{[1]},\vx_\beta^{[1]}  \right>\\
&+{\frac{c_\sigma}{\sqrt{m}}} \vsigma_{[2]}^{(r+1)}(\vx_\alpha) \mathrm{diag}\left(\frac{\mW_t^{[2]}}{\sqrt{m}}\vsigma^{(1)}_{[1]}(\vx_{\alpha})\vsigma^{(1)}_{[1]}(\vx_{\beta})\left(\mE_{t,\beta}^{[2:L]}\right)^{\T}\va_t \right)\left<\vx_\alpha,\vx_\beta  \right>,\\
\vsigma_{[l+1]}^{(r)}(\vx_\alpha)&\to\frac{c_{\mathrm{res}}}{L\sqrt {m}}\vsigma_{[l+1]}^{(r+1)}(\vx_\alpha)\mathrm{diag}\left( \vsigma^{(1)}_{[l+1]}(\vx_{\beta})\left(\mE_{t,\beta}^{[(l+2):L]}\right)^{\T}\va_t \right) \left<\vx_\alpha^{[l]},\vx_\beta^{[l]}  \right>\\
+\sum_{k=2}^{l}&\frac{c_{\mathrm{res}}^2}{L^2 \sqrt {m}}\vsigma_{[l+1]}^{(r+1)}(\vx_\alpha) \mathrm{diag}\left(\frac{\mW_t^{[l+1]}}{\sqrt{m}}\mE_{t,\alpha}^{[(k+1):l]}\vsigma^{(1)}_{[k]}(\vx_{\alpha})\vsigma^{(1)}_{[k]}(\vx_{\beta})\left(\mE_{t,\beta}^{[(k+1):L]}\right)^{\T}\va_t  \right) \left<\vx_\alpha^{[k-1]},\vx_\beta^{[k-1]}  \right>\\
&+{\frac{c_\sigma}{\sqrt{m}}} \vsigma_{[l+1]}^{(r+1)}(\vx_\alpha)\mathrm{diag}\left(\frac{\mW_t^{[l+1]}}{\sqrt{m}}\mE_{t,\alpha}^{[2:l]} \vsigma^{(1)}_{[1]}(\vx_{\alpha})\vsigma^{(1)}_{[1]}(\vx_{\beta})\left(\mE_{t,\beta}^{[2:L]}\right)^{\T}\va_t \right) \left<\vx_\alpha,\vx_\beta  \right>,
\end{align*}
and finally of  $\mE_{t,\alpha}^{[l]}, 2\leq l\leq L-1$, 
\begin{align*}
\mE_{t,\alpha}^{[2]}&\to\frac{c_{\mathrm{res}}^2}{L^2\sqrt{m}}\mathrm{diag}\left(\vsigma_{[2]}^{(1)}(\vx_\alpha)\vsigma^{(1)}_{[2]}(\vx_{\beta})  \left(\mE_{t,\beta}^{[3:L]}\right)^{\T}       \va_t\right)\vone \otimes  (\frac{\sqrt{m}\vx_{\beta}^{[1]}}{m})^\T\\
&+\frac{c_{\mathrm{res}}}{L\sqrt {m}}\vsigma_{[2]}^{(2)}(\vx_\alpha)\mathrm{diag}\left( \vsigma^{(1)}_{[2]}(\vx_{\beta})\left(\mE_{t,\beta}^{[3:L]}\right)^{\T}\va_t \right)\frac{c_{\mathrm{res}}}{L}\frac{\mW_t^{[2]}}{\sqrt{m}} \left<\vx_\alpha^{[1]},\vx_\beta^{[1]}  \right>\\
&+{\frac{c_\sigma}{\sqrt{m}}} \vsigma_{[2]}^{(2)}(\vx_\alpha) \mathrm{diag}\left(\frac{\mW_t^{[2]}}{\sqrt{m}}\vsigma^{(1)}_{[1]}(\vx_{\alpha})\vsigma^{(1)}_{[1]}(\vx_{\beta})\left(\mE_{t,\beta}^{[2:L]}\right)^{\T}\va_t \right)\frac{c_{\mathrm{res}}}{L}\frac{\mW_t^{[2]}}{\sqrt{m}}\left<\vx_\alpha,\vx_\beta  \right>,  \\
\mE_{t,\alpha}^{[l+1]}&\to\frac{c_{\mathrm{res}}^2}{L^2\sqrt{m}}\mathrm{diag}\left(\vsigma_{[l+1]}^{(1)}(\vx_\alpha)\vsigma^{(1)}_{[l+1]}(\vx_{\beta})  \left(\mE_{t,\beta}^{[(l+2):L]}\right)^{\T}       \va_t \right)\vone\otimes  (\frac{\sqrt{m}\vx_{\beta}^{[l]}}{m})^\T\\
&+\frac{c_{\mathrm{res}}^2}{L^2\sqrt {m}}\vsigma_{[l+1]}^{(2)}(\vx_\alpha)\mathrm{diag}\left( \vsigma^{(1)}_{[l+1]}(\vx_{\beta})\left(\mE_{t,\beta}^{[(l+2):L]}\right)^{\T}\va_t \right)\frac{\mW_t^{[l+1]}}{\sqrt{m}} \left<\vx_\alpha^{[l]},\vx_\beta^{[l]}  \right>\\
+\sum_{k=2}^{l}&\frac{c_{\mathrm{res}}^3}{L^3 \sqrt {m}}\vsigma_{[l+1]}^{(2)}(\vx_\alpha) \mathrm{diag}\left(\frac{\mW_t^{[l+1]}}{\sqrt{m}}\mE_{t,\alpha}^{[(k+1):l]}\vsigma^{(1)}_{[k]}(\vx_{\alpha})\vsigma^{(1)}_{[k]}(\vx_{\beta})\left(\mE_{t,\beta}^{[(k+1):L]}\right)^{\T}\va_t  \right) \frac{\mW_t^{[l+1]}}{\sqrt{m}}\left<\vx_\alpha^{[k-1]},\vx_\beta^{[k-1]}  \right>\\
&+{\frac{c_\sigma}{\sqrt{m}}}\frac{c_{\mathrm{res}}}{L} \vsigma_{[l+1]}^{(2)}(\vx_\alpha)\mathrm{diag}\left(\frac{\mW_t^{[l+1]}}{\sqrt{m}}\mE_{t,\alpha}^{[2:l]} \vsigma^{(1)}_{[1]}(\vx_{\alpha})\vsigma^{(1)}_{[1]}(\vx_{\beta})\left(\mE_{t,\beta}^{[2:L]}\right)^{\T}\va_t \right) \frac{\mW_t^{[l+1]}}{\sqrt{m}}\left<\vx_\alpha,\vx_\beta  \right>.
\end{align*}
We   notice that the constant $\frac{c_{\mathrm{res}}}{L}$ plays an important part in our proof, so that the width per layer $m$ does not depend exponentially in depth $L$.

Using the above rules, the derivative for NTK $\fK_t^{(2)}(\cdot)$ is obtained in the following form
\begin{equation*}
\partial_t   \fK_t^{(2)}(\vx_{\alpha_1},\vx_{\alpha_2})=-\frac{1}{n} \sum_{\beta=1}^n   \fK_t^{(3)}(\vx_{\alpha_1},\vx_{\alpha_2},\vx_\beta)(f_\beta(t)-y_{\beta}),
\end{equation*}
where each term in $\fK_t^{(3)}(\vx_{\alpha_1},\vx_{\alpha_2},\vx_\beta)$ is the summation of all the terms generated from   $\fK_t^{(2)}(\vx_{\alpha_1},\vx_{\alpha_2})$ by performing the replacement procedure. In order to illustrate the idea, we give out an example in the proof of Theorem \ref{theorem...vary at 1/m} in Section \ref{subsection...sketch of proof}. 

By the same reasoning, we could obtain the higher order kernels inductively by performing all the possible replacements. For instance, for kernel $\fK_t^{(r)}(\vx_{\alpha_1},\vx_{\alpha_2},\dots,\vx_{\alpha_r} )$, we could obtain $\fK_t^{(r+1)}(\vx_{\alpha_1},\dots,\vx_{\alpha_r};\vx_\beta )$ given by the following Ordinary Differential Equation
\begin{equation*}
        \partial_t\fK_t^{(r)}(\vx_{\alpha_1},\vx_{\alpha_2}\dots,\vx_{\alpha_r})=-\frac{1}{n}\sum_{\beta=1}^n \fK_t^{(r+1)}(\vx_{\alpha_1},\vx_{\alpha_2}\dots,\vx_{\alpha_r},\vx_{\beta})(f_\beta (t)-y_{\beta}).
\end{equation*}
In order to describe the vectors appearing in $\fK_t^{(r)}(\vx_{\alpha_1},\vx_{\alpha_2}\dots,\vx_{\alpha_r})$, we need to introduce some systematic notations. 

\subsection{Hierarchical Sets of Kernel Expressions}\label{subsection...Hierarchy of Sets}
The hierarchy of sets are proposed originally by Huang and Yau in \cite{Huang2019Dynamics}. We denote $\sA_0$ the first set of expressions in the following form, which corresponds to the terms in $\fK_t^{(2)}(\vx_{\alpha_1},\vx_{\alpha_2})$. We define  $\sA_0$ as :
\begin{equation}\label{definition of hierarchy of sets... the set  A0}
    \sA_0\triangleq\left\{\ve_s\ve_{s-1}\dots \ve_1\ve_0: 0\leq s\leq 4L\right\},
\end{equation}
where $\ve_j$ is chosen following the rules:
\begin{align}
    \ve_0\in\left\{\va_t,\{ \sqrt{m} \vx_\beta^{[1]},\sqrt{m} \vx_\beta^{[2]},\dots,\sqrt{m} \vx_\beta^{[L]}  \}_{1\leq \beta\leq n}   \right\},\label{definition of sets A0...for kernel order 2}
\end{align}
and for $1\leq j\leq s,$
\begin{align}
    \ve_j \in \left\{\left\{ \mE_{t,\beta}^{[2]},\left(\mE_{t,\beta}^{[2]}\right)^{\T},\dots,\mE_{t,\beta}^{[L]},\left(\mE_{t,\beta}^{[L]}\right)^{\T}\right\}_{1\leq \beta\leq n},\left\{\vsigma_{[1]}^{(1)}(\vx_\beta),\dots,\vsigma_{[L]}^{(1)}(\vx_\beta)\right\}_{1\leq \beta\leq n}\right\}.\label{definition of sets A r...for kernel order 2}
\end{align}
From equation \eqref{eq for definition....definition for kernel G1}, \eqref{eq for definition....definition for kernel G L} and \eqref{eq for definition....definition for kernel G L+1}, each term in $\fK_t^{(2)}(\vx_{\alpha_1},\vx_{\alpha_2})$ writes as 
\begin{equation*}
    \frac{\left<\vv_1(t),\vv_2(t)\right>}{m} \ \text{or} \   \frac{\left<\vv_1(t),\vv_2(t)\right>}{m} \frac{\left<\vv_3(t),\vv_4(t)\right>}{m},
\end{equation*}
where $\vv_1(t),\vv_2(t),\vv_3(t),\vv_4(t) \in \sA_0.$ Note that $\vv_i(t)$ can take the value of 
$\vv_i(t)=\sqrt{m} \vx_{\alpha},$
which are not contained in $\sA_0$, however such singularity is not a big issue, see Appendix~\ref{Appen subsection....Aprioir L2 bounds for sets in A0}. We remark that compared with \cite{Huang2019Dynamics}, $\ve_j$ is chosen in a way different from ours, the counterpart in \cite{Huang2019Dynamics} is chosen from the set $$\left\{\left\{ \frac{\mW_t^{[2]}}{\sqrt{m}},\left(\frac{\mW_t^{[2]}}{\sqrt{m}}\right)^{\T},\dots,\frac{\mW_t^{[L]}}{\sqrt{m}},\left(\frac{\mW_t^{[L]}}{\sqrt{m}}\right)^{\T}\right\}_{1\leq \beta\leq n},\left\{\vsigma_{[1]}^{(1)}(\vx_\beta),\dots,\vsigma_{[L]}^{(1)}(\vx_\beta)\right\}_{1\leq \beta\leq n}\right\}.$$
Such changes arise from the change of the network structure, and it has been shown in Appendix \ref{Appen subsection....Aprioir L2 bounds for sets in A0} that the group of skip-connection matrices $\mE_{t,\beta}^{[l]}$ possesses more stability than  $\frac{\mW_t^{[l]}}{\sqrt{m}}$.

Moreover,  given  the construction of  $\sA_0,\sA_1,\dots,\sA_r$, we denote $\sA_{r+1}$ the set of expressions in the following form:
\begin{equation*} 
    \sA_{r+1}\triangleq\left\{\ve_s\ve_{s-1}\dots \ve_1\ve_0: 0\leq s\leq 4L\right\},
\end{equation*}  
where $e_j$ is chosen from the following sets:
\begin{align}
    e_0\in\left\{\va_t,\vone, \{ \sqrt{m} \vx_\beta^{[1]},\sqrt{m} \vx_\beta^{[2]},\dots,\sqrt{m} \vx_\beta^{[L]}  \}_{1\leq \beta\leq n}   \right\},\label{definition of sets...kernel of order R}
\end{align}
and for $1\leq j\leq s,$ we have that each $\ve_j$ comes from one of the three following sets
\begin{align*}
   & \left\{\left\{ \mE_{t,\beta}^{[2]},\left(\mE_{t,\beta}^{[2]}\right)^{\T},\dots,\mE_{t,\beta}^{[L]},\left(\mE_{t,\beta}^{[L]}\right)^{\T}\right\}_{1\leq \beta\leq n},\left\{\vsigma_{[1]}^{(1)}(\vx_\beta),\dots,\vsigma_{[L]}^{(1)}(\vx_\beta)\right\}_{1\leq \beta\leq n}\right\},\label{definition of sets  case one...kernel of order R}\\
    &\left\{ \mathrm{diag}(\vg), \ \vg\in \sA_0\cup\sA_1\cup\dots\cup \sA_r  \right\},\\
 &\Big\{\vsigma_{[l]}^{(u+1)}(\vx_\beta) \mathrm{diag} \left(\left(\frac{\mW_t^{[l]}}{\sqrt{m}}\right)^{Q_1} \vg_1\right)\dots\mathrm{diag} \left(\left(\frac{\mW_t^{[l]}}{\sqrt{m}}\right)^{Q_u} \vg_u\right)\left(\frac{c_{\mathrm{res}}}{L}\frac{\mW_t^{[l]}}{\sqrt{m}}\right)^{Q_{u+1}},\nonumber\\
    &\left(\frac{c_{\mathrm{res}}}{L}\frac{\left(\mW_t^{[l]}\right)^\T}{\sqrt{m}}\right)^{Q_{u+1}}\vsigma_{[l]}^{(u+1)}(\vx_\beta) \mathrm{diag} \left(\left(\frac{\mW_t^{[l]}}{\sqrt{m}}\right)^{Q_1} \vg_1\right)\dots\mathrm{diag} \left(\left(\frac{\mW_t^{[l]}}{\sqrt{m}}\right)^{Q_u} \vg_u\right):2\leq l \leq L,\nonumber\\
 &~~~~~~1\leq \beta \leq n, 1\leq u\leq r, \  \vg_1,\vg_2\dots \vg_u\in \sA_0\cup\sA_1\cup\dots\cup \sA_r \  \text{and} \  Q_1,Q_2\dots Q_{u+1}\in \{0,1\}\Big  \},
\end{align*}
 the maximum possible total number of diag operations for any element in $\sA_r$ is $r$, i.e., if $\vv(t)\in \sA_r,$ it contains at most $r$ diag operations. We observe from the replacement rules, there will be a scaling of $\frac{1}{\sqrt{m}}$ whenever we take derivatives, hence inductively,   for each term in kernel $\fK_t^{(p)}(\vx_{\alpha_1},\vx_{\alpha_2},\dots,\vx_{\alpha_p} )$, it takes the form \begin{equation}\label{equation for generation....of the kernel of order p}
    \frac{1}{m^{p/2-1}}\prod_{j=1}^s\frac{\left<\vv_{2j-1}(t),\vv_{2j}(t)\right>}{m}, \ 1\leq s\leq p, \ \ \vv_i(t)\in \sA_0\cup\sA_1\cup\dots\cup \sA_{p-2},
\end{equation}
which is a direct consequence of Proposition \ref{prop...A.1}.
Note  that $\vv_i(t)$ can still take the value of 
$\vv_i(t)=\sqrt{m} \vx_{\alpha},$
which are not in the set $\sA_{r+1}.$
Huang and Yau also obtained \eqref{equation for generation....of the kernel of order p} in equation (3.8) in \cite{Huang2019Dynamics}, and they use the tensor program proposed by Yang \cite{yang2019scaling} to estimate the initial value of the kernel $\fK_0^{(p)}(\vx_{\alpha_1},\vx_{\alpha_2},\dots,\vx_{\alpha_p} ).$ They showed that for each vector $\vv_j(t)$ in \eqref{equation for generation....of the kernel of order p} at $t=0$,  it is  a linear  combination of projections of independent Gaussian vectors.  Hence, if we consider such quantity
\begin{equation*}
    \eta(t)=\left\{\Norm{\vv(t)}_{\infty}:  \vv(t)\in \sA_0\cup\sA_1\dots\cup\sA_{r}\right\},
\end{equation*}
    at $t=0$, since $\vv(0)$ is  a linear  combination of projections of independent Gaussian vectors, then with high probability, $\eta(0)\lesssim {\LogLn}^{C}.$ For $t>0,$ Huang and Yau derived a self-consistent Ordinary Differential Inequality   for $\eta(t)$ :
\begin{align}
    \partial_t^{(p+1)}\eta(t) \lesssim \frac{\eta(t)^{2p}}{m^{p/2}}, \label{ode..huang}\\
    \eta(0)\lesssim {\LogLn}^{C}, \label{ode..condition}
\end{align}
then it holds that $\eta(t)\lesssim {\LogLn}^{C}$ for time $0\leq t\leq m^{\frac{p}{2p+1}}/{\LogLn}^{C'}.$

Our approach is different from them, instead of using tensor programs, we use a special  matrix norm, the $2$ to infinity matrix norm , to show that  $\eta(0)\lesssim {\LogLn}^{C},$ and we show a  Gronwall-type inequality for $\eta(t)$:
\begin{align*}
    \eta(t) \lesssim {\LogLn}^{C}+ \frac{1}{\sqrt{m}} \int_{0}^t \eta(s) \diff s,
\end{align*}
then it follows that for time $0\leq t \leq \sqrt{m}/\LogLn^{C'},$ $\eta(t)\lesssim {\LogLn}^{C}$ holds. Then~\eqref{eq for thm...neural tangent kernel .... order 2} and~\eqref{eq for thm...neural tangent kernel .... higher order 2} in Theorem~\ref{thm......infinite family} holds, and we are able to show that the kernels of higher order vary slowly, which brings us the proof of Theorem \ref{theorem...vary at 1/m}.

\subsection{Least Eigenvalue for Randomly Initialized Matrix}\label{subsection....least eigenvalue}
Firstly, since $\mK^{[L]}$ is a recursively defined matrix, we use results in  Du et al.~\cite{Du2018Gradient} to show that the Gram matrix $\mK^{[L]}$ is positive definite. Second,  we need to analyze how the
difference between $\mG^{[1]}$ and $\mK^{[1]}$, termed  \emph{the
perturbation} by Du et al.~\cite{Du2018Gradient},  from lower layers propagates to the $L$-th layer. We quantitatively characterize how large such propagation dynamics would be and   rediscover  that ResNet architecture serves as a stabilizer for such propagation~(Proposition \ref{proposition..................Initialization Norms layer 1}). Our proof is slightly different from \cite{Du2018Gradient}, where we use the concentration inequality for Lipschitz functions. We refer readers to Appendix \ref{appendix section,,,,,random gram} for details.

\subsection{Sketch of Proof}\label{subsection...sketch of proof}
We use Figure \ref{diagram} to  illustrate    the ideas of the proofs. Due to space contraints, all the proofs of the techinical Lemmas and Propositions are provided in Supplementary Material. Note that for the quantities in Figure \ref{diagram}, $\lambda_0$ is the least eigenvalue of $\widetilde{\mK}_{ij}^{[1]},$  $\xi_{\infty,r}(t)=\sup_{0\leq t'\leq t}\left\{\Norm{\vv(t')}_2:  \vv(t')\in \sA_r\right\},$ and $\eta_{\infty,r}(t)=\sup_{0\leq t'\leq t}\left\{ \Norm{\vv(t')}_{\infty}:  \vv(t')\in \sA_r\right\},$ where $ r\geq 0.$ Now we proceed to the Proof of Theorem \ref{thm......infinite family}.
\begin{figure}[h]
\caption{Diagram of the Proof of Main Theorems}
\centering
\includegraphics[width=\textwidth]{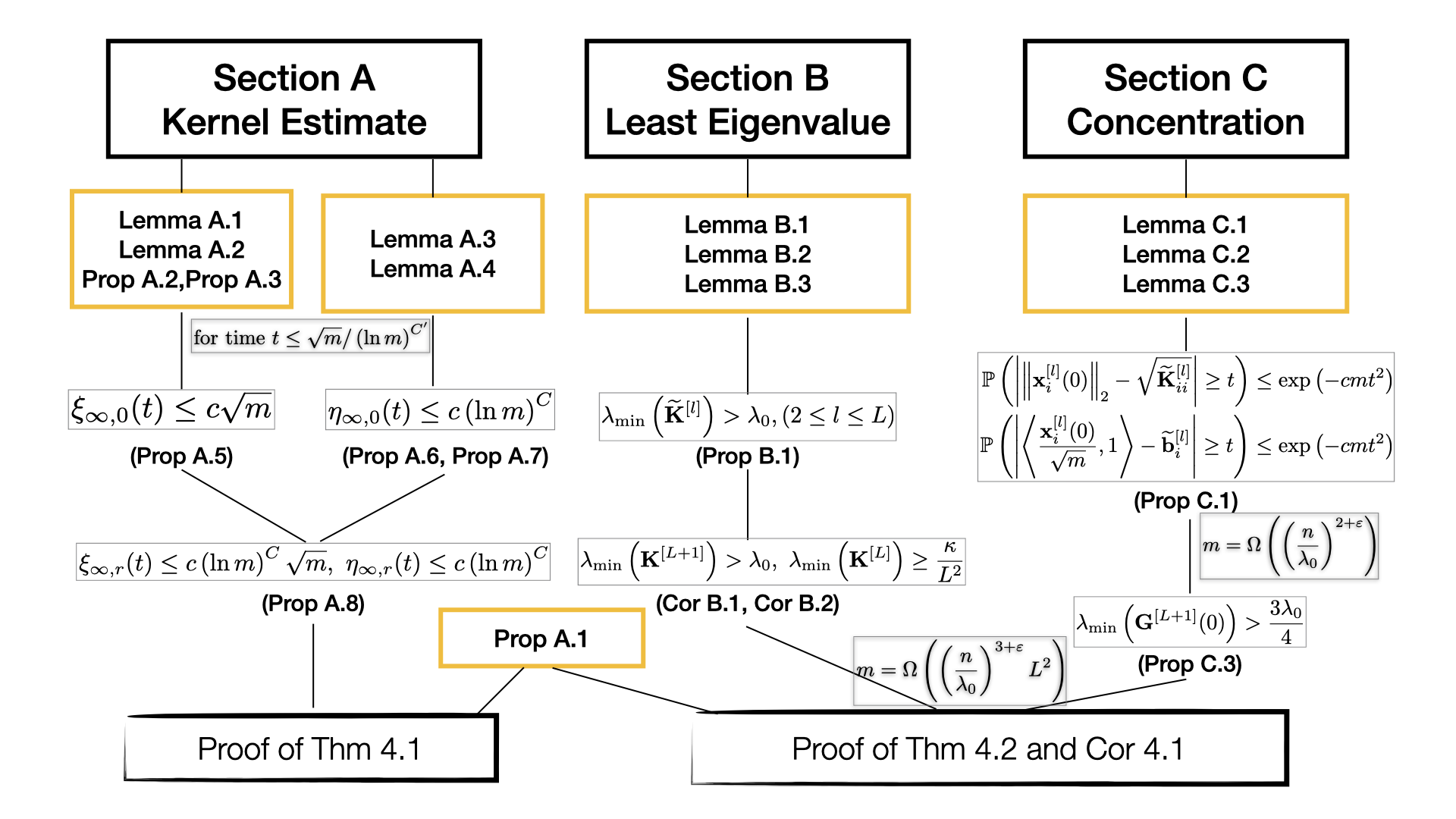}
\label{diagram}
\end{figure}

\begin{proof}[Proof of Theorem \ref{thm......infinite family}]
Since each term in kernel $\fK_t^{(r)}(\vx_{\alpha_1},\vx_{\alpha_2},\dots,\vx_{\alpha_r} ),$ it takes the form \begin{equation*}
    \frac{1}{m^{r/2-1}}\prod_{j=1}^s\frac{\left<\vv_{2j-1}(t),\vv_{2j}(t)\right>}{m}, \ 1\leq s\leq r, \ \ \vv_i(t)\in \sA_0\cup\sA_1\cup\dots\cup \sA_{r-2},
\end{equation*}
then for time  $0\leq t \leq \sqrt{m}/{\LogLn}^{C'},$ $r\geq 3$
\begin{align*}
    \Norm{\fK_t^{(2)}(\cdot)}_{\infty}&\lesssim  \left(\frac{\xi_{\infty,0}(t)^2}{m}\right)^2 \lesssim 1,\\
     \Norm{\fK_t^{(r)}(\cdot)}_{\infty}&\lesssim   \frac{1}{m^{r/2-1}} \left(\frac{\xi_{\infty,r}(t)^2}{m}\right)^s\lesssim    \frac{1}{m^{r/2-1}} \left(\frac{\left(c {\LogLn}^{C}\sqrt{m}\right)^2}{m}\right)^r \lesssim \frac{{\LogLn}^{2rC}}{m^{r/2-1}}.
\end{align*}
\end{proof}
Now we sketch the proofs for Theorem \ref{theorem...vary at 1/m} and Corollary \ref{corollary....for thm}. Details can be found in Appendix \ref{appendix...subsection for proof of thm and cor}.
\begin{proof}[Sketch of the Proof of Theorem  \ref{theorem...vary at 1/m}]
Since there exists  $\frac{1}{L^2}$ scaling in some kernels, we use $C(r,L)$ to denote the `effective terms' in each kernel.
We denote $\fG_t^{[L+1]}\left(\cdot\right)$ by $\fG_t^{[2]}\left(\cdot\right),$ i.e., $\fG_t^{(2)}(\cdot):=\fG_t^{[L+1]}(\cdot) ,$ it's natural for us to get that $C(2,L)=\fO(1).$  

 Next, we apply the replacement rule, all the possible terms generated from $\fG_t^{(2)}(\cdot)$ are
\begin{align*}
&\fG_t^{(2)}(\vx_{\alpha_{1}},\vx_{\alpha_{2}})=\left<\vx_{\alpha_1}^{[L]},\vx_{\alpha_2}^{[L]}\right>\to \fG_t^{(3)}(\vx_{\alpha_{1}},\vx_{\alpha_{2}}, \vx_{\beta})\\
 &\fG_t^{(3)}(\vx_{\alpha_{1}},\vx_{\alpha_{2}}, \vx_{\beta})={\frac{c_\sigma}{m}}\underbrace{\left< \mathrm{diag}\left(\mE_{t,\alpha_1}^{[2:L]} \vsigma^{(1)}_{[1]}(\vx_{\alpha_1})\vsigma^{(1)}_{[1]}(\vx_{\beta})\left(\mE_{t,\beta}^{[2:L]}\right)^{\T}\va_t \right)\vone,\vx_{\alpha_2}^{[L]}\right> \left<\vx_{\alpha_1},\vx_\beta  \right>}_{\textrm{I}}\\
&+\sum_{k=2}^{L} \frac{c_{\mathrm{res}}^2}{L^2 {m}}\underbrace{\left< \mathrm{diag}\left(\mE_{t,\alpha_1}^{[(k+1):L]}\vsigma^{(1)}_{[k]}(\vx_{\alpha_1})\vsigma^{(1)}_{[k]}(\vx_{\beta})\left(\mE_{t,\beta}^{[(k+1):L]}\right)^{\T}\va_t  \right) \vone,\vx_{\alpha_2}^{[L]}\right> \left<\vx_{\alpha_1}^{[k-1]},\vx_\beta^{[k-1]}  \right>}_{\mathrm{II}}\\
&+{\frac{c_\sigma}{m}}\left< \mathrm{diag}\left(\mE_{t,\alpha_2}^{[2:L]} \vsigma^{(1)}_{[1]}(\vx_{\alpha_2})\vsigma^{(1)}_{[1]}(\vx_{\beta})\left(\mE_{t,\beta}^{[2:L]}\right)^{\T}\va_t \right)\vone,\vx_{\alpha_1}^{[L]}\right> \left<\vx_{\alpha_2},\vx_\beta  \right>\\
&+\sum_{k=2}^{L} \frac{c_{\mathrm{res}}^2}{L^2 {m}}\left< \mathrm{diag}\left(\mE_{t,\alpha_2}^{[(k+1):L]}\vsigma^{(1)}_{[k]}(\vx_{\alpha_2})\vsigma^{(1)}_{[k]}(\vx_{\beta})\left(\mE_{t,\beta}^{[(k+1):L]}\right)^{\T}\va_t  \right) \vone,\vx_{\alpha_1}^{[L]}\right> \left<\vx_{\alpha_2}^{[k-1]},\vx_\beta^{[k-1]}  \right>.
\end{align*}
 we have
$C(3,L)=\fO\left( 2\left(1+\frac{L-1}{L^2}\right)\right)=\fO\left(1+\frac{1}{L}\right).$

Finally for $\fG_t^{(4)}(\cdot),$ by symmetry, we are only going to analyze  terms I  and II.
Since there are at most $(2L+2)$ symbols in term I to be replaced, and by the  replacement rules, each replacement will bring about up to  $(L+1)$ many terms. For term II, for each summand, there are also at most $(2L+2)$ symbols  to be replaced. Since there are $L-1$ summands in II, and  each replacement will bring about up to  $(L+1)$ many terms. we have that  
$$C(4,L)=\fO\left( 2\left((2L+2)(L+1)+ \frac{1}{L^2}(L-1)(2L+2)(L+1) \right)\right)=\fO\left(L^2\right).$$

It holds that for time $0\leq t\leq \sqrt{m}/{\LogLn}^{C'}$
\begin{align*}
    \Abs{\partial_t\fG_t^{(3)}\left(\vx_{\alpha_1},\vx_{\alpha_2},\vx_{\alpha_3}\right)}& \leq   \Norm{\fG_t^{(4)}(\cdot)}_{\infty} \sqrt{ R_S(\vtheta_0)} \leq C(4,L) \frac{{\LogLn}^C}{m},\\
    \Abs{\fG_t^{(3)}\left(\vx_{\alpha_1},\vx_{\alpha_2},\vx_{\alpha_3}\right)} 
    &\leq  \Norm{\fG_0^{(3)}\left(\cdot\right)}_{\infty}+tC(4,L)\frac{{\LogLn}^C}{m}. 
\end{align*}
Finally, we need to make estimate on $\Norm{\fG_0^{(3)}\left(\cdot\right)}_{\infty}.$ Each term in $\fG_0^{(3)}\left(\cdot\right)$ is of the form $\frac{c}{m}\left<\mB \va_0, \vx_{\alpha_1}^{[L]} \right>\left<\vx_{\alpha_2}^{[l]},\vx_\beta^{[l]}  \right> $
where $\mB$ is some specific matrix that changes from term to term. After taking conditional expectation up to the random variable $\va_0,$   we have with high probability 
\begin{align}
    \frac{c}{m}\left<\va_0, \mB^{\T}\vx_{\alpha_1}^{[L]} \right>\left<\vx_{\alpha_2}^{[l]},\vx_\beta^{[l]}  \right>\leq c\frac{{\LogLn}^C}{m}. \label{rewrite...final}
\end{align}
Consequently,  for time $0\leq t\leq \sqrt{m}/{\LogLn}^{C'}$
\begin{align*}
    \Abs{\fG_t^{(3)}\left(\vx_{\alpha_1},\vx_{\alpha_2},\vx_{\alpha_3}\right)} 
    &\leq C(3,L)\frac{{\LogLn}^C}{m}+tC(4,L)\frac{{\LogLn}^C}{m},\\
    \Abs{\partial_t\fG_t^{(2)}\left(\vx_{\alpha_1},\vx_{\alpha_2}\right)}&\leq \Norm{\fG_t^{(3)}( \cdot)}_{\infty}  \sqrt{R_S(\vtheta_0)} \leq  \left(C(3,L) +tC(4,L)\right)\frac{{\LogLn}^C}{m}, 
\end{align*}
which finishes the proof of Theorem \ref{theorem...vary at 1/m}.
\end{proof}

\begin{proof}[Sketch of the Proof of Corollary \ref{corollary....for thm}]
If $m=\Omega\left( \left(\frac{n}{\lambda_0}\right)^{2+\eps}\right)$,  with high probability w.r.t random initialization, $\lambda_{\min}\left[\fK_0^{(2)}\left(\vx_{\alpha},\vx_{\beta}\right)\right]_{1\leq \alpha,\beta\leq n}>
    \lambda_{\min}\left(\mG^{[L+1]}(0)\right)>\frac{3 \lambda_0}{4},$
by setting $\lambda=\frac{3 \lambda_0}{4},$ we finish  the proof of \eqref{least eigenvalue}.

Concerning the change of the least eigenvalue of the NTK, from the Sketch Proof of Theorem \ref{theorem...vary at 1/m}, for time  $0\leq t\leq \sqrt{m}/{\LogLn}^{C'},$
\begin{align*}
    \Norm {\left(\fG_t^{(2)}-\fG_0^{(2)}\right)(\cdot)}_{2\to 2} \leq  &\Norm {\left(\fG_t^{(2)}-\fG_0^{(2)}\right)(\cdot)}_{\mathrm{F}} \\
    \leq n &\Norm {\left(\fG_t^{(2)}-\fG_0^{(2)}\right)(\cdot)}_{\infty} \leq  nt\left(C(3,L) +tC(4,L)\right) \frac{{\LogLn}^{C}}{m},
\end{align*}
  set $t^{*}$ satisfying :
\begin{align}
   C(4,L) (t^*)^2 + C(3,L)t^* &= \frac{\lambda m}{2{\LogLn}^C n},\label{quadratic}
\end{align}
after solving \eqref{quadratic}  
\begin{align*}
    t^*= \frac{-C(3,L)+\sqrt{\left(C(3,L)\right)^2+2 C(4,L)\frac{\lambda m}{{\LogLn}^{C}n}}}{2 C(4,L) }\geq \frac{1}{2}\sqrt{\frac{\lambda m}{C(4,L){\LogLn}^{C}n}}. 
\end{align*}
Let $\Bar{t}:=\inf\left\{t: \lambda_{\min}\left[\fK_t^{(2)}\left(\vx_{\alpha},\vx_{\beta}\right)\right]_{1\leq \alpha,\beta\leq n}\geq \lambda/2 \right\},$ naturally we have $ t^*\leq \Bar{t}.$ Using \eqref{eq for thm...neural tangent kernel .... order 2} we have for any $0\leq t \leq \Bar{t},  R_S(\vtheta_t)\leq \exp\left(-{\lambda t}/{n}\right) R_S(\vtheta_0).$

Set  $R_S(\vtheta_t)=\eps$, it takes time $t\leq \left({n}/{\lambda}\right) \ln ({R_S(\vtheta_0)}/{\eps})$ for loss  $R_S(\vtheta_t)$ to reach accuracy $\eps,$ hence if $ t\leq \left({n}/{\lambda} \right)\ln \left({R_S(\vtheta_0)}/{\eps}\right)\leq t^*\leq \Bar{t},$ then   width $m$ is required to be
\begin{equation}\label{widtrh}
    \frac{n}{\lambda}\ln \left(\frac{R_S(\vtheta_0)}{\eps}\right)\leq  \frac{1}{2}\sqrt{\frac{\lambda m}{C(4,L){\LogLn}^{C}n}}.
\end{equation}
thus we have 
\begin{align*}
 m\geq C(4,L)\left(\frac{n}{\lambda}\right)^3\LogLn^{C}\ln\left(\frac{R_S(\vtheta_0)}{\eps}\right)^2,
\end{align*}
since $C(4,L)=\fO\left(L^2\right),$ we finish the proof. 
\end{proof}

\section{Discussion}\label{section....Discussion}

In this paper, we show that the GD on ResNet can obtain zero training loss, and its training dynamic is given by an infinite hierarchy of ordinary differential equations, i.e., the NTH, which makes it possible to  study the change of the NTK directly for deep neural networks. Our proof builds on a careful analysis of the least eigenvalue of randomly initialized Gram matrix, and the uniform upper bound on  kernels of higher order in the  NTH.

We list out some future directions for  research:
\begin{itemize}
\item  The NTH is an infinite sequence of relationship. However, Huang and Yau showed that under certain conditions  on the  width and the data set dimension, the NTH can be truncated and the truncated version of  NTH is still able to
approximate the original dynamic up to any precision. We believe that for ResNet, such technical conditions   can be loosened based on our result.
    \item In Corollary \ref{corollary....for thm}, the dependence of $m$ on the depth $L$ is quadratic, we believe that the dependence can be reduced even further. We conjecture that $m$ is independent of $L.$ 
    \item In this paper, we focus on the GD, and we believe that it can be extended to SGD, while maintaining the linear convergence rate.
    \item We focus on the training loss, but does not address the test loss. To further investigate the generalization power of ResNet, we believe some Apriori estimate for the generalization error of ResNet may be useful~\cite{ma2019priori}.
\end{itemize}

\bibliographystyle{siam}
\bibliography{ResNet.bib}
\newpage
\appendix
\section{ Estimates on the Kernel}\label{appen....initial estimate on kernel}
\subsection{Structure on Hierarchical Sets of Kernel Expressions}\label{appen subsection...kernel structure}
Since we have mentioned the replacement rules in Section \ref{subsection....replacemetn rules},~we haven't rigorously justified it yet. Hence we use Proposition \ref{prop...A.1}  to shed light on the structures of the elements in $\sA_r$, and consequently on the structures of each term in kernel~$\fK_t^{(r)}(\vx_{\alpha_1},\vx_{\alpha_2},\dots,\vx_{\alpha_r} )$.
\begin{prop}\label{prop...A.1}
For any vector $\vv(t)\in \sA_r$, the new vector obtained from $\vv(t)$ by performing the replacement rules are the sum of terms of the following forms:
\begin{align*}
(a).\frac{C}{\sqrt{m}}&\vv'(t): \vv'(t)\in \sA_r ,\\
(b).\frac{C}{\sqrt{m}}&\vv'(t)\frac{\left<\vp,\vq\right>}{m}: \vv'(t)\in  \sA_{r+1}, \vp,\vq\in\sA_0,\\
(c).\frac{C}{\sqrt{m}}&\vv'(t)\frac{\left<\sqrt{m} \vx_{\alpha},\sqrt{m} \vx_{\beta}\right>}{m}: \vv'(t)\in \sA_{r+1}, \vp,\vq\in\sA_0,\\
(d).\frac{C}{\sqrt{m}}&\vv'(t)\frac{\left<\vp,\vq\right>}{m}: \vv'(t)\in \sA_{r-s+1}, \vp\in\sA_s,\vq\in\sA_0, \ \text{for some } \ s\geq 1,\\
(e).\frac{C}{\sqrt{m}}&\vv'(t)\frac{\left<\vp,\vq\right>}{m}: \vv'(t)\in \sA_{s}, \vp\in\sA_{r-s+1},\vq\in\sA_0, \ \text{for some} \ s\geq 1.
\end{align*}

\end{prop}
\begin{proof}
The proof comes as follows. Note that the constant $C$   listed out below might keep changing from term to term.

Since $\va_t$  appears only at the position $\ve_0$, if $\vv(t)\in \sA_r,$  based on the replacement rule
\begin{align*}
    \vv(t)=\ve_s\ve_{s-1}\dots \ve_1\va_t\to \widetilde{\vv}(t)=\frac{1}{\sqrt{m}}\ve_s\ve_{s-1}\dots \ve_1 \sqrt{m} \vx_\beta^{[L]}=\frac{1}{\sqrt{m}}\vv'(t),
\end{align*}
then $\vv'(t)=\ve_s\ve_{s-1}\dots \ve_1 \sqrt{m} \vx_\beta^{[L]} \in \sA_r.$ 

Similarly, $\sqrt{m}\vx_\alpha^{[l]}$  also appears only at   $\ve_0$, then if $\vv(t)\in \sA_r,$ by the replacement rule
\begin{align*}
\vv(t)&=\ve_s\ve_{s-1}\dots \ve_1\sqrt{m}\vx_\alpha^{[l]} \to \widetilde{\vv}(t),\\
\widetilde{\vv}(t)&=\sum_k \frac{C}{\sqrt{m}}\ve_s\ve_{s-1}\dots \ve_1 \  \mathrm{diag}(\vf_k) \vone \frac{\left<\sqrt{m}\vx_\alpha^{[k]},\sqrt{m}\vx_\beta^{[k]}  \right>}{m}=\sum_k\frac{C}{\sqrt{m}}\vv'_k(t)\frac{\left<\sqrt{m}\vx_\alpha^{[k]},\sqrt{m}\vx_\beta^{[k]}  \right>}{m},
\end{align*}
given that $\vf_k\in \sA_0,$ then $\vv'_k(t) \in \sA_{r+1}.$

Since $\vsigma_{[l]}^{(u)}(\vx_\alpha)$ only appears at the starting or the middle position, i.e., $\ve_j, \  j\geq 1$. For $u=1,$ $\vsigma_{[l]}^{(1)}(\vx_\alpha)$ has no diag operations accompanied with it, and any vector $\vv(t)\in\sA_r$ could contain $\vsigma_{[l]}^{(1)}(\vx_\alpha),$ for $r\geq 0$
\begin{align*}
\vv(t)&=\ve_s\dots\ve_{j+1}\vsigma_{[l]}^{(1)}(\vx_\alpha)\ve_{j-1}\dots\ve_0 \to \widetilde{\vv}(t),\\
\widetilde{\vv}(t)&=\frac{C}{\sqrt{m}} \ve_s\dots\ve_{j+1}\vsigma_{[l]}^{(2)}(\vx_\alpha)\mathrm{diag} \left( \vf_1\right) \ve_{j-1}\dots\ve_0 \frac{\left<\vp_1,\vq_1 \right>}{m}\\
&+\sum_{k}  \frac{C}{\sqrt{m}} \ve_s\dots\ve_{j+1}\vsigma_{[l]}^{(2)}(\vx_\alpha)\mathrm{diag} \left( \frac{\mW_t^{[l]}}{\sqrt{m}}\vf_k\right) \ve_{j-1}\dots\ve_0 \frac{\left<\vp_k,\vq_k \right>}{m}\\
&=\sum_{l}\frac{C}{\sqrt{m}}\vv'_l(t) \frac{\left<\vp_{l},\vq_{l} \right>}{m},
\end{align*}
 since $\vf_{k}\in \sA_{0},$ then $\vv'_l(t)\in \sA_{r+1},$ and $\vp_{l},\vq_{l}\in\sA_0.$

For $u\neq 1,$ $\vsigma_{[l]}^{(u)}(\vx_\alpha)$ has at most $u-1$ diag operations behind it, and only vector $\vv(t)\in\sA_r$ could contain $\vsigma_{[l]}^{(u)}(\vx_\alpha),$ for $r\geq u-1.$ 
\begin{align*}
\vv(t)&=\ve_s\dots\ve_{j+1}\ve_j\ve_{j-1}\dots\ve_0 \to \widetilde{\vv}(t),\\
\text{with} \ \ve_j&=\vsigma_{[l]}^{(u)}(\vx_\alpha) \mathrm{diag} \left(\left(\frac{\mW_t^{[l]}}{\sqrt{m}}\right)^{Q_1} \vg_1\right)\dots\mathrm{diag} \left(\left(\frac{\mW_t^{[l]}}{\sqrt{m}}\right)^{Q_{u-1}} \vg_{u-1}\right)\left(\frac{c_{\mathrm{res}}}{L}\frac{\mW_t^{[l]}}{\sqrt{m}}\right)^{Q_{u}},\\
\text{or} \ \ve_j&=\left(\frac{c_{\mathrm{res}}}{L}\frac{\left(\mW_t^{[l]}\right)^\T}{\sqrt{m}}\right)^{Q_{u}}\vsigma_{[l]}^{(u)}(\vx_\alpha) \mathrm{diag} \left(\left(\frac{\mW_t^{[l]}}{\sqrt{m}}\right)^{Q_1} \vg_1\right)\dots\mathrm{diag} \left(\left(\frac{\mW_t^{[l]}}{\sqrt{m}}\right)^{Q_{u-1}} \vg_{u-1}\right),
\end{align*}
and after applying replacement rules on $\ve_j\to \ve_j^{'}$, 
\begin{align*}
 \ve_j^{'}&=\frac{C}{\sqrt{m}} \vsigma_{[l]}^{(u+1)}(\vx_\alpha) \mathrm{diag}\left(\vf_1\right)\mathrm{diag} \left(\left(\frac{\mW_t^{[l]}}{\sqrt{m}}\right)^{Q_1} \vg_1\right)\dots\left(\frac{c_{\mathrm{res}}}{L}\frac{\mW_t^{[l]}}{\sqrt{m}}\right)^{Q_{u}}\frac{\left<\vp_1,\vq_1 \right>}{m}\\
 +&\sum_{k}  \frac{C}{\sqrt{m}}\vsigma_{[l]}^{(u+1)}(\vx_\alpha)\mathrm{diag} \left( \left(\frac{\mW_t^{[l]}}{\sqrt{m}}\right)^{Q_{0}}\vf_k\right)\mathrm{diag} \left(\left(\frac{\mW_t^{[l]}}{\sqrt{m}}\right)^{Q_1} \vg_1\right)\dots\left(\frac{c_{\mathrm{res}}}{L}\frac{\mW_t^{[l]}}{\sqrt{m}}\right)^{Q_{u}}\frac{\left<\vp_k,\vq_k \right>}{m},\\
\text{or} \ \ve_j^{'}&=\frac{C}{\sqrt{m}} \left(\frac{c_{\mathrm{res}}}{L}\frac{\left(\mW_t^{[l]}\right)^\T}{\sqrt{m}}\right)^{Q_{u}}\vsigma_{[l]}^{(u+1)}(\vx_\alpha) \mathrm{diag}\left(\vf_1\right)\dots\mathrm{diag} \left(\left(\frac{\mW_t^{[l]}}{\sqrt{m}}\right)^{Q_{u-1}} \vg_{u-1}\right)    \frac{\left<\vp_1,\vq_1 \right>}{m}\\
+\sum_{k} & \frac{C}{\sqrt{m}}\left(\frac{c_{\mathrm{res}}}{L}\frac{\left(\mW_t^{[l]}\right)^\T}{\sqrt{m}}\right)^{Q_{u}}\vsigma_{[l]}^{(u+1)}(\vx_\alpha)\mathrm{diag} \left( \left(\frac{\mW_t^{[l]}}{\sqrt{m}}\right)^{Q_{0}}\vf_k\right)\dots\mathrm{diag} \left(\left(\frac{\mW_t^{[l]}}{\sqrt{m}}\right)^{Q_{u-1}} \vg_{u-1}\right)  \frac{\left<\vp_k,\vq_k \right>}{m},
\end{align*}
then $$\widetilde{\vv}(t)=\sum_{l}\frac{C}{\sqrt{m}}\vv'_l(t) \frac{\left<\vp_{l},\vq_{l} \right>}{m},$$ 
since $\vf_{k}\in \sA_{0},$ then $\vv'_l(t)\in \sA_{r+1},$ and $\vp_{l},\vq_{l}\in\sA_0.$

Since $\frac{\mW_t^{[l]}}{\sqrt{m}}$  only appears at the starting or the middle position $\ve_j$, we have that if $\vv(t)\in \sA_r,$ then based on the replacement rules
\begin{align*}
\vv(t)&=\ve_s\ve_{s-1}\dots\ve_{j+1}\frac{\mW_t^{[l]}}{\sqrt{m}}\ve_{j-1}\dots \ve_1\ve_0 \to \widetilde{\vv}(t),
\\
\widetilde{\vv}(t)&=\frac{C}{{m}}\ve_s\ve_{s-1}\dots\ve_{j+1}\mathrm{diag} (\vg)\ \vone  \otimes (\vx^{[l-1]}_\beta)^{\T}\ve_{j-1}\dots \ve_1\ve_0 \\
&= \frac{C}{\sqrt{m}}\ve_s\ve_{s-1}\dots\ve_{j+1}\mathrm{diag} (\vg)\ \vone \frac{\left<\ve_{j-1}\dots \ve_1\ve_0, \sqrt{m} \vx_\beta ^{[l-1]}\right>}{m}\\
&=\frac{C}{\sqrt{m}} \vv' (t) \frac{\left<\vp,\vq\right>}{m},
\end{align*}
with $\vv' (t)\in\sA_{r-s+1},$ and $\vp\in\sA_s, \vq\in\sA_0,$ for some $s\geq 1.$

Similarly for $\frac{(\mW_t^{[l])^{\T}}}{\sqrt{m}},$
\begin{align*}
\vv(t)&=\ve_s\ve_{s-1}\dots\ve_{j+1}\frac{(\mW_t^{[l]})^{\T}}{\sqrt{m}}\ve_{j-1}\dots \ve_1\ve_0 \to \widetilde{\vv}(t)\\
\widetilde{\vv}(t)&=\frac{C}{{m}}\ve_s\ve_{s-1}\dots\ve_{j+1} \vx_\beta^{[l-1]}\  \otimes\vone ^{\T}\mathrm{diag} (\vg)\ve_{j-1}\dots \ve_1\ve_0 \\
&= \frac{C}{\sqrt{m}} \ve_s\ve_{s-1}\dots\ve_{j+1} \sqrt{m}\vx_\beta^{[l-1]}\ \frac{\left<\mathrm{diag} (\vg)\ve_{j-1}\dots \ve_1\ve_0,\vone\right>}{m}\\
&=\frac{C}{\sqrt{m}}\vv' (t)\frac{\left<\vp,\vq\right>}{m},
\end{align*}
with $\vv' (t)\in\sA_{r-s},$ and $\vp\in\sA_{r-s+1}, \vq\in\sA_0,$ for some $s\geq 1.$ Since $\mE_{t,\alpha}^{[l]}$  is situations combined with $\frac{\mW_t^{[l]}}{\sqrt{m}}$ and  $\vsigma_{[l]}^{(1)}(\vx_\alpha),$ so we will skip the analysis.
\end{proof}
 
From the discussion above, if we apply Proposition \ref{prop...A.1} to $\fK_t^{(r)}(\vx_{\alpha_1},\vx_{\alpha_2})$ inductively $(r-1)$ times,  for each term in kernel $\fK_t^{(r)}(\vx_{\alpha_1},\vx_{\alpha_2},\dots,\vx_{\alpha_r} )$, it takes the form:
\begin{equation}\label{equation for generation....of the kernel of order p...in the appendix}
    \frac{1}{m^{r/2-1}}\prod_{j=1}^s\frac{\left<\vv_{2j-1}(t),\vv_{2j}(t)\right>}{m}, \ 1\leq s\leq r, \ \ \vv_i(t)\in \sA_0\cup\sA_1\cup\dots\cup \sA_{r-2}.
\end{equation}
\subsection{Apriori $L^2$ bounds for expressions in $\sA_0$}\label{Appen subsection....Aprioir L2 bounds for sets in A0}
We begin with an estimate on the empirical risk   $R_S(\vtheta_t)$.
\begin{prop}\label{proposition...on the apriori loss at t=0}
Under Assumption \ref{Assump...Assumption on activation functions} and \ref{Assump... ont he imput of the  samples}, we have for $t\geq 0,$ 
\begin{equation}\label{inequality in Proposition on apriori loss...decay of empirical risk}
     R_S(\vtheta_t)\leq  R_S(\vtheta_0) \sim \fO(1).
\end{equation}
\end{prop}
\begin{proof}
We  get  inequality \eqref{inequality in Proposition on apriori loss...decay of empirical risk} by  non-negative definiteness of kernel $\fK_t^{(2)}(\cdot)$. From \eqref{eq for thm...neural tangent kernel .... order 2}, we obtain that
\begin{equation}\label{eq for proof of prop apriori loss...decay of empirical risk}
    \partial_t\sum_{\alpha=1}^n\Norm{f_\alpha(t)-y_\alpha}_2^2=-\frac{2}{n}\sum_{\alpha,\beta=1}^n \fK_t^{(2)}(\vx_\alpha,\vx_\beta)(f_\alpha(t)-y_\alpha)(f_\beta(t)-t_\beta)\leq 0,
\end{equation}
hence
\begin{equation*}
    R_S(\vtheta_t)\leq  R_S(\vtheta_0),
\end{equation*}
which finish the proof of Proposition \ref{proposition...on the apriori loss at t=0}.
\end{proof}

Our next proposition is mainly on the spectral property of the skip-connection matrices. This proposition is similar to Proposition $\mathrm{B.1.}$ in \cite{Huang2019Dynamics}.
\begin{prop}\label{proposition..A priori spectral property random matrix and a t}
Under Assumptions \ref{Assump...Assumption on activation functions} and \ref{Assump... ont he imput of the  samples},  we define $\xi(t)$  as follows
\begin{align}
    \xi(t)=\sup_{0\leq t'\leq t}\max \Bigg\{1, \frac{1}{\sqrt{m}}  \Big\{&\Norm{\mW_{t'}^{[2]}}_{2\to2},\Norm{\left(\mW_{t'}^{[2]}\right)^{\T}}_{2\to2},\dots\nonumber\\
    \dots,  &\Norm{\mW_{t'}^{[L]}}_{2\to2},\Norm{\left(\mW_{t'}^{[L]}\right)^{\T}}_{2\to2}, \Norm{\va_{t'}}_2\Big\}\Bigg\},\label{def...on the xi t...}
\end{align}
then with high probability w.r.t the random initialization, 
 for $t\lesssim  \sqrt{m}$ 
\begin{equation}\label{inequality....on the xi t}
    \xi(t)\leq c_{w,t}.
\end{equation}
where $c_{w,t}>2$ is a constant independent of the depth of the network $L$.

Moreover for $t\lesssim\sqrt{m}$, $c_{w,t}$ has a uniform upper bound in $t,$ i.e.,
\begin{equation}
    c_{w,t}\leq \bar{c},
\end{equation}
where $\bar{c}$ is  independent of   depth  $L$ and time $t.$
\end{prop}
\begin{proof}
For the purpose  of proving the proposition, we shall state two lemmas, Lemma~\ref{lemma...sigular value of random matrix} and \ref{lemma...on chi2 distribution}. Lemma \ref{lemma...sigular value of random matrix} is given out as Lemma $\mathrm{G}.2.$ in Du et al.\cite{Du2018Gradient},  also  consequence of the results in \cite{vershynin2010introduction}.
\begin{lem}\label{lemma...sigular value of random matrix}
Given a matrix $\mW\in \sR^{m\times m}$ with each entry $W_{i,j}\sim \fN(0, 1),$ then  with probability at least $1-\exp\left(-\frac{(c^{'}_{w,0}-2)^2 m}{2}\right),$ the following holds
\begin{equation}\label{eq... eigenvalue of a random given matrices}
    \Norm{\mW}_{2\to 2}\leq {c^{'}_{w,0}}\sqrt{m},
\end{equation}
where $c^{'}_{w,0}>2$ is a constant.
\end{lem}
\begin{rmk}
 This event is  an event that holds with high probability.
\end{rmk}
Next concerning the term $\frac{1}{\sqrt{m}} \Norm{\va_0}_2$, we shall state a lemma on the tail bound of the chi-square distribution, using Lemma 1 from \cite{Laurent2000AdaptiveEstimationQuadratic} 
\begin{lem}\label{lemma...on chi2 distribution}
If $Z\sim \chi^2(m)$, then we have a tail bound 
\begin{equation}\label{equation for lemma...tail bound of chi square}
    \Prob\left(Z \geq m + 2\sqrt{mx} + 2x\right) \leq e^{-x}.
\end{equation}
\end{lem}
\begin{rmk}
 This event is also an event that holds with high probability.
 
\end{rmk}

Then if we write $2tm = m + (2t - 1)m$, letting $x=\frac{mt}{10},$ we can obtain that 
\begin{equation*}
    \Prob\left(\Norm{\va_0}_2^2\geq m + 2m\left(\sqrt{t/10}+t/10\right)\right) \leq\exp(-tm/10)),
\end{equation*}
and for $t\geq 1$, we have $2t - 1 \geq 2 \left(\sqrt{t/10} + t/10\right).$ Thus, if we choose $t$ properly, we  see that such event $$\frac{1}{\sqrt{m}} \Norm{\va_0}_2\leq c^{''}_{w,0}$$ holds with high probability.
Hence, for $t=0,$ $\xi(0)\leq \max\left\{1,c^{'}_{w,0},c^{''}_{w,0}\right\}.$ We set $c_{w,0}$ as $c_{w,0}=\max\left\{1,c^{'}_{w,0},c^{''}_{w,0}\right\},$ then
\begin{equation}\label{inequation of proof of proposition.....initialization}
    \xi(0)\leq c_{w,0}.
\end{equation}
In the following we are going to show the upper bound of $\partial_t\xi(t).$ In order to do that, we need to estimate $L^2$ bound on each output layer. For $l=1,$
\begin{align}
    \Norm{\vx^{[1]}}_2&=\sqrt{\frac{c_\sigma}{{m}}}\Norm{\sigma(\mW_t^{[1]}\vx)}_2\leq \sqrt{c_\sigma}\left(\Abs{\sigma(0)}+\frac{C_L}{\sqrt{m}}\Norm{\mW_t^{[1]}\vx}_2\right) \nonumber\\
    &\leq \sqrt{c_\sigma} C_L\left( 1+ \xi(t) \Norm{\vx}_2\right)\leq C\xi(t),\label{equation of proof of proposition...output of each layer at initialiation}
\end{align}
and for $2\leq l \leq L,$
\begin{align}
     \Norm{\vx^{[l]}}_2&\leq\Norm{\vx^{[l-1]}}_2+
     \frac{c_{\mathrm{res}}}{L\sqrt{m}}\Norm{\sigma\left(\mW_t^{[l]}\vx^{[l-1]}\right)}_2\nonumber \\
     &\leq \Norm{\vx^{[l-1]}}_2+ \frac{c_{\mathrm{res}}}{L}\left(\Abs{\sigma(0)}+C_L\xi(t)\Norm{\vx^{[l-1]}}_2\right)\nonumber\\
     &\leq  \Norm{\vx^{[l-1]}}_2+ \frac{c_{\mathrm{res}}}{L}\left(C_L+C_L\xi(t)\Norm{\vx^{[l-1]}}_2\right)\nonumber\\
     &\leq \left(1+\frac{2c_{\mathrm{res}}}{L}\xi(t) \right)\Norm{\vx^{[l-1]}}_2.\label{inequation of proof of proposition.....initialization on the value of each output l;ayers}
\end{align}
Hence we can obtain an inductive relation on the $2$-norm of $\vx^{[l]}.$
\begin{equation}\label{ineqlity.... a universal inequlaity on the output value of each layer}
\Norm{\vx^{[l]}}_2\leq C \left(1+\frac{2c_{\mathrm{res}}}{L}\xi(t) \right)^{l-1}\xi(t).
\end{equation}
Based on \eqref{eqgroup...dynamic for parameter a_t}, \eqref{eqgroup...dynamic for parameter W_H}, \eqref{eqgroup...dynamic for parameter W_l} and \eqref{eqgroup...original dynamic for parameter W_1}, ~combined with Proposition \ref{proposition...on the apriori loss at t=0}  
\begin{align}
    \partial_t\Norm{\mW_t^{[l]}}_{2\to 2}&\leq \frac{1}{n}\sum_{\beta=1}^n\frac{C}{\sqrt{m}} \Norm{\vsigma^{(1)}_{[l]}(\vx_{\beta}) \left(\mE_{t,\beta}^{[(l+1):L]}\right)^{\T}\va_t}_2  \Norm{\vx_{\beta}^{[l-1]}}_2 \Abs{f_{\beta}(t)-y_{\beta}}\nonumber\\
    &\leq \frac{1}{n}\sum_{\beta=1}^n{C C_L}\left(1+\frac{c_{\mathrm{res}}C_L}{L}\xi(t) \right)^{L-l}\xi(t)\left(1+\frac{2c_{\mathrm{res}}}{L}\xi(t) \right)^{l-1}\xi(t)\Abs{f_{\beta}(t)-y_{\beta}}\nonumber\\
    &\leq C\left(1+\frac{2c_{\mathrm{res}}}{L}\xi(t)\right)^{L-1}\xi(t)^2\sqrt{\frac{1}{n}\sum_{\beta=1}^n \Norm{f_{\beta}(t)-y_{\beta}}_2^2}\nonumber\\
     &\leq C\left(1+\frac{2c_{\mathrm{res}}}{L}\xi(t)\right)^{L-1}\xi(t)^2\sqrt{R_S\left(\vtheta_0\right)}\nonumber\\
    &\leq C\left(1+\frac{2c_{\mathrm{res}}}{L}\xi(t)\right)^{L-1}\xi(t)^2\leq C\exp\left(2  c_{\mathrm{res}} \xi(t)\right) \ \xi(t)^2, \label{inequality for proof of proposition....for W_t l }\\
    \partial_t\Norm{\va_t}_2&\leq\frac{1}{n}\sum_{\beta=1}^n\Norm{\vx_\beta^{[L]}}_2\Abs{f_{\beta}(t)-y_{\beta}}\leq C\left(1+\frac{2c_{\mathrm{res}}}{L}\xi(t)\right)^{L-1}\xi(t)\sqrt{\frac{1}{n}\sum_{\beta=1}^n \Norm{f_{\beta}(t)-y_{\beta}}_2^2}\nonumber\\
    &\leq C\left(1+\frac{2c_{\mathrm{res}}}{L}\xi(t)\right)^{L-1}\xi(t)\sqrt{R_S\left(\vtheta_0\right)}\nonumber\\
    &\leq C \left(1+\frac{2c_{\mathrm{res}}}{L}\xi(t)\right)^{L-1}\xi(t)\leq C\exp(2c_{\mathrm{res}}\xi(t)) \ \xi(t). \label{inequality.for proof of propositio....for a_t }
\end{align}
Based on  \eqref{inequality for proof of proposition....for W_t l } and \eqref{inequality.for proof of propositio....for a_t }, 
we have 
\begin{equation*}
    \sqrt{m} \ \partial_t \xi(t) \leq  C \exp(2c_{\mathrm{res}}\xi(t))\xi^2(t),
\end{equation*}
we can obtain an integration inequality,
\begin{equation}\label{equation....intergraton}
    \int_{\xi(0)}^{\xi(t)} \frac{\diff u}{\exp(2 c_{\mathrm{res}}u)u^2} \leq \frac{Ct}{\sqrt{m}}.
\end{equation}
Hence the integration term on the LHS of  \eqref{equation....intergraton} is 

\begin{align*}
    \int_{\xi(0)}^{\xi(t)} \frac{\diff u}{\exp(2 c_{\mathrm{res}}u)u^2}&\geq \frac{1}{\exp(2 c_{\mathrm{res}}\xi(t))} \int_{\xi(0)}^{\xi(t)} \frac{\diff u}{u^2}\\
    &=\frac{1}{\exp((2 c_{\mathrm{res}}\xi(t))}\left(\frac{1}{\xi(0)}-\frac{1}{\xi(t)}\right)\\
    &\geq\frac{1}{\exp(2 c_{\mathrm{res}}\xi(t))}\left(\frac{1}{c_{w,0}}-\frac{1}{\xi(t)}\right) .
\end{align*}
We shall notice for the single variable function $f(z)$ 
\begin{equation*}
    f(z)=\frac{1}{\exp(2 c_{\mathrm{res}} z)}(\frac{1}{c_{w,0}}-\frac{1}{z}),
\end{equation*}
 maximum  of  $f(z)$  can be achieved at point 
$$z_0=\frac{c_{w,0}+\sqrt{c_{w,0}^2+2{c_{w,0}}/{ c_{\mathrm{res}}}}}{2},$$ 
and $f(z)$ is monotone increasing in  the interval $\left[c_{w,0}, z_0\right].$ Thus, if we choose time $t$ properly, say $t\leq c \sqrt{m}$, $c$ being small enough, the following holds
$$\xi(t)\leq \frac{c_{w,0}+\sqrt{c_{w,0}^2+2{c_{w,0}}/{ c_{\mathrm{res}}}}}{2}.$$
In other words, if $t\leq c\sqrt{m}$ for some small enough $c>0,$ we have
$$\xi(t)\leq c_{w,t}\leq \frac{c_{w,0}+\sqrt{c_{w,0}^2+2{c_{w,0}}/{ c_{\mathrm{res}}}}}{2},$$
where the last quantity is  independent of depth $L$ and time $t$, and we denote this by $$\bar{c}=\frac{c_{w,0}+\sqrt{c_{w,0}^2+2{c_{w,0}}/{ c_{\mathrm{res}}}}}{2},$$
which finishes the proof of Proposition \ref{proposition..A priori spectral property random matrix and a t}.
\end{proof}

 We state the inductive  relation \eqref{ineqlity.... a universal inequlaity on the output value of each layer}  as a  proposition.
\begin{prop}\label{proposition.... on the output of layes}
Under Assumptions \ref{Assump...Assumption on activation functions} and \ref{Assump... ont he imput of the  samples},  we have with high probability w.r.t the random initialization, 
for time $ t\lesssim  \sqrt{m}$ with $0\leq l \leq L,$
\begin{equation}\label{inequality...on the estimate of each layer output L2 norm}
    \Norm{\vx^{[l]}}_2\leq C,
\end{equation}
where $C>0$ is a constant, independent of depth $L$. 
\end{prop}
\begin{rmk}
 We shall note that the constant $C$ in Proposition \ref{proposition.... on the output of layes} only depends on $c_{\mathrm{res}},c_{w,0}$ and $c_{\sigma}.$ However, for  a fully-connected feedforward network, \eqref{inequality...on the estimate of each layer output L2 norm} in Proposition \ref{proposition.... on the output of layes}  become 
 \begin{equation}
  \Norm{\vx^{[l]}}_2\leq C \  2^{l}.   
 \end{equation} 
\end{rmk}
 Note that the $2$-norm for each output layer increase exponentially layer by layer for  fully-connected network, showing that ResNet possesses more stability compared with fully-connected network.

Next we end this part by making an Apriori estimate on the $L^2$-norm for  arbitrary vector $\vv(t)\in \sA_0.$

\begin{prop}\label{proposition...on the l2 norm of any vector in A0} 
Under Assumption \ref{Assump...Assumption on activation functions} and \ref{Assump... ont he imput of the  samples},  with high probability w.r.t~the random initialization, uniformly for any vector $\vv(t)\in \sA_0$ and time $t\lesssim\sqrt{m},$ the following holds
\begin{equation}\label{equation on proposition...on the 2 norm of vector}
    \Norm{\vv(t)}_2\leq c \sqrt{m},
\end{equation}
where $c>0$ is a  constant independent of depth $L$ and time $t$.
\end{prop}
\begin{proof}
We shall start our analysis on the whole expressions in set $\sA_0.$ For any vector $\vv(t)\in\sA_0,$ we can write $\vv(t)=\ve_s\ve_{s-1}\dots \ve_1\ve_0$ with $ 0\leq s\leq 4L.$

We start with the estimate on $\ve_0,$ since $\ve_0$ is chosen following the rules:
\begin{align*}
    \ve_0\in\left\{\va_t,\{\sqrt{m} \vx_\beta^{[1]},\sqrt{m} \vx_\beta^{[2]},\dots,\sqrt{m} \vx_\beta^{[L]}  \}_{1\leq \beta\leq n}   \right\}. 
\end{align*}
\begin{itemize}
    \item (a). If $\ve_0=\va_t,$ then by Lemma \ref{lemma...on chi2 distribution}, for $t\lesssim \sqrt{m}$
    \begin{equation*}
        \Norm{\va_t}_{2} \leq c_{w,t}\sqrt{m}\leq c\sqrt{m}.
    \end{equation*}
    \item (b). If $\ve_0=\sqrt{m} \vx_\beta^{[l]}$ where $ 1\leq l \leq L,$ then based on Proposition \ref{proposition.... on the output of layes}, for $t\lesssim \sqrt{m}$
    \begin{align*}
       \Norm{ \sqrt{m} \vx_\beta^{[l]}}_{2}=\sqrt{m}\Norm{\vx_\beta^{[l]}}_{2} \leq c\sqrt{m}.
    \end{align*}
    \end{itemize}
Now we proceed to other terms in the expression $\ve_j$ where $ j\geq 1.$
\begin{itemize}
    \item (i). If $\ve_j=\vsigma_{[l]}^{(1)}(\vx_\beta),$ then we have
    \begin{align*}
    \Norm{\vv(t)}_{2}&= \Norm{\ve_s\ve_{s-1}\dots \ve_1\ve_0}_{2}\\
    &=\Norm{\ve_s}_{2\to 2}\Norm{\ve_{s-1}}_{2\to 2}\dots\Norm{\ve_{1}}_{2\to 2}\Norm{\ve_0}_{2}.
    \end{align*}
Since $\Norm{\vsigma_{[l]}^{(1)}(\vx_\beta)}_{2\to 2}\leq C_L\leq 1,$ thus 
for all $j\geq 1$ with $\ve_j=\vsigma_{[l]}^{(1)}(\vx_\beta)$  
    \begin{equation*}
      \Norm{\vv(t)}_{2}\leq \left( C_L\right)^{4L}    c\sqrt{m}\leq  c\sqrt{m}. 
    \end{equation*}
    \item (ii). If $\ve_j=\mE_{t,\beta}^{[l]}$ or  $\ve_j=\left(\mE_{t,\beta}^{[l]}\right)^{\T},$ then based on Proposition \ref{proposition..A priori spectral property random matrix and a t}
    \begin{align*}
    \Norm{\vv(t)}_{2}&=\Norm{\ve_s}_{2\to 2}\Norm{\ve_{s-1}}_{2\to 2}\dots\Norm{\ve_{1}}_{2\to 2}\Norm{\ve_0}_{2}.
    \end{align*}
    Since 
    \begin{equation*}
          \Norm{\mE_{t,\beta}^{[l]}}_{2\to 2}=\Norm{\left(\mE_{t,\beta}^{[l]}\right)^\T}_{2\to 2}\leq \left(1+\frac{c_{\mathrm{res}}C_L}{L}\xi(t)\right) \leq  \left(1+\frac{c_{\mathrm{res}}c_{w,t}}{L}\right),  
    \end{equation*}
    thus for all $j\geq 1$ with $\ve_j=\mE_{t,\beta}^{[l]}$ or  $\ve_j=\left(\mE_{t,\beta}^{[l]}\right)^{\T},$
        \begin{align}
    \Norm{\vv(t)}_{2}&\leq \left(1+\frac{c_{\mathrm{res}}c_{w,t}}{L}\right)^{s}\Norm{\ve_0}_{2}, \nonumber
    \end{align}
    then by taking supreme on $0\leq s\leq 4L,$ we have 
        \begin{align*}
    \Norm{\vv(t)}_{2}&\leq \left(1+\frac{c_{\mathrm{res}}c_{w,t}}{L}\right)^{4L}\Norm{\ve_0}_{2}\\
    &\leq c \exp(4c_{\mathrm{res}}c_{w,t})\sqrt{m}\leq c\sqrt{m}.
    \end{align*}
\end{itemize}
   Combining these two observations, we finish the proof.
\end{proof}
Thus,  if we define the quantity $\xi_{\infty,0}(t)$ as follows,
\begin{equation}\label{xi  0 infty}
    \xi_{\infty,0}(t)=\sup_{0\leq t'\leq t}\left\{\Norm{\vv(t')}_2:  \vv(t')\in \sA_0\right\}.
 \end{equation}
Then directly from Proposition \ref{proposition...on the l2 norm of any vector in A0}, for time $0\leq t \leq \sqrt{m}/{\LogLn}^{C'},$ the following holds
\begin{equation}\label{eq for xi infty 0}
    \xi_{\infty,0} (t)\leq c\sqrt{m}.
\end{equation}
\subsection{Apriori $L^{\infty}$ bounds for expressions in $\sA_0$}\label{Appen subsection....Aprioir L infinity bounds for A 0}
In this  part, we shall make estimate on the quantity $ \eta_{\infty,0}(t)$ defined below
\begin{equation}\label{eta 0 infty}
    \eta_{\infty,0}(t)=\sup_{0\leq t'\leq t}\left\{\Norm{\vv(t')}_{\infty}:  \vv(t')\in \sA_0\right\}.
 \end{equation}
We shall begin by a lemma on the $\Norm{\cdot}_{\infty}$ norm of a standard Gaussian vector.
\begin{lem}\label{lemma..... on the initial of gaussian vectors}
For any i.i.d. normal distribution $X_1,X_2,\dots,X_m\sim \fN(0, 1),$ it holds with high probability that the $L^{\infty}$-norm of the  gaussian vector $\vX=\left(X_1,X_2,\dots, X_m\right)^{\T}$ is upper bounded by 
 \begin{equation*}
\Norm{\vX}_{\infty}\leq {\LogLn}^{C},        
 \end{equation*}
 for some really large constant $C>0$ . 
\end{lem}
\begin{proof}
For any $X_i\sim\fN(0,1)$, we have that for some $\eps, \lambda>0$
\begin{align*}
\Prob\left(X_i \geq \eps \right) &=   \Prob\left(\exp\left(\lambda X_i\right) \geq \exp\left(\lambda \eps\right) \right) \\
&\leq \frac{\Exp\left( \lambda X_i\right)}{\exp\left(\lambda \eps\right)}=\frac{\exp\left( \frac{1}{2} \lambda^2\right)}{\exp\left(\lambda \eps\right)}=\exp\left(\frac{1}{2} \lambda^2-\lambda \eps\right).
\end{align*}
We optimize over $\lambda,$ 
\begin{equation*}
  \Prob\left(X_i \geq \eps \right)\leq \min_{\lambda>0}\exp\left(\frac{1}{2} \lambda^2-\lambda \eps\right)  =\exp \left( -\frac{\eps^2}{2}\right).
\end{equation*}
By taking absolute value 
\begin{equation*}
  \Prob\left(\Abs{X_i} \geq \eps \right)\leq 2\exp \left( -\frac{\eps^2}{2}\right).
\end{equation*}
Hence if we take over $m$ unions  
\begin{equation*}
  \Prob\left(\Norm{\vX}_{\infty} \geq \eps \right)\leq 2m \exp \left( -\frac{\eps^2}{2}\right).
\end{equation*}
Set $\eps= {\LogLn}^{C}$, we have that 
\begin{equation*}
  \Prob\left(\Norm{\vX}_{\infty} \leq {\LogLn}^{C} \right)\geq 1- 2m \exp \left( -\frac{{\LogLn}^{2C}}{2}\right).
\end{equation*}
Note that when $C>0$ is really large, ${\LogLn}^{C}\approx m^{\eps}$ for some small $\eps>0.$
\end{proof}
We now state a lemma on the matrix two to infinity norm.
\begin{lem}\label{lemma on the 2 to inmi}
Given a matrix $\mW\in \sR^{m\times m}$ with each entry $W_{i,j}\sim \fN(0, 1),$ then  with high probability, the following holds
\begin{equation}\label{inequality.....statement of lemma on the 2 to infinity norm}
     \Norm{\mW}_{2\to \infty}=\sup_{\Norm{\vx}_2=1}\Norm{\mW\vx}_{\infty}\leq {\LogLn}^{C}. 
\end{equation}
\end{lem}
\begin{proof}
Note that  $\mW\vx$ shares the same distribution as the Gaussian vector $\vX$ in Lemma \ref{lemma..... on the initial of gaussian vectors}, i.e. $\mW\vx\sim \vX.$ Then apply Lemma \ref{lemma..... on the initial of gaussian vectors} directly, we obtain the result.
\end{proof}
Finally, to evaluate $\eta_{\infty,0}(t),$ we need to state a lemma.
\begin{lem}\label{lemma.....lemma on the eta wrt to he alphabet}
Under Assumption \ref{Assump...Assumption on activation functions} and \ref{Assump... ont he imput of the  samples}, for any vector $\vv(t)\in\sA_0,$ we can write 
$$\vv(t)=\ve_s\ve_{s-1}\dots \ve_1\ve_0, 0\leq s\leq 4L,    \ t\geq 0.$$ 
For some vectors in $\sA_0$ with \textbf{length} $q$, we define $\eta_{q,0}(t)$ as
\begin{equation}
    \eta_{q,0}(t):=\sup_{0\leq t'\leq t}\left\{\Norm{\vv_q(t')}_{\infty}:\vv_q(t')=\ve_q\ve_{q-1}\dots \ve_1\ve_0, \ \vv_q(t')\in\sA_0 \right\}.
\end{equation}
Moreover,  we define $\omega(t)$ as 
 \begin{equation*} 
     \omega(t):=\sup_{0\leq t'\leq t} \ \max \left\{ \Norm{\mW_{t'}^{[2]}}_{2\to\infty},\Norm{\left(\mW_{t'}^{[2]}\right)^{\T}}_{2\to\infty},\dots,\Norm{\mW_{t'}^{[L]}}_{2\to\infty},\Norm{\left(\mW_{t'}^{[L]}\right)^{\T}}_{2\to\infty} \right\},
 \end{equation*}
then  with high probability w.r.t the random initialization, for $t\lesssim\sqrt{m}$
\begin{equation}\label{inequality....induction relation of lemma}
    \eta_{q,0}(t) \leq \eta_{0,0}(t)+c \ \omega(t)  \left(1+\frac{c_{\mathrm{res}}c_{w,t}}{L}\right)^{q},
\end{equation}
where constant  $c>0$ is independent of  depth $L$, and $c_{w,t}$ has been defined in Proposition \ref{proposition..A priori spectral property random matrix and a t}.
\end{lem}
\begin{proof}
Since for any vector $\vv_q(t)\in\sA_0$ of length $q, \ 0\leq q\leq 4L$, we can write~$\vv_q(t)$ into $$\vv_q(t)=\ve_q\ve_{q-1}\dots \ve_1\ve_0,$$  then we shall prove \eqref{inequality....induction relation of lemma} by performing induction on $q .$ Firstly, for $q=0,$ \eqref{inequality....induction relation of lemma} is trivial.
While for $q\geq 1,$ we shall investigate on  the terms $\ve_j$ in the expression $\vv_q(t),$ where $j\geq 1.$
\begin{itemize}
    \item (i). If $\ve_j=\vsigma_{[l]}^{(1)}(\vx_\beta),$ then we have 
    \begin{align*}
    \Norm{\vv_q(t)}_{\infty}&= \Norm{\ve_q\ve_{q-1}\dots \ve_1\ve_0}_{\infty}\\
    &=\Norm{\ve_q}_{\infty\to\infty}\Norm{\ve_{q-1}}_{\infty\to\infty}\dots\Norm{\ve_{1}}_{\infty\to\infty}\Norm{\ve_0}_{\infty},
    \end{align*}
    since $\Norm{\vsigma_{[l]}^{(1)}(\vx_\beta)}_{\infty\to\infty}\leq C_L\leq 1,$  we have 
    \begin{equation*}
      \Norm{\vv_q(t)}_{\infty}\leq \left( C_L\right)^{q}    c{\LogLn}^{C} \leq c{\LogLn}^{C}. 
    \end{equation*}
    \item (ii). If $\ve_j=\mE_{t,\beta}^{[l]}$ or  $\ve_j=\left(\mE_{t,\beta}^{[l]}\right)^{\T}$ where $2\leq l \leq L,$ $\Norm{\ve_{j}}_{\infty\to\infty}\geq 1,$ so we need to tackle it differently. 
    \begin{align*}
    \Norm{\vv_q(t)}_{\infty}&=\Norm {\mE_{t,\beta}^{[l]}\vv_{q-1}(t)}_{\infty}\\
    &=\Norm{\vv_{q-1}(t)+ \frac{c_{\mathrm{res}}}{L} \vsigma^{(1)}_{[l]}(\vx_{\beta})\frac{\mW_t^{[l]}}{\sqrt{m}}\vv_{q-1}(t) }_{\infty}\\
    &\leq\Norm{\vv_{q-1}(t)}_{\infty}+\frac{c_{\mathrm{res}}C_L}{L}  \Norm{\frac{\mW_t^{[l]}}{\sqrt{m}}}_{2\to\infty}\Norm{\vv_{q-1}(t)}_2,\\
 \text{or} \        \Norm{\vv_q(t)}_{\infty}&=\Norm {\left(\mE_{t,\beta}^{[l]}\right)^\T\vv_{q-1}(t)}_{\infty}\\
    &=\Norm{\vv_{q-1}(t)+ \frac{c_{\mathrm{res}}}{L} \left(\frac{\mW_t^{[l]}}{\sqrt{m}}\right)^{\T}\vsigma^{(1)}_{[l]}(\vx_{\beta})\vv_{q-1}(t) }_{\infty}\\
    &\leq\Norm{\vv_{q-1}(t)}_{\infty}+\frac{c_{\mathrm{res}}C_L}{L}  \Norm{\left(\frac{\mW_t^{[l]}}{\sqrt{m}}\right)^{\T}}_{2\to\infty}\Norm{\vv_{q-1}(t)}_2,
    \end{align*}
recall the definition of $\omega(t)$,  we have 
    \begin{equation*}
        \Norm{\vv_q(t)}_{\infty}\leq  \Norm{\vv_{q-1}(t)}_{\infty}+ \frac{c_{\mathrm{res}}}{L\sqrt{m}}\omega(t) \Norm{\vv_{q-1}(t)}_{2}.
    \end{equation*}
Based on  Proposition \ref{proposition..A priori spectral property random matrix and a t} 
\begin{align*}
\Norm{\vv_{q-1}(t)}_{2}&\leq c\left(1+\frac{c_{\mathrm{res}}c_{w,t}}{L}\right)^{q-1}\sqrt{m},
        \end{align*}
then
        \begin{align*}
       \Norm{\vv_q(t)}_{\infty}&\leq  \Norm{\vv_{q-1}(t)}_{\infty}+ \frac{c_{\mathrm{res}}}{L\sqrt{m}}\omega(t) \Norm{\vv_{q-1}(t)}_{2}\\
       &\leq \Norm{\vv_{q-1}(t)}_{\infty}+ \frac{c \ c_{\mathrm{res}}}{L}\omega(t) \left(1+\frac{c_{\mathrm{res}}c_{w,t}}{L}\right)^{q-1},
    \end{align*}
inductively we have
    \begin{align*}
 \Norm{\vv_q(t)}_{\infty}&\leq  \Norm{\vv_{0}(t)}_{\infty}+ \frac{c}{c_{w,t}}\omega(t)  \left(1+\frac{c_{\mathrm{res}}c_{w,t}}{L}\right)^{q}\\
 &\leq \Norm{\vv_{0}(t)}_{\infty}+ c \ \omega(t)  \left(1+\frac{c_{\mathrm{res}}c_{w,t}}{L}\right)^{q},
    \end{align*}
where we use the property of a  geometric sum. By taking supreme on both sides, we have 
        \begin{align*}
        \eta_{q,0}(t)\leq \eta_{0,0}(t)+ c \ \omega(t)  \left(1+\frac{c_{\mathrm{res}}c_{w,t}}{L}\right)^{q}.
    \end{align*}
    \end{itemize}
\end{proof}
Based on these lemmas, recall definition \eqref{eta 0 infty}, we are able to make a proposition on the quantity $\eta_{\infty,0}(t)$ at $t=0$. 
\begin{prop}\label{propostion on the esimate of eta t at t=0..for set A_0}
Under Assumption \ref{Assump...Assumption on activation functions} and \ref{Assump... ont he imput of the  samples}, with high probability w.r.t the random initialization 
\begin{equation}\label{inequality in the proposition of the estimate of eta 0}
        \eta_{\infty,0}(0)\leq c {\LogLn}^{C},
\end{equation}
where $c,C>0$ are constants independent of the depth $L$.
\end{prop}
\begin{proof}
As always, for any vector $\vv(t)\in\sA_0,$ we can write $\vv(t)$ as $$\vv(t)=\ve_s\ve_{s-1}\dots \ve_1\ve_0, 0\leq s\leq 4L.$$
We start with the estimate on $\eta_{0,0}(0),$ since $\ve_0$ is chosen following the rules:
\begin{align*}
    \ve_0\in\left\{\va_t,\{\sqrt{m} \vx_\beta^{[1]},\sqrt{m} \vx_\beta^{[2]},\dots,\sqrt{m} \vx_\beta^{[L]}  \}_{1\leq \beta\leq n}   \right\}. 
\end{align*}
\begin{itemize}
    \item (a). If $\ve_0=\va_t,$ then at $t=0$, by Lemma \ref{lemma..... on the initial of gaussian vectors}, 
    \begin{equation*}
        \Norm{\va_0}_{\infty} \leq {\LogLn}^{C}.
    \end{equation*}
    \item (b). If $\ve_0=\sqrt{m} \vx_\beta^{[l]},$  starting with $l=1$
    \begin{align*}
       \Norm{ \sqrt{m} \vx_\beta^{[1]}}_{\infty}&=\sqrt{c_{\sigma}} \Norm{\sigma\left(\mW_0^{[1]}\vx_{\beta}\right)}_{\infty}\\
       &\leq \sqrt{c_{\sigma}}\left(\Abs{\sigma(0)}+C_L\Norm{\mW_0^{[1]}\vx_{\beta}}_{\infty}\right)\\
       &\leq \sqrt{c_{\sigma}}\left( C_L+C_L \Norm{\mW_0^{[1]}}_{2\to \infty}\Norm{\vx_{\beta}}_2\right)\\
       &\leq \sqrt{c_{\sigma}}C_L \left(1+{\LogLn}^{C}\right) \leq c {\LogLn}^{C},
    \end{align*}
   moreover, for $l \geq 1 ,$ based on Proposition \ref{proposition.... on the output of layes},
        \begin{align*}
       \Norm{ \sqrt{m} \vx_\beta^{[l]}}_{\infty}&\leq \Norm{ \sqrt{m} \vx_\beta^{[l-1]}}_{\infty}+\frac{c_{\mathrm{res}}}{L}\Norm{\sigma\left(\mW_0^{[l]}\vx_{\beta}^{[l-1]}\right)}_{\infty}\\
       &\leq \Norm{ \sqrt{m} \vx_\beta^{[l-1]}}_{\infty}+ \frac{c_{\mathrm{res}}}{L}\left( C_L+C_L \Norm{\mW_0^{[l]}}_{2\to \infty}\Norm{\vx_{\beta}^{[l-1]}}_2\right)\\
       &\leq \Norm{ \sqrt{m} \vx_\beta^{[l-1]}}_{\infty}+ \frac{c_{\mathrm{res}}C_L}{L}\left(1+ C {\LogLn}^{C}\right)\\
       &\leq  \Norm{ \sqrt{m} \vx_\beta^{[l-1]}}_{\infty}+ \frac{c}{L} {\LogLn}^{C},
    \end{align*}
inductively    for $1\leq l \leq L$, 
    \begin{equation}\label{inequalit of proof in proposition.... infinity norm on each layer}
        \Norm{ \sqrt{m} \vx_\beta^{[l]}}_{\infty} \leq c \left(1+\frac{l}{L}\right) {\LogLn}^{C}\leq c{\LogLn}^{C},
    \end{equation}
where $c$ is independent of the depth $L$.
\end{itemize}
Hence we have 
\begin{equation}\label{inequality...estimate on the zero terms of infinity norm}
   \eta_{0,0}(0)\leq c{\LogLn}^{C}. 
\end{equation}
Directly from Lemma \ref{lemma.....lemma on the eta wrt to he alphabet}  
\begin{align*}
    \eta_{q,0}(0)&\leq \eta_{0,0}(0)+c {\LogLn}^C \left(1+\frac{c_{\mathrm{res}}c_{w,0}}{L}\right)^{q}\\
    &\leq c{\LogLn}^{C}+ c {\LogLn}^C \exp(4 c_{\mathrm{res}}c_{w,0})\leq c {\LogLn}^C,
\end{align*}
by taking supreme on $0\leq q \leq 4L$, we finish our proof.
\end{proof}
Our next proposition is on $\eta_{\infty,0}(t)$   for time $0\leq t \leq \sqrt{m}/{\LogLn}^{C'}.$
\begin{prop}\label{proposition...estimate on A_0 for t >0}
Under Assumption \ref{Assump...Assumption on activation functions} and \ref{Assump... ont he imput of the  samples}, with high probability w.r.t the random initialization, for time $0\leq t \leq \sqrt{m}/{\LogLn}^{C'},$ the following holds
\begin{equation}
    \eta_{\infty,0} (t)\leq c{\LogLn}^{C},
\end{equation}
where $c,C, C'>0$ are constants independent of the depth $L$. 
\end{prop}
\begin{proof}
We shall start with the estimate on $\eta_{0,0}(t),$ since $\ve_0$ is chosen following the rules:
\begin{align*}
    \ve_0\in\left\{\va_t,\{\sqrt{m} \vx_\beta^{[1]},\sqrt{m} \vx_\beta^{[2]},\dots,\sqrt{m} \vx_\beta^{[L]}  \}_{1\leq \beta\leq n}   \right\}. 
\end{align*}
We observe that from the replacement rules given in Section \ref{subsection....replacemetn rules}, 
\begin{align*}
   \partial_t \va_t&=-\frac{1}{n}\sum_{\beta=1}^n\frac{1}{\sqrt{m}}\sqrt{m}\vx_{\beta}^{[L]}(f_{\beta}(t)-y_{\beta}) ,\\
\partial_t\sqrt{m}\vx_{\alpha}^{[l]}&=-\frac{1}{n}\sum_{\beta=1}^n
{\frac{c_\sigma}{\sqrt{m}}} \mE_{t,\alpha}^{[2:l]} \vsigma^{(1)}_{[1]}(\vx_{\alpha})\vsigma^{(1)}_{[1]}(\vx_{\beta})\left(\mE_{t,\beta}^{[2:L]}\right)^{\T}\va_t \left<\vx_\alpha,\vx_\beta  \right>(f_{\beta}(t)-y_{\beta})\\
+-\frac{1}{n}\sum_{\beta=1}^n&\sum_{k=2}^{l} \frac{c_{\mathrm{res}}^2}{L^2 \sqrt {m}} \mE_{t,\alpha}^{[(k+1):l]}\vsigma^{(1)}_{[k]}(\vx_{\alpha})\vsigma^{(1)}_{[k]}(\vx_{\beta})\left(\mE_{t,\beta}^{[(k+1):L]}\right)^{\T}\va_t  \left<\vx_\alpha^{[k-1]},\vx_\beta^{[k-1]}  \right>(f_{\beta}(t)-y_{\beta}),
\end{align*}
since for $0\leq t\leq \sqrt{m}/{\LogLn}^{C'}$, then by Proposition \ref{proposition.... on the output of layes}  
\begin{align*}
 \partial_t \Norm{\va_t}_{\infty}&\leq \frac{C}{\sqrt{m}}\Norm{\sqrt{m}\vx_{\beta}^{[L]}}_{\infty},\\
 \partial_t \Norm{\sqrt{m}\vx_{\alpha}^{[l]}}_{\infty} &\leq \sum_{k=1}^{l}\frac{C}{\sqrt{m}}\Norm{\mE_{t,\alpha}^{[(k+1):l]}\vsigma^{(1)}_{[k]}(\vx_{\alpha})\vsigma^{(1)}_{[k]}(\vx_{\beta})\left(\mE_{t,\beta}^{[(k+1):L]}\right)^{\T}\va_t }_{\infty},
\end{align*}
  by taking supreme on time $0\leq t\leq \sqrt{m}/{\LogLn}^{C'},$ we have  
\begin{equation}
    \eta_{0,0}(t)\leq c{\LogLn}^{C}+ \frac{C}{\sqrt{m}} \int_{0}^t \eta_{\infty,0}(s) \diff s.
\end{equation}
For the auxiliary term $\omega(t),$ from the replacement rules again, for $2\leq l \leq L$
\begin{align*}
    \partial_t \mW^{[l]}_t&=-\frac{1}{n}\sum_{\beta=1}^n\frac{c_{\mathrm{res}}}{L\sqrt{m}}\vsigma^{(1)}_{[l]}(\vx_{\beta})  \left(\mE_{t,\beta}^{[(l+1):L]}\right)^{\T}       \va_t  \otimes  (\vx_{\beta}^{[l-1]})^\T (f_{\beta}(t)-y_{\beta}),\\
        \partial_t \left(\mW^{[l]}_t\right)^\T&=-\frac{1}{n}\sum_{\beta=1}^n\frac{c_{\mathrm{res}}}{L\sqrt{m}} \vx_{\beta}^{[l-1]}\otimes  \left(\vsigma^{(1)}_{[l]}(\vx_{\beta})  \left(\mE_{t,\beta}^{[(l+1):L]}\right)^{\T}       \va_t \right)^\T (f_{\beta}(t)-y_{\beta}),
\end{align*}
then by Proposition \ref{proposition...on the l2 norm of any vector in A0}
\begin{align*}
   \partial_t \Norm{\mW_t^{[l]}}_{2\to \infty}&\leq \frac{C}{\sqrt{m}}\Norm{\vsigma^{(1)}_{[l]}(\vx_{\beta})  \left(\mE_{t,\beta}^{[(l+1):L]}\right)^{\T}       \va_t  }_{\infty}, \\
   \partial_t \Norm{\left(\mW_t^{[l]}\right)^\T}_{2\to \infty}&\leq \frac{C}{\sqrt{m}}\Norm{\sqrt{m}\vx_{\beta}^{[l-1]}}_{\infty},
\end{align*}
hence  by taking supreme on time $0\leq t\leq \sqrt{m}/{\LogLn}^{C'},$ we have 
\begin{equation}
    \omega(t)\leq {\LogLn}^{C}+\frac{C}{\sqrt{m}} \int_{0}^t \eta_{\infty,0}(s) \diff s.
\end{equation}
Directly from Lemma \ref{lemma.....lemma on the eta wrt to he alphabet}  
\begin{align*}
    \eta_{q,0}(t)&\leq \eta_{0,0}(t)+c \  \omega(t) \left(1+\frac{c_{\mathrm{res}}c_{w,t}}{L}\right)^{q}\\
    &\leq \left(c{\LogLn}^{C}+ \frac{C}{\sqrt{m}} \int_{0}^t \eta_{\infty,0}(s) \diff s\right)\left(1+ \left(1+\frac{c_{\mathrm{res}}c_{w,t}}{L}\right)^{q}\right).
\end{align*}
Finally by taking supreme on $0\leq q \leq 4L,$ we have 
\begin{align*}
    \eta_{\infty,0}(t)&\leq c{\LogLn}^{C}+ \frac{C}{\sqrt{m}} \int_{0}^t \eta_{\infty,0}(s) \diff s .
\end{align*}
This gives us a Gronwall-type inequality, we have that 
\begin{equation*}
    \eta_{\infty,0}(t)\leq c{\LogLn}^{C}\exp\left(  \frac{Ct}{\sqrt{m}} \right).
\end{equation*}
To sum up,  for $t\leq \sqrt{m}/{\LogLn}^{C'} , $ the following holds
\begin{equation}
    \eta_{\infty,0}(t)\leq c{\LogLn}^{C},
\end{equation}
which finishes the proof.
\end{proof}
\subsection{Apriori $L^2$ and $L^{\infty}$ bounds for expression in $\sA_r$, $r\geq 1$}\label{appendix....subsection L2 and L infinity bounds for A R}
In this part, we shall make estimates for  $\Norm{\cdot}_{\infty}$  and $\Norm{\cdot}_2$ of vectors belonging to higher order sets, i.e., $\sA_r$, $r\geq 1.$ Then it is natural for us to define several quantities for some vectors in $\sA_r$ with length $q$ 
\begin{equation}\label{definition.....apriori L2 bounds for sets higher order}
    \xi_{q,r}(t):=\sup_{0\leq t'\leq t}\left\{\Norm{\vv_q(t')}_{2}:\vv_q(t')=\ve_q\ve_{q-1}\dots \ve_1\ve_0, \ \vv_q(t')\in\sA_r \right\},
\end{equation}
note that from Proposition \ref{proposition..A priori spectral property random matrix and a t} and \ref{proposition...on the l2 norm of any vector in A0}, 
\begin{equation}\label{xi_q,0..... power}
      \xi_{q,0}(t)\leq c\left(1+\frac{c_{\mathrm{res}}c_{w,t}}{L}\right)^q \sqrt{m},
\end{equation}
moreover, we  define that 
\begin{equation}\label{definition in the proposition of eta infty}
        \xi_{\infty,r}(t)=\sup_{0\leq q \leq 4L}\left\{\xi_{q,r}(t)\right\},
\end{equation}
then by taking supreme on $0\leq q\leq 4L $ in \eqref{xi_q,0..... power} 
\begin{equation}\label{xi_infinity ,0 ...unifomr estimate }
     \xi_{\infty,0}(t)\leq c\sqrt{m},
\end{equation}
and recall the definition we made in Section \ref{Appen subsection....Aprioir L infinity bounds for A 0}, similarly we define
\begin{equation}\label{definition.....apriori L infinity bounds for sets higher order}
    \eta_{q,r}(t):=\sup_{0\leq t'\leq t}\left\{\Norm{\vv_q(t')}_{\infty}:\vv_q(t')=\ve_q\ve_{q-1}\dots \ve_1\ve_0, \ \vv_q(t')\in\sA_r \right\},
\end{equation}
moreover, we define that 
\begin{equation}\label{definition in the proposition of eta infty.}
        \eta_{\infty,r}(t)=\sup_{0\leq q \leq 4L}\left\{\eta_{q,r}(t)\right\}.
\end{equation}

Once again, for any vector $\vv(t)\in\sA_r,$ it can be written into $$\vv(t)=\ve_s\ve_{s-1}\dots \ve_1\ve_0, \   0\leq s\leq 4L,$$ 
we shall start with the estimate on $\ve_0.$ Since $\ve_0$ is chosen following the rules:
\begin{align*}
    \ve_0\in\left\{\va_t,\vone,\{ \sqrt{m} \vx_\beta^{[1]},\sqrt{m} \vx_\beta^{[2]},\dots,\sqrt{m} \vx_\beta^{[L]}  \}_{1\leq \beta\leq n}   \right\}. 
\end{align*}
 $\Norm{\vone}_{\infty}=1, \Norm{\vone}_{2}=\sqrt{m},$ then  for time $0\leq t\leq \sqrt{m}/{\LogLn}^{C'},$ by Proposition \ref{proposition..A priori spectral property random matrix and a t} and \ref{proposition...estimate on A_0 for t >0},
\begin{equation*}
   \xi_{0,r}(t)\leq  c\sqrt{m}, \  \eta_{0,r}(t)\leq c{\LogLn}^{C}.
\end{equation*}
Now we proceed to other terms in the expression $\ve_j$ where $ j\geq 1.$ For each $\ve_j,$ there are several cases:
\begin{itemize}
\item (i) $\ve_j=\vsigma_{[l]}^{(1)}(\vx_\beta), $ $\ve_j=\mE_{t,\beta}^{[l]}$ or  $\ve_j=\left(\mE_{t,\beta}^{[l]}\right)^{\T},  \ 2\leq l \leq L.$
\item (ii) $\ve_j=\mathrm{diag}(\vg).$
\item (iii) $$\ve_j=\vsigma_{[l]}^{(u+1)}(\vx_\beta) \mathrm{diag} \left(\left(\frac{\mW_t^{[l]}}{\sqrt{m}}\right)^{Q_1} \vg_1\right)\dots\mathrm{diag} \left(\left(\frac{\mW_t^{[l]}}{\sqrt{m}}\right)^{Q_u} \vg_u\right)\left(\frac{c_{\mathrm{res}}}{L}\frac{\mW_t^{[l]}}{\sqrt{m}}\right)^{Q_{u+1}},$$
or
$$\ve_j=\left(\frac{c_{\mathrm{res}}}{L}\frac{\left(\mW_t^{[l]}\right)^\T}{\sqrt{m}}\right)^{Q_{u+1}}\vsigma_{[l]}^{(u+1)}(\vx_\beta) \mathrm{diag} \left(\left(\frac{\mW_t^{[l]}}{\sqrt{m}}\right)^{Q_1} \vg_1\right)\dots\mathrm{diag} \left(\left(\frac{\mW_t^{[l]}}{\sqrt{m}}\right)^{Q_u} \vg_u\right).$$
\end{itemize}
By our observation, the total number of diag operations in $\vv(t)\in\sA_r$ is $r,$ and that is how we characterize a vector belonging to different hierarchical sets. Especially if for one of those $\ve_j$ belongs to case (iii), there are two scenarios:
\begin{itemize}
    \item $Q_{u+1}=0,$ then $\ve_j$ is just multiplication of several diagonal matrices, being a special situation for case (ii).
    \item $Q_{u+1}=1,$ since diagonal matrices commute, $\ve_j$ writes into
     $$\ve_j=\mathrm{diag} \left(\left(\frac{\mW_t^{[l]}}{\sqrt{m}}\right)^{Q_1} \vg_1\right)\dots\mathrm{diag} \left(\left(\frac{\mW_t^{[l]}}{\sqrt{m}}\right)^{Q_u} \vg_u\right)\vsigma_{[l]}^{(u+1)}(\vx_\beta) \frac{c_{\mathrm{res}}}{L}\frac{\mW_t^{[l]}}{\sqrt{m}},$$
or
$$\ve_j=\left(\frac{c_{\mathrm{res}}}{L}\frac{\mW_t^{[l]}}{\sqrt{m}}\right)^\T\vsigma_{[l]}^{(u+1)}(\vx_\beta) \mathrm{diag} \left(\left(\frac{\mW_t^{[l]}}{\sqrt{m}}\right)^{Q_1} \vg_1\right)\dots\mathrm{diag} \left(\left(\frac{\mW_t^{[l]}}{\sqrt{m}}\right)^{Q_u} \vg_u\right),$$
we shall take advantage of the special structure of $\ve_j$. Define a new type of skip-connection   matrix, $\widetilde{\mE}_{t,\beta}^{[l,r]},$ for $r\geq 2$:
\begin{equation}\label{definition for another matrices......skip conncetion}
    \widetilde{\mE}_{t,\beta}^{[l,r]}:=\left(\mI_m+\frac{c_{\mathrm{res}}}{L}\vsigma_{[l]}^{(r)}(\vx_\beta)\frac{\mW_t^{[l]}}{\sqrt{m}}\right).
\end{equation}
Then we can write $\ve_j$ into 
\begin{align*}
    \ve_j&=\mathrm{diag} \left(\left(\frac{\mW_t^{[l]}}{\sqrt{m}}\right)^{Q_1} \vg_1\right)\dots\mathrm{diag} \left(\left(\frac{\mW_t^{[l]}}{\sqrt{m}}\right)^{Q_u} \vg_u\right)\vsigma_{[l]}^{(u+1)}(\vx_\beta) \frac{c_{\mathrm{res}}}{L}\frac{\mW_t^{[l]}}{\sqrt{m}}\\
    &=\mathrm{diag} \left(\left(\frac{\mW_t^{[l]}}{\sqrt{m}}\right)^{Q_1} \vg_1\right)\dots\mathrm{diag} \left(\left(\frac{\mW_t^{[l]}}{\sqrt{m}}\right)^{Q_u} \vg_u\right)\widetilde{\mE}_{t,\beta}^{[l,u+1]}\\
    &~~-\mathrm{diag} \left(\left(\frac{\mW_t^{[l]}}{\sqrt{m}}\right)^{Q_1} \vg_1\right)\dots\mathrm{diag} \left(\left(\frac{\mW_t^{[l]}}{\sqrt{m}}\right)^{Q_u} \vg_u\right),
\end{align*}
or 
\begin{align*}
    \ve_j&=\left(\frac{c_{\mathrm{res}}}{L}\frac{\mW_t^{[l]}}{\sqrt{m}}\right)^\T\vsigma_{[l]}^{(u+1)}(\vx_\beta) \mathrm{diag} \left(\left(\frac{\mW_t^{[l]}}{\sqrt{m}}\right)^{Q_1} \vg_1\right)\dots\mathrm{diag} \left(\left(\frac{\mW_t^{[l]}}{\sqrt{m}}\right)^{Q_u} \vg_u\right)\\
&=\left(\widetilde{\mE}_{t,\beta}^{[l,u+1]}\right)^\T\mathrm{diag} \left(\left(\frac{\mW_t^{[l]}}{\sqrt{m}}\right)^{Q_1} \vg_1\right)\dots\mathrm{diag} \left(\left(\frac{\mW_t^{[l]}}{\sqrt{m}}\right)^{Q_u} \vg_u\right)\\
    &~~-\mathrm{diag} \left(\left(\frac{\mW_t^{[l]}}{\sqrt{m}}\right)^{Q_1} \vg_1\right)\dots\mathrm{diag} \left(\left(\frac{\mW_t^{[l]}}{\sqrt{m}}\right)^{Q_u} \vg_u\right).
\end{align*}
\end{itemize}
To illustrate such relation, if some vector $\bar{\vv}(t)$ contains $\ve_j$ belonging to case (iii), we write it as 
\begin{align*}
    \bar{\vv}(t)&=\ve_s\ve_{s-1}\cdots\ve_{j+1}\mathrm{diag} \left(\left(\frac{\mW_t^{[l]}}{\sqrt{m}}\right)^{Q_1} \vg_1\right)\dots\mathrm{diag} \left(\left(\frac{\mW_t^{[l]}}{\sqrt{m}}\right)^{Q_u} \vg_u\right)\vsigma_{[l]}^{(u+1)}(\vx_\beta) \frac{c_{\mathrm{res}}}{L}\frac{\mW_t^{[l]}}{\sqrt{m}}\ve_{j-1}\cdots\ve_0\\
    &=\ve_s\ve_{s-1}\cdots\ve_{j+1}\mathrm{diag} \left(\left(\frac{\mW_t^{[l]}}{\sqrt{m}}\right)^{Q_1} \vg_1\right)\dots\mathrm{diag} \left(\left(\frac{\mW_t^{[l]}}{\sqrt{m}}\right)^{Q_u} \vg_u\right)\widetilde{\mE}_{t,\beta}^{[l,u+1]}\ve_{j-1}\cdots\ve_0\\
    ~~&-\ve_s\ve_{s-1}\cdots\ve_{j+1}\mathrm{diag} \left(\left(\frac{\mW_t^{[l]}}{\sqrt{m}}\right)^{Q_1} \vg_1\right)\dots\mathrm{diag} \left(\left(\frac{\mW_t^{[l]}}{\sqrt{m}}\right)^{Q_u} \vg_u\right)\ve_{j-1}\cdots\ve_0.
\end{align*}

From the analysis above, we are able to characterize an element in   set $\sA_r.$ 
If~$\vv(t)\in\sA_r,$ then as always, we write it as
\begin{align*}
    \vv(t)&=\ve_s\ve_{s-1}\dots \ve_1\ve_0, \   0\leq s\leq 4L,
\end{align*}
and  there exists $\ve_{j_1},\ve_{j_2},\cdots, \ve_{j_{k}}$, such that 
\begin{align*}
    \ve_{j_1}&=\mathrm{diag} \left(\left(\frac{\mW_t^{[l_1]}}{\sqrt{m}}\right)^{Q_1}\vg_1\right), \ \vg_1\in\sA_{r_1-1},\\
     \ve_{j_2}&=\mathrm{diag} \left(\left(\frac{\mW_t^{[l_2]}}{\sqrt{m}}\right)^{Q_2}\vg_2\right), \ \vg_2\in\sA_{r_2-1},\\
     \vdots\\
       \ve_{j_k}&=\mathrm{diag} \left(\left(\frac{\mW_t^{[l_k]}}{\sqrt{m}}\right)^{Q_k}\vg_k\right), \ \vg_k\in\sA_{r_k-1},
\end{align*}
with 
\begin{equation}\label{equation for restriction diag operations}
   r_1+r_2+\cdots+r_k=r, r_1, r_2,\cdots, r_k\in\sN^{+}.
\end{equation}
Equation \eqref{equation for restriction diag operations} serves as the counting of the number of diag operations contained in $\vv(t),$ while for other $\ve_j\left( j\notin\left\{j_1,j_2,\cdots,j_k,0\right\}\right)$, chosen from the following sets
\begin{align}
     & \left\{ \mE_{t,\beta}^{[l]},\left(\mE_{t,\beta}^{[l]}\right)^{\T}: 2\leq l \leq L \right\}_{1\leq \beta\leq n},\label{gp...set1}\\
     &\left\{\vsigma_{[l]}^{(1)}(\vx_\beta): 1\leq l \leq L\right\}_{1\leq \beta\leq n},\label{gp...set2}\\ &\left\{\widetilde{\mE}_{t,\beta}^{[l,p]},\left(\widetilde{\mE}_{t,\beta}^{[l,p]}\right)^\T:2\leq l \leq L, \ p\geq 2\right\}_{1\leq \beta\leq n},\label{gp...set3}
\end{align}
note that the elements in set \eqref{gp...set1} and set \eqref{gp...set3} share the same matrix properties, thanks to Assumption~\ref{Assump...Assumption on activation functions} concerning the activation function.

Hence, in order to make estimates on $\xi_{q,r}(t)$ and   $\eta_{q,r}(t),$ we shall perform induction on the number of diag operations contained in each vector.
\begin{prop}\label{proposition... on the diagonal operation}
Under Assumption \ref{Assump...Assumption on activation functions} and \ref{Assump... ont he imput of the  samples},  with high probability w.r.t the random initialization, for some finite $r\geq 1$ and time $0\leq t \leq \sqrt{m}/{\LogLn}^{C'},$ the following holds
\begin{align}
    \xi_{\infty,r}(t)& \leq c {\LogLn}^{C}\sqrt{m},\label{prop..1}\\
     \eta_{\infty,r}(t)& \leq c {\LogLn}^{C},\label{prop..2}
\end{align}
where $c,C,C'>0$ are constants independent of depth $L.$
\end{prop}
\begin{proof}
We recall the definition of $\omega(t)$, $\eta_{\infty,0}(t)$ and $\xi_{\infty,0}(t)$, for time $0\leq t\leq\sqrt{m}/{\LogLn}^{C'},$ the following holds with high probability, 
\begin{align*}
    \omega(t)&\leq c{\LogLn}^{C},\\
    \eta_{\infty,0}(t)&\leq c{\LogLn}^{C},\\
    \xi_{\infty,0}(t)&\leq c\sqrt{m}. 
\end{align*}
Let's start with $r=1,$ for any $\vv(t)\in\sA_1,$ since there is only one solution to equation~\eqref{equation for restriction diag operations}, then there exists one and only one index $i,$ such that $\ve_i=\mathrm{diag}\left(\vg\right),$ or $\ve_i=\mathrm{diag}\left(\frac{\mW_t^{[l]}}{\sqrt{m}}\vg\right),$ with $\vg\in\sA_0.$ Then we have
\begin{align*}
    \xi_{i,1}(t)& \leq \sup_{\vg\in\sA_0}\Norm{\mathrm{diag} \left(\vg\right)}_{2\to 2} \xi_{i-1,0}(t)\\
    &\leq \sup_{\vg\in\sA_0}\Norm{\vg}_{\infty} \xi_{i-1,0}(t)\leq \eta_{\infty,0}(t)\xi_{i-1,0}(t)\leq c{\LogLn}^{C}\xi_{i-1,0}(t),\\
    \text{or } \     \xi_{i,1}(t)& \leq \sup_{\vg\in\sA_0}\Norm{\mathrm{diag} \left(\frac{\mW_t^{[l]}}{\sqrt{m}}\vg\right)}_{2\to 2} \xi_{i-1,0}(t)\\
    &\leq \sup_{\vg\in\sA_0}\Norm{\frac{\mW_t^{[l]}}{\sqrt{m}}\vg}_{\infty} \xi_{i-1,0}(t)\leq \frac{\omega(t)}{\sqrt{m}}\xi_{\infty,0}(t)\xi_{i-1,0}(t)\leq c{\LogLn}^{C} \xi_{i-1,0}(t),
\end{align*}
for $q>i,$ 
\begin{align*}
    \xi_{q,1}(t)& \leq \left(1+\frac{c_{\mathrm{res}}c_{w,t}}{L}\right) \xi_{q-1,1}(t),
\end{align*}
then inductively we have 
\begin{align*}
    \xi_{q,1}(t)& \leq \left(1+\frac{c_{\mathrm{res}}c_{w,t}}{L}\right)^{q-i} \xi_{i,1}(t),
    \end{align*}
  by taking supreme on $q$ and $i$  
\begin{equation}\label{xi_infinity 1}
     \xi_{\infty,1}(t)\leq  \exp(4 c_{\mathrm{res}}c_{w,t}) c{\LogLn}^{C} \xi_{\infty,0}(t)\leq c {\LogLn}^{C}\sqrt{m},
\end{equation}
and for $\eta_{i,1}(t),$ we have 
\begin{align*}
    \eta_{i,1}(t)& \leq \sup_{\vg\in\sA_0}\Norm{\mathrm{diag} \left(\vg\right)}_{\infty \to \infty} \eta_{i-1,0}(t)\\
    &\leq \sup_{\vg\in\sA_0}\Norm{\vg}_{\infty} \eta_{i-1,0}(t)\leq \eta_{\infty,0}(t)\eta_{i-1,0}(t)\leq c{\LogLn}^{C}\eta_{i-1,0}(t),\\
        \text{or } \     \eta_{i,1}(t)& \leq \sup_{\vg\in\sA_0}\Norm{\mathrm{diag} \left(\frac{\mW_t^{[l]}}{\sqrt{m}}\vg\right)}_{\infty\to \infty} \eta_{i-1,0}(t)\\
    &\leq \sup_{\vg\in\sA_0}\Norm{\frac{\mW_t^{[l]}}{\sqrt{m}}\vg}_{\infty} \eta_{i-1,0}(t)\leq \frac{\omega(t)}{\sqrt{m}}\xi_{\infty,0}(t)\eta_{i-1,0}(t)\leq c{\LogLn}^{C}\eta_{i-1,0}(t),
\end{align*}
and for $q>i,$ inductively 
\begin{align*}
       \eta_{q,1}(t)&\leq   \eta_{q-1,1}(t)+ \frac{c_{\mathrm{res}}}{L\sqrt{m}}\omega(t)  \xi_{q-1,1}(t)\\
       &\leq \eta_{i,1}(t)+ \frac{c_{\mathrm{res}}}{L\sqrt{m}}\omega(t)  \xi_{q-1,1}(t)+ \frac{c_{\mathrm{res}}}{L\sqrt{m}}\omega(t)  \xi_{q-2,1}(t)+\cdots+ \frac{c_{\mathrm{res}}}{L\sqrt{m}}\omega(t)  \xi_{i,1}(t),
    \end{align*}
then by taking supreme on $q $ and $i,$  combined with \eqref{xi_infinity 1}  
\begin{align*}
    \eta_{\infty,1}(t)&\leq c{\LogLn}^{C}\eta_{\infty,0}(t)+\frac{4c_{\mathrm{res}}}{\sqrt{m}} \omega(t)\xi_{\infty,1}(t)\\
    &\leq c{\LogLn}^{C}+c{\LogLn}^{C}\leq c{\LogLn}^{C}.
\end{align*}
In the following we assume that \eqref{prop..1} and \eqref{prop..2} holds for $1,2, \cdots, r-1$ and prove it for $r.$

If~$\vv(t)\in\sA_r,$ then as always, we write it as 
\begin{align*}
    \vv(t)&=\ve_s\ve_{s-1}\dots \ve_1\ve_0, \   0\leq s\leq 4L,
\end{align*}
and  there exists $\ve_{j_1},\ve_{j_2},\cdots, \ve_{j_{k}}$, such that 
\begin{align*}
    \ve_{j_1}&=\mathrm{diag} \left(\left(\frac{\mW_t^{[l_1]}}{\sqrt{m}}\right)^{Q_1}\vg_1\right), \ \vg_1\in\sA_{r_1-1},\\
     \ve_{j_2}&=\mathrm{diag} \left(\left(\frac{\mW_t^{[l_2]}}{\sqrt{m}}\right)^{Q_2}\vg_2\right), \ \vg_2\in\sA_{r_2-1},\\
     \vdots\\
       \ve_{j_k}&=\mathrm{diag} \left(\left(\frac{\mW_t^{[l_k]}}{\sqrt{m}}\right)^{Q_k}\vg_k\right), \ \vg_k\in\sA_{r_k-1},
\end{align*}
with 
\begin{equation*}
   r_1+r_2+\cdots+r_k=r, r_1, r_2,\cdots, r_k\in\sN^{+}.
\end{equation*}
Let $i$ be the largest index among $j_1,j_2,\cdots,j_k,$ i.e.
\begin{align*}
    i=\max\{j_1,j_2,\cdots, j_k\},
\end{align*}
and wlog, let $i=j_1,$ we have 
$\ve_i=\mathrm{diag} \left(\vg_1\right),$ or $\ve_i=  \mathrm{diag} \left(\frac{\mW_t^{[l]}}{\sqrt{m}}\vg_1\right)$ with $\vg_1\in\sA_{r_1-1},$ then 
\begin{align*}
    \xi_{i,r}(t)& \leq \sup_{\vg\in\sA_{r_1-1}}\Norm{\mathrm{diag} \left(\vg\right)}_{2\to 2} \xi_{i-1,r-r_1}(t)\\
    &\leq \sup_{\vg\in\sA_{r_1-1}}\Norm{\vg}_{\infty} \xi_{i-1,r-r_1}(t)\leq \eta_{\infty,r_1-1}(t)\xi_{i-1,r-r_1}(t)\leq c{\LogLn}^{C}\xi_{i-1,r-r_1}(t),\\
    \text{or } \     \xi_{i,r}(t)& \leq \sup_{\vg\in\sA_{r_1-1}}\Norm{\mathrm{diag} \left(\frac{\mW_t^{[l]}}{\sqrt{m}}\vg\right)}_{2\to 2} \xi_{i-1,r-r_1}(t)\leq \sup_{\vg\in\sA_{r_1-1}}\Norm{\frac{\mW_t^{[l]}}{\sqrt{m}}\vg}_{\infty} \xi_{i-1,r-r_1}(t)\\
    &\leq \frac{\omega(t)}{\sqrt{m}}\xi_{\infty,r_1-1}(t)\xi_{i-1,r-r_1}(t)\leq c{\LogLn}^{C} \xi_{i-1,r-r_1}(t),
\end{align*}
inductively 
\begin{align*}
    \xi_{q,r}(t)& \leq \left(1+\frac{c_{\mathrm{res}}c_{w,t}}{L}\right)^{q-i} \xi_{i,r-r_1}(t),
\end{align*}
then by taking supreme on $q$ and $i,$ we obtain
\begin{equation}\label{xi_infinity r}
     \xi_{\infty,r}(t)\leq  \exp(4 c_{\mathrm{res}}c_{w,t}) c{\LogLn}^{C}\xi_{\infty,r-r_1}(t)\leq c {\LogLn}^{C}\sqrt{m}.
\end{equation}
For $\eta_{i,r}(t),$ we have 
\begin{align*}
    \eta_{i,r}(t)& \leq \sup_{\vg\in\sA_{r_1-1}}\Norm{\mathrm{diag} \left(\vg\right)}_{\infty \to \infty} \eta_{i-1,r-r_1}(t)\\
    &\leq \sup_{\vg\in\sA_{r_1-1}}\Norm{\vg}_{\infty} \eta_{i-1,r-r_1}(t)\leq \eta_{\infty,r_1-1}(t)\eta_{i-1,r-r_1}(t)\leq c{\LogLn}^{C}\eta_{i-1,r-r_1}(t),\\
        \text{or } \     \eta_{i,r}(t)& \leq \sup_{\vg\in\sA_{r_1-1}}\Norm{\mathrm{diag} \left(\frac{\mW_t^{[l]}}{\sqrt{m}}\vg\right)}_{\infty\to \infty} \eta_{i-1,r-r_1}(t)\leq \sup_{\vg\in\sA_{r_1-1}}\Norm{\frac{\mW_t^{[l]}}{\sqrt{m}}\vg}_{\infty} \eta_{i-1,r-r_1}(t)\\
    &\leq \frac{\omega(t)}{\sqrt{m}}\xi_{\infty,r_1-1}(t)\eta_{i-1,r-r_1}(t)\leq c{\LogLn}^{C}\eta_{i-1,r-r_1}(t),
\end{align*}
and for $q>i$  
\begin{align*}
       \eta_{q,r}(t)&\leq   \eta_{q-1,r}(t)+ \frac{c_{\mathrm{res}}}{L\sqrt{m}}\omega(t)  \xi_{q-1,r}(t)\\
       &\leq \eta_{i,r}(t)+ \frac{c_{\mathrm{res}}}{L\sqrt{m}}\omega(t)  \xi_{q-1,r}(t)+ \frac{c_{\mathrm{res}}}{L\sqrt{m}}\omega(t)  \xi_{q-2,r}(t)+\cdots+ \frac{c_{\mathrm{res}}}{L\sqrt{m}}\omega(t)  \xi_{i,r}(t),
    \end{align*}
then by taking supreme on $q $ and $i,$   
\begin{align*}
    \eta_{\infty,r}(t)&\leq c{\LogLn}^{C}\eta_{\infty,r-r_1}(t)+\frac{4c_{\mathrm{res}}}{\sqrt{m}} \omega(t)\xi_{\infty,r}(t)\\
    &\leq c{\LogLn}^{C}+c{\LogLn}^{C}\leq c{\LogLn}^{C}.
\end{align*}
 Note that from the proof, for different $r,$ the constant $c$ grows exponentially in $r,$ while the growth rate of $C$ is linear.
\end{proof}

\section{Least Eigenvalue of Gram Matrices}\label{appendix section...least eigenvalue for Gram matrixces}
We shall recall the Gram matrices defined in Section \ref{subsection....gram matrices}. We first define a series of matrices $\left\{\widetilde{\mK}^{[l]} \right\}_{l=1}^L,$ $\left\{\widetilde{\mA}^{[l]} \right\}_{l=1}^{L+1},$  and a series of vectors $\left\{\widetilde{\vb}^{[l]} \right\}_{l=1}^L.$
Given the input samples $\fX=\{\vx_1,\vx_2,...,\vx_n\},$ $\Norm{\vx_i}_2=1 ,$  for $1\leq i\leq n,$ and $\vx_i\nparallel \vx_j,$ for any $i\neq j$  
\begin{align}
\widetilde{\mK}^{[0]}_{ij}&=\left<\vx_i,\vx_j\right>,\nonumber\\
\widetilde{\mK}^{[1]}_{ij}&=\Exp_{(u,v)^{\T}\sim \fN\left(\vzero, \begin{pmatrix}\widetilde{\mK}_{ii}^{[0]}&\widetilde{\mK}_{ij}^{[0]}\\
\widetilde{\mK}_{ji}^{[0]}&\widetilde{\mK}_{jj}^{[0]}\end{pmatrix} \right)} c_{\sigma}\sigma(u)\sigma(v),\nonumber\\
\widetilde{\vb}^{[1]}_i&=\sqrt{c_{\sigma}}\Exp_{u\sim \fN(0,\widetilde{\mK}_{ii}^{[0]})} \left[\sigma(u)\right], \nonumber\\
\widetilde{\mA}^{[l]}_{ij}&=\begin{pmatrix}\widetilde{\mK}_{ii}^{[l-1]}&\widetilde{\mK}_{ij}^{[l-1]}\\
\widetilde{\mK}_{ji}^{[l-1]}&\widetilde{\mK}_{jj}^{[l-1]}\end{pmatrix},\nonumber\\
\widetilde{\mK}^{[l]}_{ij}&=\widetilde{\mK}_{ij}^{[l-1]}+\Exp_{(u,v)^{\T}\sim \fN\left(\vzero, \widetilde{\mA}^{[l]}_{ij}\right)}\left[\frac{c_{\mathrm{res}}\widetilde{\vb}_{i}^{[l-1]}\sigma(v) }{L} +\frac{c_{\mathrm{res}}\widetilde{\vb}_{j}^{[l-1]}\sigma(u) }{L}+\frac{c_{\mathrm{res}}^2\sigma(u)\sigma(v)}{L^2} \right],\nonumber\\
\widetilde{\vb}^{[l]}_i&=\widetilde{\vb}_{i}^{[l-1]}+\frac{c_{\mathrm{res}}}{L}\Exp_{u\sim \fN(0,\widetilde{\mK}_{ii}^{[l-1]})} \left[\sigma(u)\right], \nonumber\\
\widetilde{\mA}^{[L+1]}_{ij}&=
\begin{pmatrix}\widetilde{\mK}_{ii}^{[L]}&\widetilde{\mK}_{ij}^{[L]}\\
\widetilde{\mK}_{ji}^{[L]}&\widetilde{\mK}_{jj}^{[L]}\end{pmatrix},\nonumber
\end{align}
given these definitions, we define that for  $2\leq l\leq L-1,$
\begin{align}
\mK^{[L+1]}_{ij}&=\widetilde{\mK}_{ij}^{[L]}+\Exp_{(u,v)^{\T}\sim \fN\left(\vzero, \widetilde{\mA}^{[L+1]}_{ij}\right)}\left[\frac{c_{\mathrm{res}}\widetilde{\vb}_{i}^{[L]}\sigma(v) }{L} +\frac{c_{\mathrm{res}}\widetilde{\vb}_{j}^{[L]}\sigma(u) }{L}+\frac{c_{\mathrm{res}}^2\sigma(u)\sigma(v)}{L^2} \right],\label{matrix L+1}\\
{\mK}^{[L]}_{ij}&= \frac{c_{\mathrm{res}}^2}{L^2} \widetilde{\mK}_{ij}^{[L-1]}\Exp_{(u,v)^{\T}\sim \fN\left(\vzero, \widetilde{\mA}^{[L]}_{ij}\right)}\left[ \sigma^{(1)}(u)\sigma^{(1)}(v)\right], \label{matrix L}\\
 {\mK}^{[l]}_{ij}&=\frac{c_{\mathrm{res}}^2}{L^2}\widetilde{\mK}_{ij}^{[l-1]} \lim_{m\to\infty}\frac{1}{m}\left<\vsigma^{(1)}_{[l]}(\vx_{i})  \left(\mE_{0,i}^{[(l+1):L]}\right)^{\T} \va_0,\vsigma^{(1)}_{[l]}(\vx_{j})  \left(\mE_{0,j}^{[(l+1):L]}\right)^{\T} \va_0\right>,  \label{matrix l}\\
  {\mK}^{[1]}_{ij}&=c_{\sigma}\widetilde{\mK}_{ij}^{[0]} \lim_{m\to\infty}\frac{1}{m}\left<\vsigma^{(1)}_{[1]}(\vx_{i})  \left(\mE_{0,i}^{[2:L]}\right)^{\T} \va_0,\vsigma^{(1)}_{[1]}(\vx_{j})  \left(\mE_{0,j}^{[2:L]}\right)^{\T} \va_0\right>.\label{matrix 1}
\end{align}
We shall state two lemmas concerning full rankness of the Gram matrices, which have been stated as Lemma $\mathrm{F.1.}$ and Lemma $\mathrm{F.2.}$ in Du et al.~\cite{Du2018Gradient}.
\begin{lem}\label{lemma... gram matrix witout derivative}
Assume $\sigma(\cdot)$ is analytic and not a polynomial function. Consider input data set as 
$\fV=\{\vv_1,\vv_2,\dots,\vv_n\}$, and  non-parallel with each other, i.e. $\vv_j\notin \mathrm{span}\left(\vv_k\right)$ for any $j\neq k$, we  define
\begin{equation}\label{eq...gram matrix witout derivative}
    \mG(\fV)_{ij}:=\Exp_{\vw\sim \fN\left(\vzero,  \mI \right)}\left[ \sigma(\vw^{\T}\vv_i)\sigma(\vw^{\T}\vv_j)\right],
\end{equation}
then $\lambda_{\min}\left(\mG(\fV)\right)>0.$
\end{lem}
Similar to Lemma \ref{lemma... gram matrix witout derivative}, we have Lemma \ref{lemma... gram matrix with derivative}
\begin{lem}\label{lemma... gram matrix with derivative}
Assume $\sigma(\cdot)$ is analytic and not a polynomial function. Consider input data set as 
$\fV=\{\vv_1,\vv_2,\dots,\vv_n\}$, and  non-parallel with each other, i.e. $\vv_j\notin \mathrm{span}\left(\vv_k\right)$ for any $j\neq k$, we  define
\begin{equation}\label{eq...gram matrix with derivative}
    \mG(\fV)_{ij}:=\Exp_{\vw\sim \fN\left(\vzero, \mI \right)}\left[ \sigma^{(1)}(\vw^{\T}\vv_i)\sigma^{(1)}(\vw^{\T}\vv_j)\left(\vv_i^{\T}\vv_j\right)\right],
\end{equation}
then $\lambda_{\min}\left(\mG(\fV)\right)>0.$
\end{lem}
Now we proceed to quantify the least eigenvalues of these Gram matrices.
\subsection{Full Rankness for $(L+1)$-th Gram matrix  }\label{appendix subsection.....full rankness for L+1 gram matrix}
We begin this part by a lemma on the estimate of the entry of Gram matrices,
\begin{lem}\label{lemma...identical entry}
Given the input samples $\fX=\{\vx_1,\vx_2,...,\vx_n\},$ $\Norm{\vx_i}_2=1, 1\leq i\leq n,$ and $\vx_i\nparallel \vx_j,$ for any $i\neq j$, then for every fixed $l$, where $1\leq l \leq L,$ each diagonal entry of  $\widetilde{\mK}^{[l]}$ is the same with each other. Also  for every fixed $l,$ where $1\leq l \leq L$, each element of the vector $\widetilde{\vb}^{[l]}$ is the same with each other, i.e.,$$\widetilde{\mK}^{[l_1]}_{ii}=\widetilde{\mK}^{[l_1]}_{jj},\widetilde{\vb}_i^{[l_2]}=\widetilde{\vb}_j^{[l_2]} , i\neq j.$$ 
Moreover  
\begin{equation}\label{k_ii entry induction wrt layer}
   \left(1-\frac{l}{L}\frac{c}{\sqrt{c_{\sigma}}}\right)^{2}\leq \widetilde{\mK}_{ii}^{[l]}\leq \left(1+\frac{l}{L}\frac{c}{\sqrt{c_{\sigma}}}\right)^{2},
\end{equation}
and 
\begin{equation}\label{b_ii entry induction wrt layer}
\left(\widetilde{\vb}^{[l]}_i\right)^2<{ \widetilde{\mK}_{ii}^{[l]}},
\end{equation}
where $c>0$ and only depends on $c_{\mathrm{res}}$ and the activation function $\sigma(\cdot).$
\end{lem}
\begin{proof}
We shall prove it by induction on $l.$
Firstly, we notice that $\widetilde{\mK}^{[0]}_{ii}=\widetilde{\mK}^{[0]}_{jj}$ for any $i\neq j,$ this is obvious because  $\Norm{\vx_{i}}_2=1 ,$ then $\widetilde{\mK}^{[0]}_{ii}=\widetilde{\mK}^{[0]}_{jj}=1.$
Next we show that it holds true for $l=1.$

Since based on definition, recall that $c_{\sigma}=\left(\Exp_{x\sim \fN(0,1)} \left[\sigma(x)^2\right]\right)^{-1},$ 
\begin{align*}
    {\mK}^{[1]}_{ii}&=c_{\sigma}\Exp_{u\sim\fN(0,\widetilde{\mK}^{[0]}_{ii})}\left(\sigma(u)^2\right)=c_{\sigma}\Exp_{u\sim\fN(0,1)}\left(\sigma(u)^2\right)=1,
\end{align*}
and 
\begin{align*}
    \widetilde{\vb}^{[1]}_i&=\sqrt{c_{\sigma}}\Exp_{u\sim \fN(0,\widetilde{\mK}_{ii}^{[0]})} \left[\sigma(u)\right]=\sqrt{c_{\sigma}}\Exp_{u\sim \fN(0,1)} \left[\sigma(u)\right],
\end{align*}
then 
\begin{align*}
    \left(\widetilde{\vb_i}^{[1]}\right)^2&={c_{\sigma}}\left(\Exp_{u\sim \fN(0,\widetilde{\mK}_{ii}^{[0]})} \left[\sigma(u)\right]\right)^2<1,
\end{align*}
the last inequality holds because 
$$\left(\Exp_{x\sim \fN(0,1)} \left[\sigma(x)\right]\right)^2< \Exp_{x\sim \fN(0,1)} \left[\sigma(x)^2\right],$$
since the quantity is independent of our choice of $i,$ then $\widetilde{\mK}^{[1]}_{ii}=\widetilde{\mK}^{[1]}_{jj},\widetilde{\vb}_i^{[1]}=\widetilde{\vb}_j^{[1]} ,$ for any $i\neq j.$

Now we assume that it holds for $1,2,\cdots, l-1 $ and want to show that it holds for $l.$ Hence based on definition
\begin{align*}
    \widetilde{\mK}^{[l]}_{ii}&=\widetilde{\mK}_{ii}^{[l-1]}+\Exp_{u\sim\fN(0,\widetilde{\mK}_{ii}^{[l-1]})}\left[\frac{c_{\mathrm{res}}\widetilde{\vb}_{i}^{[l-1]}\sigma(u) }{L} +\frac{c_{\mathrm{res}}\widetilde{\vb}_{i}^{[l-1]}\sigma(u) }{L}+\frac{c_{\mathrm{res}}^2\sigma(u)\sigma(u)}{L^2} \right],\nonumber\\
\widetilde{\vb}^{[l]}_i&=\widetilde{\vb}_{i}^{[l-1]}+\frac{c_{\mathrm{res}}}{L}\Exp_{u\sim \fN(0,\widetilde{\mK}_{ii}^{[l-1]})} \left[\sigma(u)\right], 
\end{align*}
such quantities are also independent of our choice of $i.$ 

Moreover we would like to show that \eqref{k_ii entry induction wrt layer} and \eqref{b_ii entry induction wrt layer} hold for all $l.$

Firstly, for $\widetilde{\vb}^{[l]}_i,$ assume \eqref{b_ii entry induction wrt layer} holds for $1,2,\cdots, l-1 ,$ then we have 
\begin{align*}
   \left( \widetilde{\vb}^{[l]}_i\right)^2 &=\left(\widetilde{\vb}^{[l-1]}_i\right)^2+2 \widetilde{\vb}^{[l-1]}_i\frac{c_{\mathrm{res}}}{L}\Exp_{u\sim \fN(0,\widetilde{\mK}_{ii}^{[l-1]})} \left[\sigma(u)\right]+\left(\frac{c_{\mathrm{res}}}{L}\Exp_{u\sim \fN(0,\widetilde{\mK}_{ii}^{[l-1]})} \left[\sigma(u)\right]\right)^2\\
    &< \left(\widetilde{\vb}^{[l-1]}_i\right)^2+2 \widetilde{\vb}^{[l-1]}_i\frac{c_{\mathrm{res}}}{L}\Exp_{u\sim \fN(0,\widetilde{\mK}_{ii}^{[l-1]})} \left[\sigma(u)\right]+\frac{c_{\mathrm{res}}^2}{L^2}\Exp_{u\sim \fN(0,\widetilde{\mK}_{ii}^{[l-1]})} \left[\sigma(u)^2\right]\\
    &<\widetilde{\mK}_{ii}^{[l-1]}+2 \widetilde{\vb}^{[l-1]}_i\frac{c_{\mathrm{res}}}{L}\Exp_{u\sim \fN(0,\widetilde{\mK}_{ii}^{[l-1]})} \left[\sigma(u)\right]+\frac{c_{\mathrm{res}}^2}{L^2}\Exp_{u\sim \fN(0,\widetilde{\mK}_{ii}^{[l-1]})} \left[\sigma(u)^2\right]=\widetilde{\mK}_{ii}^{[l]},
\end{align*}
showing that \eqref{b_ii entry induction wrt layer} holds for $l.$

For $\widetilde{\mK}_{ii}^{[l]},$ we have 
\begin{align}
  \left(\sqrt{\widetilde{\mK}_{ii}^{[l-1]}}-\frac{c_{\mathrm{res}}}{L}\sqrt{\Exp_{u\sim\fN(0,\widetilde{\mK}_{ii}^{[l-1]})}\left[\sigma(u)^2\right]}\right)^2 \leq  \widetilde{\mK}_{ii}^{[l]} \leq \left(\sqrt{\widetilde{\mK}_{ii}^{[l-1]}}+\frac{c_{\mathrm{res}}}{L}\sqrt{\Exp_{u\sim\fN(0,\widetilde{\mK}_{ii}^{[l-1]})}\left[\sigma(u)^2\right]}\right)^2,\label{eq in prop.....in between}
\end{align}
since $\sigma(\cdot)$ is $C_L$-Lipschitz, then for any $1/2\leq\alpha\leq 2,$ we have
\begin{align*}
    &\Abs{\Exp_{X\sim \fN(0,1)}\left[\sigma(\alpha X)^2\right]-\Exp_{X\sim \fN(0,1)}\left[\sigma( X)^2\right]}\\
    &\leq \Exp_{X\sim \fN(0,1)}\left[\Abs{\sigma(\alpha X)^2-\sigma( X)^2}\right]\\
    &\leq C_L \Abs{\alpha-1} \Exp_{X\sim \fN(0,1)}\left[\Abs{X\left(\sigma\left(\alpha X\right)+\sigma\left( X\right)\right)}\right]\\
    &\leq C_L \Abs{\alpha-1} \Exp_{X\sim \fN(0,1)}\left[\Abs{X}\Abs{2\sigma(0)}\right]+C_L\Abs{\alpha+1}\Exp_{X\sim \fN(0,1)}\left[X^2\right]\\
    &= C_L \Abs{\alpha-1} \left(\Abs{2\sigma(0)} \sqrt{\frac{2}{\pi}}+C_L \Abs{\alpha+1}\right)\\
    &\leq \frac{C}{c_{\sigma}}\Abs{\alpha-1},
\end{align*}
then
\begin{equation*}
 \Exp_{X\sim \fN(0,1)}\left[\sigma(\alpha X)^2\right]\leq \frac{1}{c_{\sigma}}+\frac{C}{c_{\sigma}}\Abs{\alpha-1},
\end{equation*}
by induction  
\begin{equation*}
 1-\frac{l-1}{L} \frac{c}{\sqrt{c_{\sigma}}} \leq  \sqrt{\widetilde{\mK}_{ii}^{[l-1]} } \leq 1+\frac{l-1}{L} \frac{c}{\sqrt{c_{\sigma}}},
\end{equation*}
set $\alpha=\sqrt{\widetilde{\mK}_{ii}^{[l-1]} } ,$ we obtain
\begin{equation*}
     \Exp_{X\sim \fN(0,\widetilde{\mK}_{ii}^{[l-1]})}\left[\sigma(X)^2\right]\leq \frac{1}{c_{\sigma}}+\frac{C}{c_{\sigma}}\frac{l-1}{L} \frac{c}{\sqrt{c_{\sigma}}},
\end{equation*}
then if we choose $c$ wisely,  let 
\begin{equation*}
c=\frac{C c_{\mathrm{res}}^2}{2\sqrt{c_{\sigma}}}+\sqrt{\frac{C^2 c_{\mathrm{res}}^4}{4c_{\sigma}}+c_{\mathrm{res}}^2},
\end{equation*}
by our choice of $c,$ combined with \eqref{eq in prop.....in between} 
\begin{align*}
\left(\sqrt{\widetilde{\mK}_{ii}^{[l-1]}}-\frac{1}{L}\frac{c}{\sqrt{c_{\sigma}}}\right)^2 \leq  \widetilde{\mK}_{ii}^{[l]} \leq \left(\sqrt{\widetilde{\mK}_{ii}^{[l-1]}}+\frac{1}{L}\frac{c}{\sqrt{c_{\sigma}}}\right)^2,
\end{align*}
then   
\begin{equation*}
 \left(1-\frac{l}{L} \frac{c}{\sqrt{c_{\sigma}}}\right)^2 \leq  {\widetilde{\mK}_{ii}^{[l-1]} } \leq \left(1+\frac{l}{L} \frac{c}{\sqrt{c_{\sigma}}}\right)^2 ,
\end{equation*}
which finishes our proof.
\end{proof}
Our next lemma is crucial in that it revels a `covariance-type' structure for the Gram matrices. We need to introduce a standard notation related to matrices. We denote that $\mA\succeq \mB$ if and only if $\mA-\mB$ is a semi-positive definite matrix, and $\mA\succ \mB$ if and only if $\mA-\mB$ is a strictly positive definite matrix.
\begin{prop}\label{proposition..covarince type}
Given the input samples $\fX=\{\vx_1,\vx_2,...,\vx_n\},$ $\Norm{\vx_i}_2=1 ,$  for $1\leq i\leq n,$ and $\vx_i\nparallel \vx_j, i\neq j$, then we have for every fixed $l,$ where $1\leq l \leq L,$
\begin{align}
    &\widetilde{\mK}^{[l]}-\widetilde{\vb}^{[l]}\otimes \left(\widetilde{\vb}^{[l]}\right)^\T\succ \widetilde{\mK}^{[l-1]}-\widetilde{\vb}^{[l-1]}\otimes \left(\widetilde{\vb}^{[l-1]}\right)^\T.\label{ineq..matrix comparision hierachy}
\end{align}
Moreover, since $$\widetilde{\mK}^{[1]}-\widetilde{\vb}^{[1]}\otimes \left(\widetilde{\vb}^{[1]}\right)^\T \succ 0,$$
we denote that
\begin{align}
    & \lambda_{\min}\left(\widetilde{\mK}^{[1]}-\widetilde{\vb}^{[1]}\otimes \left(\widetilde{\vb}^{[1]}\right)^\T\right)=\lambda_0,\label{ineq..matrix comparision hierachy at 0}
\end{align}
then we can conclude that for $2\leq l \leq L,$
\begin{equation}\label{equation..uniform bounded for eigenvalue}
    \lambda_{\min}\left( \widetilde{\mK}^{[l]}\right) > \lambda_0,
\end{equation}
where $\lambda_0$ only depends on the activation function and input data and independent of depth $L$.
\end{prop}
\begin{proof}
We only need to show that for $1\leq i,j\leq n$ and  $1\leq  l \leq L$
\begin{align*}
&~~\widetilde{\mK}_ {ij}^{[l]}-\widetilde{\vb}_{i}^{[l]}\widetilde{\vb}_{j}^{[l]}\\
&=\widetilde{\mK}_ {ij}^{[l-1]}+\Exp_{(u,v)^{\T}\sim \fN\left(\vzero, \begin{pmatrix}\widetilde{\mK}_ {ii}^{[l-1]}&\widetilde{\mK}_ {ij}^{[l-1]}\\
\widetilde{\mK}_ {ji}^{[l-1]}&\widetilde{\mK}_ {jj}^{[l-1]}\end{pmatrix}\right)}\left[\frac{c_{\mathrm{res}}\widetilde{\vb}_{i}^{[l-1]}\sigma(v) }{L} +\frac{c_{\mathrm{res}}\widetilde{\vb}_{j}^{[l-1]}\sigma(u) }{L}+\frac{c_{\mathrm{res}}^2\sigma(u)\sigma(v)}{L^2} \right]\\
&~~~~-\left(\widetilde{\vb}_{i}^{[l-1]}+\frac{c_{\mathrm{res}}}{L}\Exp_{u\sim \fN(0, \widetilde{\mK}_ {ii}^{[l-1]})} \left[\sigma(u)\right]\right)\left(\widetilde{\vb}_{j}^{[l-1]}+\frac{c_{\mathrm{res}}}{L}\Exp_{v\sim \fN(0, \widetilde{\mK}_ {jj}^{[l-1]})} \left[\sigma(v)\right]\right)\\
&=\widetilde{\mK}_ {ij}^{[l-1]}-\widetilde{\vb}_{i}^{[l-1]}\widetilde{\vb}_{j}^{[l-1]}+\Exp_{(u,v)^{\T}\sim \fN\left(\vzero,  \begin{pmatrix}\widetilde{\mK}_ {ii}^{[l-1]}&\widetilde{\mK}_ {ij}^{[l-1]}\\
\widetilde{\mK}_ {ji}^{[l-1]}&\widetilde{\mK}_ {jj}^{[l-1]}\end{pmatrix}\right)}\left[\frac{c_{\mathrm{res}}^2\sigma(u)\sigma(v)}{L^2} \right]\\
&~~~~-\frac{c_{\mathrm{res}}}{L}\Exp_{u\sim \fN(0, \widetilde{\mK}_ {ii}^{[l-1]})} \left[\sigma(u)\right]\frac{c_{\mathrm{res}}}{L}\Exp_{v\sim \fN(0, \widetilde{\mK}_ {jj}^{[l-1]})} \left[\sigma(v)\right]\\
&=\widetilde{\mK}_ {ij}^{[l-1]}-\widetilde{\vb}_{i}^{[l-1]}\widetilde{\vb}_{j}^{[l-1]}+\frac{c_{\mathrm{res}}^2}{L^2}\mathrm{Cov}_{(u,v)^{\T}\sim \fN\left(\vzero,  \begin{pmatrix}\widetilde{\mK}_ {ii}^{[l-1]}&\widetilde{\mK}_ {ij}^{[l-1]}\\
\widetilde{\mK}_ {ji}^{[l-1]}&\widetilde{\mK}_ {jj}^{[l-1]}\end{pmatrix}\right)}\left[\sigma(u)\sigma(v) \right],
\end{align*}
which brings us to the definition of a series of covariance matrices $\left\{\mP^{[s]}:1\leq s \leq L\right\},$
\begin{equation*}
\mP^{[s]}_{ij}:= \frac{c_{\mathrm{res}}^2}{L^2}\mathrm{Cov}_{(u,v)^{\T}\sim \fN\left(\vzero, \begin{pmatrix}\widetilde{\mK}_{ii}^{[s]}&\widetilde{\mK}_{ij}^{[s]}\\
\widetilde{\mK}_{ji}^{[s]}&\widetilde{\mK}_{jj}^{[s]}\end{pmatrix}\right)}\left[\sigma(u)\sigma(v) \right], \  1\leq s \leq L, 
\end{equation*}
$\mP^{[s]}$ are covariance matrices, naturally we have $\mP^{[s]}\succeq 0$, and $\mP^{[s]}\succ 0$ except that  one sample is an exact linear function of the others. Apply Lemma \ref{lemma... gram matrix witout derivative} directly, we can guarantee that $\mP^{[s]}$ is  positive definite for every $s$. Hence, inductively we have
\begin{align*}
  \widetilde{\mK}^{[l]} &\succeq\widetilde{\mK}^{[l]}-\widetilde{\vb}^{[l]}\otimes \left(\widetilde{\vb}^{[l]}\right)^\T\\
  &=\widetilde{\mK}^{[l-1]}-\widetilde{\vb}^{[l-1]}\otimes \left(\widetilde{\vb}^{[l-1]}\right)^\T+\mP^{[l-1]}\\
  &\succ \widetilde{\mK}^{[l-1]}-\widetilde{\vb}^{[l-1]}\otimes \left(\widetilde{\vb}^{[l-1]}\right)^\T\\
  &=\widetilde{\mK}^{[l-2]}-\widetilde{\vb}^{[l-2]}\otimes \left(\widetilde{\vb}^{[l-2]}\right)^\T+\mP^{[l-2]}\\
  &~~~~\vdots\\
  &\succ \widetilde{\mK}^{[1]}-\widetilde{\vb}^{[1]}\otimes \left(\widetilde{\vb}^{[1]}\right)^\T,
\end{align*}
the last line brings us to the entry of  $\widetilde{\mK}^{[1]}-\widetilde{\vb}^{[1]}\otimes \left(\widetilde{\vb}^{[1]}\right)^\T$, we have that 
\begin{equation*}
   \left(\widetilde{\mK}^{[1]}-\widetilde{\vb}^{[1]}\otimes \left(\widetilde{\vb}^{[1]}\right)^\T\right)_{ij}=  c_{\sigma}\mathrm{Cov}_{(u,v)^{\T}\sim \fN\left(\vzero, \begin{pmatrix}\widetilde{\mK}_{ii}^{[0]}&\widetilde{\mK}_{ij}^{[0]}\\
\widetilde{\mK}_{ji}^{[0]}&\widetilde{\mK}_{jj}^{[0]}\end{pmatrix}\right)}\left[\sigma(u)\sigma(v) \right] ,
\end{equation*}
then apply Lemma \ref{lemma... gram matrix witout derivative} again   $$\lambda_{\min}\left(\widetilde{\mK}^{[1]}-\widetilde{\vb}^{[1]}\otimes \left(\widetilde{\vb}^{[1]}\right)^\T\right)=\lambda_0>0,$$ 
and $\lambda_0$ only depends on the input data and activation function.
\end{proof}
\begin{cor}\label{corollary cocnernong L+1}
Given the input samples $\fX=\{\vx_1,\vx_2,...,\vx_n\},$ $\Norm{\vx_i}_2=1 ,$  for $1\leq i\leq n,$ and $\vx_i\nparallel \vx_j, i\neq j$, then we have
\begin{equation}\label{equation..uniform bounded for eigenvalue continued to L+1}
    \lambda_{\min}\left( {\mK}^{[L+1]}\right) > \lambda_0,
\end{equation}
where $\lambda_0$ has been defined in \eqref{ineq..matrix comparision hierachy at 0}.
\end{cor}
\begin{proof}
The proof is quite similar to the proof of Proposition \ref{proposition..covarince type}, recall that 
\begin{align*}
    \mK^{[L+1]}_{ij}&=\widetilde{\mK}_{ij}^{[L]}+\Exp_{(u,v)^{\T}\sim \fN\left(\vzero, \widetilde{\mA}^{[L+1]}_{ij}\right)}\left[\frac{c_{\mathrm{res}}\widetilde{\vb}_{i}^{[L]}\sigma(v) }{L} +\frac{c_{\mathrm{res}}\widetilde{\vb}_{j}^{[L]}\sigma(u) }{L}+\frac{c_{\mathrm{res}}^2\sigma(u)\sigma(v)}{L^2} \right],
\end{align*}
and we define that 
\begin{align*}
    \vb^{[L+1]}_i:=\widetilde{\vb}_{i}^{[L]}+\frac{c_{\mathrm{res}}}{L}\Exp_{u\sim \fN(0,\widetilde{\mK}_{ii}^{[L]})} \left[\sigma(u)\right], 
\end{align*}
then 
\begin{align*}
     \mK^{[L+1]}_{ij}- \vb^{[L+1]}_i \vb^{[L+1]}_j&=\widetilde{\mK}_ {ij}^{[L]}-\widetilde{\vb}_{i}^{[L]}\widetilde{\vb}_{j}^{[L]}+\frac{c_{\mathrm{res}}^2}{L^2}\mathrm{Cov}_{(u,v)^{\T}\sim \fN\left(\vzero,  \begin{pmatrix}\widetilde{\mK}_ {ii}^{[L]}&\widetilde{\mK}_ {ij}^{[L]}\\
\widetilde{\mK}_ {ji}^{[L]}&\widetilde{\mK}_ {jj}^{[L]}\end{pmatrix}\right)}\left[\sigma(u)\sigma(v) \right],
\end{align*}
hence  
\begin{align*}
     \mK^{[L+1]}&\succeq \mK^{[L+1]}-\vb^{[L+1]}\otimes \left(\vb^{[L+1]}\right)^\T\\
     &\succ \widetilde{\mK}^{[L]}-\widetilde{\vb}^{[L]}\otimes \left(\widetilde{\vb}^{[L]}\right)^\T,
\end{align*}
apply Proposition \ref{proposition..covarince type} directly, we are able to finish the proof.
\end{proof}
By Corollary \ref{corollary cocnernong L+1}, we see that $\lambda_{\min}\left(\mK^{[L+1]}\right)\sim \Omega(1).$
 
 \subsection{Full Rankness for the $2$-nd   Gram matrix}\label{appendix subsection.....full rankness for L th gram matrix}
Our next Proposition is  related to the eigenvalue of the $L$-th Gram matrix, whose entries concerning the derivative of the activation function. This Proposition has been stated as Proposition $\mathrm{F.2}$ in Du et al.~\cite{Du2018Gradient}, and we will mimic its proof.
\begin{prop}\label{proposition..Lth gram matrix}
Given the input samples $\fX=\{\vx_1,\vx_2,...,\vx_n\},$ $\Norm{\vx_i}_2=1,$  for $1\leq i\leq n,$ and $\vx_i\nparallel \vx_j, i\neq j,$ then for $2\leq l \leq L$
\begin{equation}\label{equation..uniform bounded for eigenvalue continued to L }
    \lambda_{\min}\left( {\mK}^{[l]}\right) \geq  \frac{c_{\mathrm{res}}^2}{L^2} \kappa,
\end{equation}
where $\kappa$ is a constant that only depends on $\sigma(\cdot)$ and input samples, independent of depth $L.$
\end{prop}
\begin{proof}
Based on Lemma \ref{lemma...identical entry},  uniformly for any $1\leq l \leq L,$
\begin{equation*}
   1/c \leq \widetilde{\mK}_{ii}^{[l]} \leq c,
\end{equation*}
then we can define a function $\mG:\sR^{n\times n} \to \sR^{n\times n},$ such that 
\begin{align*}
    \mG(\mK)_{ij}:={\mK}_{ij}\Exp_{(u,v)^{\T}\sim \fN\left(\vzero, \begin{pmatrix}{\mK}_{ii}&{\mK}_{ij}\\
{\mK}_{ji}&{\mK}_{jj}\end{pmatrix} \right)}\sigma^{(1)}(u)\sigma^{(1)}(v),
\end{align*}
consequently, a scalar function $g(\lambda)$ can be defined as follows:
\begin{align*}
    g(\lambda):= \min_{\mK: \mK\succ 0, 1/c\leq \mK_{ii}\leq c, \lambda_{\min}\left(\mK\right)\geq \lambda} \lambda_{\min}\left(\mG(\mK)\right),
\end{align*}
then Lemma \ref{lemma... gram matrix with derivative} guarantees that 
\begin{equation*}
     g(\lambda_0)>0,
\end{equation*}
moreover, based on Proposition \ref{proposition..covarince type}   $$\lambda_{\min}\left( \widetilde{\mK}^{[L-1]}\right) > \lambda_0,$$
hence we have 
\begin{equation}
    \lambda_{\min}\left({\mK}^{[L]}\right) \geq \frac{c_{\mathrm{res}}^2}{L^2} g(\lambda_0),
\end{equation}
let $\kappa=g(\lambda_0),$ since $\kappa$ is independent of depth $L,$ we finish our proof. 
\end{proof}
By Proposition \ref{proposition..Lth gram matrix}, we see that $\lambda_{\min}\left(\mK^{[L]}\right)\sim \Omega(\frac{1}{L^2}).$
\section{Random Initialization of Gram Matrices} \label{appendix section,,,,,random gram}
In this part, we are going to show that with high probability w.r.t the random initialization, 
$$\lambda_{\min}\left[\fG_t^{[L+1]}\left(\vx_{\alpha},\vx_{\beta}\right)\right]_{1\leq \alpha,\beta\leq n}>\frac{3\lambda_0}{4},$$
where $\lambda_0$ is defined in \eqref{ineq..matrix comparision hierachy at 0}.

Let's get started  with a lemma concerning the Gaussian concentrations.
\begin{lem}\label{lemma..Gaussian concentration inequality for Lipschitz function} 
Let $\vX = (X_1,\cdots X_p)\in \sR^p, X_1,\cdots X_p $ be a vector of i.i.d.  Gaussian variables from $\fN(0,\sigma^2) ,$
and let $f(\cdot) : \sR^p \to \sR$ be  $L$-Lipschitz function, i.e. $\Abs{f(\vx)-f(\vy)} \leq L \Norm{\vx-\vy}_2$ for all $\vx, \vy \in \sR^p$, then for any $t \geq 0$ 
\begin{equation}\label{ineq...concentration tail inequality on subgaussian}
    \Prob\left(\Abs{f(\vX) - \Exp f(\vX)} \geq t\right) \leq 2 \exp(-\frac{t^2}{2L^2\sigma^2}).
\end{equation}
\end{lem}
Before we proceed to the stability of the randomly initialized Gram matrix of higher order, we need to state two lemmas. The first lemma has been stated as Lemma $\mathrm{G.3.}$ in Du et al. \cite{Du2018Gradient},
\begin{lem}\label{lemma....one expectation constant}
If $\sigma(\cdot)$ is $C_L$-Lipschitz, then for $a,b\in\sR^{+},$ with $1/c\leq\min(a,b), \max(a,b)\leq~c$ for some $c>0,$ then we have 
\begin{align}
   \Abs{ \Exp_{z\sim \fN(0,1)}\left[\sigma(az)\right]-  \Exp_{z\sim \fN(0,1)}\left[\sigma(bz)\right]}\leq C\Abs{a-b},
\end{align}
where $C>0$ only depends on $c$ and Lipschitz constant $C_L.$
\end{lem}
Next lemma has been stated as Lemma $\mathrm{G.4.}$ in Du et al. \cite{Du2018Gradient},
\begin{lem}\label{lemma....two expectation constant}
If $\sigma(\cdot)$ is $C_L$-Lipschitz, define  a scalar function $F(\mK)$ as follows:
\begin{align*}
    F(\mK)=\Exp_{(u,v)^{\T}\sim \fN\left(\vzero, \mK\right)}\left[\sigma(u)\sigma(v) \right],
\end{align*}
then for any two matrices $\mA,\mB$ being
\begin{align*}
    \mA&=\begin{pmatrix}a_1^2&\rho_1 a_1b_1\\
\rho_1 a_1b_1 &b_1^2\end{pmatrix},\\
\mB&=\begin{pmatrix}a_2^2&\rho_2 a_2b_2\\
\rho_2 a_2b_2 &b_2^2\end{pmatrix},
\end{align*}
and their entries satisfying
$$1/c \leq \min(a_1,b_1), \min(a_2,b_2), \max(a_1,b_1), \max(a_2,b_2) \leq c, \ $$
and
$$-1< \rho_1, \rho_2 < 1$$
for some $c>0,$ then we have 
\begin{align*}
    \Abs{F(\mA)-F(\mB)} \leq C\Norm{\mA-\mB}_{\mathrm{F}}\leq 2C\Norm{\mA-\mB}_{\infty},
\end{align*}
where the constant $C>0$  only relies on $c$ and the Lipschitz constant $C_L.$
\end{lem}
We shall begin with a proposition on the initial estimate of the output of  each layer $\vx_j^{[l]}(0),$ 
\begin{prop}\label{proposition..................Initialization Norms layer 1}
Under Assumption \ref{Assump...Assumption on activation functions} and \ref{Assump... ont he imput of the  samples}, we have that for some $t>0$, $1\leq i\leq n$, $1\leq l \leq L$
\begin{align}
    \Prob\left(\Abs{\Norm{\vx_i^{[l]}(0)}_2-\sqrt{\widetilde{\mK}_{ii}^{[l]}}}\geq t\right) &\leq \exp\left(-cmt^2\right),\label{concentration on  norm diaongal }\\
     \Prob\left(\Abs{\left<\frac{\vx_i^{[l]}(0)}{\sqrt{m}},\vone\right>-\widetilde{\vb}_i^{[l]}}\geq t\right) &\leq \exp\left(-cmt^2\right),\label{concentration on vector}
\end{align}
where $c>0$ is a constant independent of depth $L.$
\end{prop}
 \begin{proof}
For $l=1,$ we have
\begin{equation*}
\Norm{\vx_i^{[1]}(0)}^2_2={\frac{c_{\sigma} }{m}\sum_{j=1}^m\left(\sigma(\mW^{[1]}(0)\vx_i)_j\right)^2},
\end{equation*}
then 
\begin{equation*}
     \Exp\left[\Norm{\vx_i^{[1]}(0)}^2_2\right] =c_{\sigma}{\Exp_{x\sim \fN(0,1)}\left[\sigma(x)^2\right]}=\widetilde{\mK}_{ii}^{[1]}=1 ,
\end{equation*}
since $\left(\mW^{[1]}(0)\vx_i\right)_j$ are i.i.d standard Gaussian variables, and $\sigma(\cdot)$ is $1$-Lipschitz, then $\left(\sigma(\mW^{[1]}(0)\vx_i)_j\right)$ are sub-exponential variables, then we have  for $\lambda>0,$
\begin{align*}
    \Exp\left[ \exp  \lambda\left(m\Norm{\vx_i^{[1]}(0)}^2_2-m \right)  \right]\leq \exp\left( c m \lambda^2 \right),
\end{align*}
hence applying Markov inequality directly  
\begin{align*}
   \Prob\left(\Abs{\Norm{\vx_i^{[1]}(0)}_2-\sqrt{\widetilde{\mK}_{ii}^{[1]}}}\geq t\right) &\leq \Prob\left(\Abs{\Norm{\vx_i^{[1]}(0)}^2_2-1  }\geq 2t\right) \leq \exp\left(-c mt^2\right),
\end{align*}
and 
\begin{align*}
  \left<\frac{\vx_i^{[1]}(0)}{\sqrt{m}},\vone\right> &=\frac{\sqrt{c_{\sigma}}}{m}\sum_{j=1}^m\left(\sigma(\mW^{[1]}(0)\vx_i)_j\right),
\end{align*}
then
\begin{align*}
  \Exp\left[\left<\frac{\vx_i^{[1]}(0)}{\sqrt{m}},\vone\right>\right] &=\widetilde{\vb_i}^{[1]},
\end{align*}
we should note that $\vx_i^{[1]}(0)$ writes into 
$$\vx_i^{[1]}(0)=\sqrt{\frac{c_{\sigma}}{m}}\sigma\left(\vX\right),$$
with $\vX$ being a standard normal Gaussian vector, we shall focus on the inner product function $g^{[1]}(\cdot): \sR^m\to \sR,$ with $$g^{[1]}(\vX)=\frac{\sqrt{c_{\sigma}}}{{m}}\left<\sigma\left(\vX\right), \vone\right>,$$ 
we have  for any $\vX_1,\vX_2 \in \sR^m,$
\begin{align*}
    \Abs{g^{[1]}(\vX_1)-g^{[1]}(\vX_2)}&\leq \Abs{\frac{\sqrt{c_{\sigma}}}{{m}}\left<\sigma\left(\vX_1\right), \vone\right>-\frac{\sqrt{c_{\sigma}}}{{m}}\left<\sigma\left(\vX_2\right), \vone\right>}\\
    &\leq \frac{\sqrt{c_{\sigma}}}{{m}}\left<\Abs{\vX_1-\vX_2},\vone\right>\leq \sqrt{\frac{c_{\sigma}}{m}}\Norm{\vX_1-\vX_2}_2,
\end{align*}
hence $g^{[1]}(\cdot)$ is $\frac{C}{\sqrt{m}}$-Lipschitz, then apply Lemma \ref{lemma..Gaussian concentration inequality for Lipschitz function}  
\begin{equation*}
    \Prob\left(\Abs{{g}^{[1]}(\vX) - \Exp {g}^{[1]}(\vX)} \geq t\right)\leq  \exp(-cmt^2),
\end{equation*}
then we have
\begin{align*}
   \Prob\left(\Abs{ \left<\frac{\vx_i^{[1]}(0)}{\sqrt{m}},\vone\right>-\widetilde{\vb_i}^{[1]} }\geq t\right)\leq  \exp(-cmt^2).
\end{align*}
Our next step is to prove that \eqref{concentration on  norm diaongal } and \eqref{concentration on vector} hold for $l\geq 2,$ and we will prove it by induction.

Assume that \eqref{concentration on  norm diaongal } and \eqref{concentration on vector} hold for $1,2,3,\cdots, l$ and want to show that they hold for $l+1.$
\begin{align}
    \Prob\left(\Abs{\Norm{\vx_i^{[l+1]}(0)}_2-\sqrt{\widetilde{\mK}_{ii}^{[l+1]}}}\geq t\right) &\leq \exp\left(-cmt^2\right),\label{conclusion 1}\\
     \Prob\left(\Abs{\left<\frac{\vx_i^{[l+1]}(0)}{\sqrt{m}},\vone\right>-\widetilde{\vb}_i^{[l+1]}}\geq t\right) &\leq \exp\left(-cmt^2\right),\label{conclusion 2}
\end{align}
we recall that,
\begin{align*}
    \vx^{[l+1]}_i(0)&= \vx^{[l]}_i(0)+\frac{c_{\mathrm{res}}}{L\sqrt{m}}\sigma\left(\mW^{[l+1]}(0)\vx_i^{[l]}(0)\right),
\end{align*}
and the definition of  $\widetilde{\mK}^{[l]}_{ii}$ and $\widetilde{\vb}^{[l]}_i$
\begin{align*}
    \widetilde{\mK}^{[l+1]}_{ii}&=\widetilde{\mK}_{ii}^{[l]}+\Exp_{u\sim\fN(0,\widetilde{\mK}_{ii}^{[l]})}\left[\frac{c_{\mathrm{res}}\widetilde{\vb}_{i}^{[l]}\sigma(u) }{L} +\frac{c_{\mathrm{res}}\widetilde{\vb}_{i}^{[l]}\sigma(u) }{L}+\frac{c_{\mathrm{res}}^2\sigma(u)\sigma(u)}{L^2} \right],\nonumber\\
\widetilde{\vb}^{[l+1]}_i&=\widetilde{\vb}_{i}^{[l]}+\frac{c_{\mathrm{res}}}{L}\Exp_{u\sim \fN(0,\widetilde{\mK}_{ii}^{[l]})} \left[\sigma(u)\right], 
\end{align*}
then we have 
\begin{align*}
    \Norm{\vx^{[l+1]}_i(0)}_2^2&= \Norm{\vx^{[l]}_i(0)}_2^2+2\frac{c_{\mathrm{res}}}{L} \underbrace{\left<\frac{\vx^{[l]}_i(0)}{\sqrt{m}},\sigma\left(\mW^{[l+1]}(0)\vx_i^{[l]}(0)\right)\right>}_\textrm{I} \\
    &+\frac{c_{\mathrm{res}}^2}{L^2 } \underbrace{\frac{1}{m}  \left<\sigma\left(\mW^{[l+1]}(0)\vx_i^{[l]}(0)\right),\sigma\left(\mW^{[l+1]}(0)\vx_i^{[l]}(0)\right)\right>}_\textrm{II} ,
\end{align*}
then we need to focus on the terms I and II, note that for term I
there is a $\frac{1}{\sqrt{m}}$ scaling factor contained in ${\vx^{[l]}_i(0)}$, and $\sigma\left(\mW^{[l+1]}(0)\vx_i^{[l]}(0)\right)$ has distribution 
$$\sigma\left(\mW^{[l+1]}(0)\vx_i^{[l]}(0)\right)\sim \sigma\left(\Norm{\vx_i^{[l]}(0)}_2\vY\right),$$
with $\vY$ being a standard normal Gaussian vector, then we have
\begin{align*}
    \Exp\left[\frac{1}{\sqrt{m}}\left<{\vx^{[l]}_i(0)},\sigma\left(\mW^{[l+1]}(0)\vx_i^{[l]}(0)\right)\right>\right]=\frac{1}{\sqrt{m}}\left<{\vx^{[l]}_i(0)},\Exp\left[\sigma\left(\Norm{\vx_i^{[l]}(0)}_2\vY\right)\right]\right>,
\end{align*}
we shall focus on the inner product function $g^{[l]}(\cdot): \sR^m\to \sR,$ with $$g^{[l]}(\vY)=\frac{1}{\sqrt{m}}\left<{\vx^{[l]}_i(0)}, \sigma\left(\Norm{\vx_i^{[l]}(0)}_2\vY\right)\right>,$$ 
we have  for any $\vX_1,\vX_2 \in \sR^m,$
\begin{align*}
    \Abs{g^{[l]}(\vY_1)-g^{[l]}(\vY_2)}&\leq \frac{1}{\sqrt{m}}\Abs{\left<{\vx^{[l]}_i(0)}, \sigma\left(\Norm{\vx_i^{[l]}(0)}_2\vY_1\right)\right>-\left<{\vx^{[l]}_i(0)}, \sigma\left(\Norm{\vx_i^{[l]}(0)}_2\vY_2\right)\right>}\\
    &\leq \frac{1}{\sqrt{m}}\left<{\vx^{[l]}_i(0)}, \Norm{\vx_i^{[l]}(0)}_2\Abs{\vY_1-\vY_2}\right>\leq \frac{1}{\sqrt{m}}\Norm{\vx_i^{[l]}(0)}_2^2\Norm{\vY_1-\vY_2}_2^2
\end{align*}
based on our induction hypothesis, $\Norm{\vx_i^{[l]}(0)}_2\leq C$ with high probability,
hence $g^{[l]}(\cdot)$ is $\frac{C}{\sqrt{m}}$-Lipschitz. Apply Lemma \ref{lemma..Gaussian concentration inequality for Lipschitz function} again  
\begin{align}
    &\Prob\left(\Abs{  \frac{1}{\sqrt{m}}\left<{\vx^{[l]}_i(0)},\sigma\left(\mW^{[l+1]}(0)\vx_i^{[l]}(0)\right)\right>-\frac{1}{\sqrt{m}}\left<{\vx^{[l]}_i(0)},\Exp\left[\sigma\left(\Norm{\vx_i^{[l]}(0)}_2\vY\right)\right]\right>} \geq t\right)\nonumber\\
    &\leq  \exp(-cmt^2),\label{intermediate step 1}
\end{align}
and based on our induction hypothesis,
\begin{align}
    &\Prob\left(\Abs{ \frac{1}{\sqrt{m}}\left<{\vx^{[l]}_i(0)},\Exp\left[\sigma\left(\Norm{\vx_i^{[l]}(0)}_2\vY\right)\right]\right>-\widetilde{\vb_i}^{[l]}\Exp\left[\sigma\left(\Norm{\vx_i^{[l]}(0)}_2\vY\right)\right]} \geq t\right)\nonumber\\
    &\leq  \exp(-cmt^2),\label{intermediate step 2}
\end{align}
from Lemma \ref{lemma....one expectation constant}
\begin{align*}
 \Abs{   \Exp\left[\sigma\left(\Norm{\vx_i^{[l]}(0)}_2\vY\right)\right]-\Exp\left[\sigma\left(\sqrt{\widetilde{\mK}_{ii}^{[l]}}\vY\right)\right]}\leq C\Abs{\Norm{\vx_i^{[l]}(0)}_2-\sqrt{\widetilde{\mK}_{ii}^{[l]}}},
\end{align*}
altogether we have 
\begin{align}
    &\Prob\left(\Abs{ \widetilde{\vb_i}^{[l]}\Exp\left[\sigma\left(\Norm{\vx_i^{[l]}(0)}_2\vY\right)\right]-\widetilde{\vb_i}^{[l]}\Exp\left[\sigma\left(\sqrt{\widetilde{\mK}_{ii}^{[l]}}\vY\right)\right]} \geq t\right)\leq  \exp(-cmt^2),\label{intermediate step 3}
\end{align}
combine \eqref{intermediate step 1}, \eqref{intermediate step 2} and  \eqref{intermediate step 3} 
\begin{align}
    &\Prob\left(\Abs{  \frac{1}{\sqrt{m}}\left<{\vx^{[l]}_i(0)},\sigma\left(\mW^{[l+1]}(0)\vx_i^{[l]}(0)\right)\right>-\widetilde{\vb_i}^{[l]}\Exp\left[\sigma\left(\sqrt{\widetilde{\mK}_{ii}^{[l]}}\vY\right)\right]} \geq t\right)\leq  \exp(-cmt^2).\label{final part .....1}
\end{align}
Finally for term II 
\begin{equation*}
     \Exp\left[\frac{1}{m}  \left<\sigma\left(\mW^{[l+1]}(0)\vx_i^{[l]}(0)\right),\sigma\left(\mW^{[l+1]}(0)\vx_i^{[l]}(0)\right)\right>\right] ={\Exp_{x\sim \fN(0,1)}\left[\sigma\left(\Norm{\vx_i^{[l]}(0)}_2 x\right)^2\right]},
\end{equation*}
since $\left(\mW^{[l+1]}(0)\vx_i^{[l]}(0)\right)$ are i.i.d standard Gaussian variables, and $\sigma(\cdot)$ is $1$-Lipschitz, then $\left(\mW^{[l+1]}(0)\vx_i^{[l]}(0)\right)$ are sub-exponential variables, then we have 
\begin{align}
   &\Prob\left(\Abs{\frac{1}{m}  \left<\sigma\left(\mW^{[l+1]}(0)\vx_i^{[l]}(0)\right),\sigma\left(\mW^{[l+1]}(0)\vx_i^{[l]}(0)\right)\right>-{\Exp_{x\sim \fN(0,1)}\left[\sigma\left(\Norm{\vx_i^{[l]}(0)}_2 x\right)^2\right]} }\geq t\right)\nonumber\\
   &\leq \exp\left(-c mt^2\right),\label{intermediate.....1}
\end{align}
and apply Lemma \ref{lemma....two expectation constant}
\begin{align*}
    \Abs{{\Exp_{x\sim \fN(0,1)}\left[\sigma\left(\Norm{\vx_i^{[l]}(0)}_2 x\right)^2\right]}-\Exp_{x\sim \fN(0,1)}\left[\sigma\left(\sqrt{\widetilde{\mK_{ii}}^{[l]}} x\right)^2\right]}\leq C\Abs{\Norm{\vx_i^{[l]}(0)}_2-\sqrt{\widetilde{\mK_{ii}}^{[l]}}},
\end{align*}
then based on our induction hypothesis
\begin{align}
   \Prob\left(  \Abs{{\Exp_{x\sim \fN(0,1)}\left[\sigma\left(\Norm{\vx_i^{[l]}(0)}_2 x\right)^2\right]}-{\Exp_{x\sim \fN(0,1)}\left[\sigma\left(\sqrt{\widetilde{\mK_{ii}}^{[l]}} x\right)^2\right]}}\geq t\right)\leq \exp\left(-cmt^2\right),\label{intermediate.....2}
\end{align}
combining \eqref{intermediate.....1} and \eqref{intermediate.....2}
\begin{align}
    & \Prob\left(  \Abs{\frac{1}{m}  \left<\sigma\left(\mW^{[l+1]}(0)\vx_i^{[l]}(0)\right),\sigma\left(\mW^{[l+1]}(0)\vx_i^{[l]}(0)\right)\right>-{\Exp_{x\sim \fN(0,1)}\left[\sigma\left(\sqrt{\widetilde{\mK_{ii}}^{[l]}} x\right)^2\right]}}\geq t\right)\nonumber\\
    &\leq \exp\left(-cmt^2\right),\label{final part....2}
\end{align}
since  we have 
\begin{align*}
    \Norm{\vx^{[l+1]}_i(0)}_2^2&= \Norm{\vx^{[l]}_i(0)}_2^2+2\frac{c_{\mathrm{res}}}{L} \underbrace{\left<\frac{\vx^{[l]}_i(0)}{\sqrt{m}},\sigma\left(\mW^{[l+1]}(0)\vx_i^{[l]}(0)\right)\right>}_\textrm{I} \\
    &+\frac{c_{\mathrm{res}}^2}{L^2 } \underbrace{\frac{1}{m}  \left<\sigma\left(\mW^{[l+1]}(0)\vx_i^{[l]}(0)\right),\sigma\left(\mW^{[l+1]}(0)\vx_i^{[l]}(0)\right)\right>}_\textrm{II} ,
\end{align*}
then 
\begin{align}
    &\Prob\left(  \Abs{ \Norm{\vx^{[l+1]}_i(0)}^2_2-{\widetilde{\mK_{ii}}^{[l+1]}}}\geq t\left(1+\frac{c_{\mathrm{res}}}{L}\right)^2\right)\nonumber \\
    &\leq\Prob\left(  \Abs{ \Norm{\vx^{[l]}_i(0)}^2_2-{\widetilde{\mK_{ii}}^{[l]}}}\geq t\right)+ \Prob\left(  \Abs{\mathrm{II}}\geq t\right)+\Prob\left(  \Abs{\mathrm{III}}\geq t\right)\nonumber\\
&\leq \exp\left(-cmt^2\right),
\end{align}
we shall see that thanks to the $\frac{c_{\mathrm{res}}}{L}$ structure, with high probability the difference of $\Abs{\Norm{\vx^{[l]}_i(0)}^2_2-{\widetilde{\mK_{ii}}^{[l]}}}$ does not explode exponentially layer by layer.

For $\widetilde{\vb}_i^{[l+1]}$, apply Lemma \ref{lemma..Gaussian concentration inequality for Lipschitz function},
\begin{align}
    \Prob\left(\Abs{\left<\frac{\sigma\left(\mW^{[l+1]}(0)\vx_i^{[l]}(0)\right)}{{m}},\vone\right>-\Exp\left[ \sigma\left(\Norm{\vx_i^{[l]}(0)}_2 \vY\right)\right]}\geq t\right) &\leq \exp\left(-cmt^2\right),\label{intermediate........1}
\end{align}
and apply Lemma \ref{lemma....one expectation constant} 
\begin{align}
    \Prob\left(\Abs{\Exp\left[ \sigma\left(\Norm{\vx_i^{[l]}(0)}_2 \vY\right)\right]-\Exp\left[ \sigma\left(\sqrt{\widetilde{\mK_{ii}}^{[l]}}\vY\right)\right]}\geq t\right) &\leq \exp\left(-cmt^2\right),\label{intermediate........2}
\end{align}
combine \eqref{intermediate........1} and \eqref{intermediate........2},
\begin{align}
     \Prob\left(\Abs{\left<\frac{\sigma\left(\mW^{[l+1]}(0)\vx_i^{[l]}(0)\right)}{{m}},\vone\right>-\Exp\left[ \sigma\left(\sqrt{\widetilde{\mK_{ii}}^{[l]}}\vY\right)\right]}\geq t\right) &\leq \exp\left(-cmt^2\right), \label{conclusion ..........3}
\end{align}
then
\begin{align}
 &\Prob\left(\Abs{\left<\frac{\vx_i^{[l+1]}(0)}{\sqrt{m}},\vone\right>-\widetilde{\vb}_i^{[l+1]}}\geq t\left(1+\frac{c_{\mathrm{res}}}{L}\right)\right) \nonumber\\
 &\leq\Prob\left(\Abs{\left<\frac{\vx_i^{[l]}(0)}{\sqrt{m}},\vone\right>-\widetilde{\vb}_i^{[l]}}\geq t\right)+\Prob\left(\Abs{\left<\frac{\sigma\left(\mW^{[l+1]}(0)\vx_i^{[l]}(0)\right)}{{m}},\vone\right>-\Exp\left[ \sigma\left(\sqrt{\widetilde{\mK_{ii}}^{[l]}}\vY\right)\right]}\geq t\right)\nonumber\\
 &\leq \exp\left(-cmt^2\right),\label{conclusion....final}
\end{align}
We shall see  again that thanks to the $\frac{c_{\mathrm{res}}}{L}$ structure, with high probability the difference of $\Abs{\left<\frac{\vx_i^{[l]}(0)}{\sqrt{m}},\vone\right>-\widetilde{\vb}_i^{[l]}}$ only has slight increment w.r.t each layer $l.$
\end{proof}
 Our next Proposition is on the least eigenvalue of the randomly initialized Gram matrix $\mG^{[1]}(0).$
\begin{prop}\label{proposition......eigenvalue of order 1}
Under Assumption \ref{Assump...Assumption on activation functions} and \ref{Assump... ont he imput of the  samples}, if $m=\Omega\left( \left(\frac{n}{\lambda_0}\right)^{2+\eps}\right)$,  then with high probability
\begin{equation}
    \lambda_{\min}\left(\mG^{[1]}(0)\right)\geq \frac{3 \lambda_0}{4},
\end{equation}
where $\lambda_0$ has been defined in \eqref{ineq..matrix comparision hierachy at 0}.
\end{prop}
\begin{proof}
We have that 
\begin{align*}
   \mG_{ij}^{[1]}(0)& =\left< \vx_i^{[1]}(0), \vx_j^{[1]}(0)\right>\nonumber\\  \widetilde{\mK}_{ij}^{[0]}&= \left<\vx_i,\vx_j\right>,\nonumber\\
 \widetilde{\mK}_{ij}^{[1]}&= c_{\sigma}\Exp_{(u,v)^{\T}\sim \fN\left(\vzero,  \begin{pmatrix}\widetilde{\mK}_{ii}^{[0]}&\widetilde{\mK}_{ij}^{[0]}\\
\widetilde{\mK}_{ji}^{[0]}&\widetilde{\mK}_{jj}^{[0]}\end{pmatrix} \right)}\left[ \sigma(u)\sigma(v)\right],
\end{align*}
 now we need to apply Lemma \ref{lemma..Gaussian concentration inequality for Lipschitz function} again, except that this time we are going to apply it to the inner product function $h^{[1]}(\cdot):\sR^{2m}\to \sR$, with $$h^{[1]}(\vZ)=\frac{c_{\sigma}}{m}\left<\sigma(\vX),\sigma(\rho\vX+\sqrt{1-\rho^2}\vY) \right>,$$ 
where $-1\leq \rho\leq 1.$ 

Specifically with  $\vZ^{\T}=(\vX^{\T},\vY^{\T})$, we have  for any $\vZ_1,\vZ_2 \in \sR^m,$
\begin{align*}
    \Abs{h^{[1]}(\vZ_1)-h^{[1]}(\vZ_2)}&\leq \sqrt{\frac{c_{\sigma}}{m}}\Norm{\sigma(\rho\vX_1+\sqrt{1-\rho^2}\vY_1)}_2 \sqrt{\frac{c_{\sigma}}{m}}\Norm{\vX_1-\vX_2}_2\\
    &+\sqrt{\frac{c_{\sigma}}{m}}\Norm{\sigma(\vX_2)}_2 \sqrt{\frac{c_{\sigma}}{m}}\left(\Abs{\rho}\Norm{\vX_1-\vX_2}_2+\sqrt{1-\rho^2}\Norm{\vY_1-\vY_2}_2\right),
\end{align*}
combined with Proposition \ref{proposition..................Initialization Norms layer 1},  with probability $1-\exp(-cm)$,
\begin{equation*}
    \sqrt{\frac{c_{\sigma}}{m}}\Norm{\sigma(\rho\vX_1+\sqrt{1-\rho^2}\vY_1)}_2,\sqrt{\frac{c_{\sigma}}{m}}\Norm{\sigma(\vX_2)}_2\leq 2,
\end{equation*}
so we have 
\begin{align*}
    \Abs{h^{[1]}(\vZ_1)-h^{[1]}(\vZ_2)}\leq 8\sqrt{\frac{c_{\sigma}}{m}} \Norm{\vZ_1-\vZ_2}_2.
\end{align*}
hence $h^{[1]}\left(\vZ\right)$ is $8\sqrt{\frac{c_{\sigma}}{m}}$-Lipschitz, then we shall set $\rho=\widetilde{\mK}_{ij}^{[0]},$
\begin{equation}
\Prob\left(\Abs{ \mG_{ij}^{[1]}(0) -\widetilde{\mK}_{ij}^{[1]}} \geq  t\right)\leq  \exp(-cmt^2),
\end{equation}
note that  we have
\begin{align*}
    \Norm {\mG^{[1]}(0) -\widetilde{\mK}^{[1]}}_{2\to 2} \leq  \Norm {\mG^{[1]}(0) -\widetilde{\mK}^{[1]}}_{\mathrm{F}} \leq n \Norm {\mG^{[1]}(0) -\widetilde{\mK}^{[1]}}_{\infty},
\end{align*}
based on Proposition \ref{proposition..covarince type}, $\lambda_{\min}(\widetilde{\mK}^{[1]})\geq \lambda_0$, then if we choose $t=\frac{\lambda_0}{4n}$ and with a union $m^2$ such events, we have  with probability $1-m^2\exp\left( -cm\lambda_0^2/n^2\right)$
\begin{equation}
        \Norm {\mG^{[1]}(0) -\widetilde{\mK}^{[1]}}_{2\to 2} \leq  \frac{\lambda_0}{4},
\end{equation}
hence if $m=\Omega\left( \left(\frac{n}{\lambda_0}\right)^{2+\eps}\right)$,  we have with probability $1-\exp(-m^{\eps})$
\begin{equation}
\lambda_{\min}(\mG^{[1]}(0))\geq    \lambda_{\min}(\widetilde{\mK}^{[1]})-\Norm {\mG^{[1]}(0) -\widetilde{\mK}^{[1]}}_{2\to 2}\geq \frac{3\lambda_0}{4}.
\end{equation}
\end{proof}

Our next Proposition  on the stability of the randomly initialized Gram matrix $\mG^{[l]}(0)$ for $2\leq l\leq L+1. $
\begin{prop}\label{proposition......eigenvalue of order higher}
Under Assumption \ref{Assump...Assumption on activation functions} and \ref{Assump... ont he imput of the  samples}, if $m=\Omega\left( \left(\frac{n}{\lambda_0}\right)^{2+\eps}\right)$,  then with high probability
\begin{equation}
    \lambda_{\min}\left(\mG^{[l]}(0)\right)\geq \frac{3 \lambda_0}{4}, \ 2\leq l\leq L+1
\end{equation}
where $\lambda_0$ has been defined in \eqref{ineq..matrix comparision hierachy at 0}.
\end{prop}
 \begin{proof}
 
For $l=2,$ we shall make estimate on the norm, $\Norm{\mG^{[2]}(0) -\widetilde{\mK}^{[2]}}_{\infty},$ since by definition
\begin{align}
 \mG_{ij}^{[2]}(0)& =\left< \vx_i^{[2]}(0), \vx_j^{[2]}(0)\right>=\mG_{ij}^{[1]}(0)+\frac{c_{\mathrm{res}}}{L}\underbrace{\frac{1}{\sqrt{m}}\left< \vx_i^{[1]}(0),\sigma\left( \mW^{[2]}(0)\vx_j^{[1]}(0)\right)\right>}_\textrm{I}\nonumber\\
 &+\frac{c_{\mathrm{res}}}{L}\underbrace{\frac{1}{\sqrt{m}}\left< \vx_j^{[1]}(0),\sigma\left( \mW^{[2]}(0)\vx_i^{[1]}(0)\right)\right>}_\textrm{II}\nonumber\\
 &+\frac{c_{\mathrm{res}}^2}{L^2}\underbrace{\frac{1}{{m}}\left< \sigma\left( \mW^{[2]}(0)\vx_i^{[1]}(0)\right),\sigma\left( \mW^{[2]}(0)\vx_j^{[1]}(0)\right)\right>}_\textrm{III}\nonumber\\ 
 \widetilde{\mK}_{ij}^{[1]}&= c_{\sigma}\Exp_{(u,v)^{\T}\sim \fN\left(\vzero,  \begin{pmatrix}\widetilde{\mK}_{ii}^{[0]}&\widetilde{\mK}_{ij}^{[0]}\nonumber\\
\widetilde{\mK}_{ji}^{[0]}&\widetilde{\mK}_{jj}^{[0]}\end{pmatrix} \right)}\left[ \sigma(u)\sigma(v)\right],\\
 \widetilde{\vb}_{i}^{[1]}&=\sqrt{c_{\sigma}}\Exp_{u\sim \fN(0,\widetilde{\mK}_{ii}^{[0]})} \left[\sigma(u)\right], \nonumber\\
\widetilde{\mA}^{[2]}_{ij}&=\begin{pmatrix}\widetilde{\mK}_{ii}^{[1]}&\widetilde{\mK}_{ij}^{[1]}\\
\widetilde{\mK}_{ji}^{[1]}&\widetilde{\mK}_{jj}^{[1]}\end{pmatrix},\nonumber\\
 \widetilde{\mK}_{ij}^{[2]}&=\widetilde{\mK}_{ij}^{[1]}+\Exp_{(u,v)^{\T}\sim \fN\left(\vzero, \widetilde{\mA}^{[2]}_{ij}\right)}\left[\frac{c_{\mathrm{res}} }{L}\underbrace{\widetilde{\vb}_{i}^{[1]}\sigma(v)}_\textrm{I'} +\frac{c_{\mathrm{res}} }{L}\underbrace{\widetilde{\vb}_{j}^{[1]}\sigma(u)}_\textrm{II'}+\frac{c_{\mathrm{res}}^2}{L^2}\underbrace{\sigma(u)\sigma(v)}_\textrm{III'} \right]\nonumber.
\end{align}
 
 We need to tackle the difference between I and I', in order for that, we need to write the difference into
\begin{align*}
   &\Abs{ \frac{1}{\sqrt{m}}\left< \vx_i^{[1]}(0),\sigma\left( \mW^{[2]}(0)\vx_j^{[1]}(0)\right)\right>-\widetilde{\vb}_{i}^{[1]}\sigma(v)}\\
   &\leq \Abs{\frac{1}{\sqrt{m}}\left< \vx_i^{[1]}(0),\sigma\left( \mW^{[2]}(0)\vx_j^{[1]}(0)\right)\right>-\frac{1}{\sqrt{m}}\left< \vx_i^{[1]}(0),\Exp\left[\sigma\left( \Norm{\vx_j^{[1]}(0)}_2\vY\right)\right]\right>}\\
   &+\Abs{\frac{1}{\sqrt{m}}\left< \vx_i^{[1]}(0),\Exp\left[\sigma\left( \Norm{\vx_j^{[1]}(0)}_2\vY\right)\right]\right>-\widetilde{\vb}^{[1]}_i\Exp\left[\sigma\left( \Norm{\vx_j^{[1]}(0)}_2\vY\right)\right]}\\
   &+\Abs{\widetilde{\vb}^{[1]}_i\Exp\left[\sigma\left( \Norm{\vx_j^{[1]}(0)}_2\vY\right)\right]-\widetilde{\vb}_{i}^{[1]}\Exp\left[\sigma\left( \sqrt{\widetilde{\mK_{jj}}^{[1]}}\vY\right)\right]},
\end{align*}
similar to the proof in Proposition \ref{proposition..................Initialization Norms layer 1}
with $\vY$ being a standard normal Gaussian vector  
\begin{align}
    \Prob\left(\Abs{\mathrm{I-I'}}\geq t\right)\leq \exp\left(-cmt^2\right),\label{intemediate ..............  1}
\end{align}
similarly  
\begin{align}
    \Prob\left(\Abs{\mathrm{II-II'}}\geq t\right)\leq \exp\left(-cmt^2\right),\label{intemediate ..............  2}
\end{align}
 
 for the difference between III and III',  we need to define another inner product function $h^{[2]}(\cdot):\sR^{2m}\to \sR,$ being 
\begin{align*}
    h^{[2]}(\vZ)=\frac{1}{ m}\left<\sigma\left(C_2\vX\right),\sigma\left(D_2\left(\rho\vX+\sqrt{1-\rho^2}\vY\right)\right) \right>,
\end{align*}
with $C_2,D_2>0$ being constants and $-1\leq \rho\leq 1.$ 

Note that the form $\vZ^{\T}=(\vX^{\T},\vY^{\T})$, similar to $h^{[1]}(\cdot)$ defined in the proof of Proposition \ref{proposition......eigenvalue of order 1}, $h^{[2]}(\cdot)$ is $\frac{C}{\sqrt{m}}$-Lipschitz, then we have 
\begin{equation*}
    \Prob\left(\Abs{{h}^{[2]}(\vZ) - \Exp {h}^{[2]}(\vZ)} \geq t\right)\leq  \exp(-cmt^2),
\end{equation*}
 hence we have
\begin{align}
   &\Prob\left(\Abs{\frac{1}{m}\left< \sigma\left( \mW^{[2]}(0)\vx_i^{[1]}(0)\right),\sigma\left( \mW^{[2]}(0)\vx_j^{[1]}(0)\right)\right> - \Exp_{(u,v)^{\T}\sim \fN\left(\vzero, {\mA}^{[2]}_{ij}\right)}\left[\sigma(u) \sigma(v) \right] }\geq t\right)\nonumber\\
   &\leq  \exp(-cmt^2),\label{prob4}
\end{align}
with
\begin{align*}
{\mA}^{[2]}_{ij}&=
\begin{pmatrix}\left<\vx_i^{[1]}(0),\vx_i^{[1]}(0)\right> &\left<\vx_i^{[1]}(0),\vx_j^{[1]}(0)\right>\\
\left<\vx_j^{[1]}(0),\vx_i^{[1]}(0)\right>&\left<\vx_j^{[1]}(0),\vx_j^{[1]}(0)\right>\end{pmatrix},
\end{align*}
 
combined with Lemma \ref{lemma....two expectation constant} and Proposition \ref{proposition..................Initialization Norms layer 1}  
 \begin{align}
     &\Prob\left(\Abs{ \Exp_{(u,v)^{\T}\sim \fN\left(\vzero, {\mA}^{[2]}_{ij}\right)}\left[\sigma(u) \sigma(v) \right]-\Exp_{(u,v)^{\T}\sim \fN\left(\vzero, \widetilde{\mA}^{[2]}_{ij}\right)}\left[\sigma(u) \sigma(v) \right] }\geq t\right)\leq  \exp(-cmt^2),\label{prob5}
\end{align}
combine \eqref{prob4} and \eqref{prob5}  
\begin{align}
    \Prob\left(\Abs{\mathrm{III-III'}}\geq t\right)\leq \exp\left(-cmt^2\right),\label{intemediate ..............  3}
\end{align}
then we have that
\begin{align}
      &\Prob\left(\Abs{\mG_{ij}^{[2]}(0)-\widetilde{\mK}_{ij}^{[2]}}\geq t\left(1+\frac{c_{\mathrm{res}}}{L}\right)^2\right)\nonumber\\
      &\leq  \Prob\left(\Abs{\mG_{ij}^{[1]}(0)-\widetilde{\mK}_{ij}^{[1]}}\geq t\right)+ \Prob\left(\Abs{\mathrm{I-I'}}\geq t\right)+\Prob\left(\Abs{\mathrm{II-II'}}\geq t\right)+\Prob\left(\Abs{\mathrm{III-III'}}\geq t\right)\nonumber\\
     &\leq\exp\left(-cmt^2\right),\label{conclusion........final....1}
\end{align}
hence inductively,   for $2\leq l \leq L$
\begin{align}
      &\Prob\left(\Abs{\mG_{ij}^{[l]}(0)-\widetilde{\mK}_{ij}^{[l]}}\geq t\left(1+\frac{c_{\mathrm{res}}}{L}\right)^{2l-2}\right)\leq\exp\left(-cmt^2\right),\label{conclusion........final....2}
\end{align}
moreover, 
\begin{align}
      &\Prob\left(\Abs{\mG_{ij}^{[L+1]}(0)-{\mK}_{ij}^{[L+1]}}\geq t\left(1+\frac{c_{\mathrm{res}}}{L}\right)^{2L}\right)\leq\exp\left(-cmt^2\right),\label{conclusion........final......2}
\end{align}
note that  we have
\begin{align*}
    \Norm {\mG^{[L+1]}(0) -{\mK}^{[L+1]}}_{2\to 2} \leq  \Norm {\mG^{[L+1]}(0) -{\mK}^{[L+1]}}_{\mathrm{F}} \leq n \Norm {\mG^{[L+1]}(0) -{\mK}^{[L+1]}}_{\infty},
\end{align*}
based on Proposition \ref{proposition..covarince type}, $\lambda_{\min}({\mK}^{[L+1]})> \lambda_0$, then if we choose $t=\frac{\lambda_0}{4n\exp\left(2c_{\mathrm{res}}\right)},$  for~$2\leq l \leq L,$ with probability $1-\exp\left( -cm\lambda_0^2/n^2\right),$
\begin{equation}
        \Norm {\mG^{[l]}(0) -\widetilde{\mK}^{[l]}}_{2\to 2} \leq  \frac{\lambda_0}{4},
\end{equation}
hence if $m=\Omega\left( \left(\frac{n}{\lambda_0}\right)^{2+\eps}\right)$,  we have with probability $1-\exp(-m^{\eps})$
\begin{equation}
\lambda_{\min}(\mG^{[l]}(0))\geq    \lambda_{\min}(\widetilde{\mK}^{[l]})-\Norm {\mG^{[l]}(0) -\widetilde{\mK}^{[l]}}_{2\to 2}>\frac{3\lambda_0}{4}.
\end{equation}
In particular, we have that 
with probability $1-\exp\left( -cm\lambda_0^2/n^2\right),$
\begin{equation}
        \Norm {\mG^{[L+1]}(0) -{\mK}^{[L+1]}}_{2\to 2} \leq  \frac{\lambda_0}{4},
\end{equation}
hence if $m=\Omega\left( \left(\frac{n}{\lambda_0}\right)^{2+\eps}\right)$,  we have with probability $1-\exp(-m^{\eps})$
\begin{equation}
\lambda_{\min}(\mG^{[L+1]}(0))\geq    \lambda_{\min}({\mK}^{[L+1]})-\Norm {\mG^{[L+1]}(0) -{\mK}^{[L+1]}}_{2\to 2}>\frac{3\lambda_0}{4}.
\end{equation}
\end{proof}
\section{Proof of Theorem \ref{theorem...vary at 1/m}  and Corollary \ref{corollary....for thm}}\label{appendix...subsection for proof of thm and cor}
We shall begin with the detailed  proof of  Theorem \ref{theorem...vary at 1/m}.
   \begin{proof}[Proof of Theorem  \ref{theorem...vary at 1/m}]
   We are only going to  use $\fG_t^{[L+1]}\left(\cdot\right)$ instead of the whole NTK $\fK_t^{(2)}\left(\cdot\right),$ thanks to the simple structure of $\fG_t^{[L+1]}\left(\cdot\right),$ we are able to bring about a more concrete proof.
   
   Since there exists a $\frac{1}{L^2}$ scaling in some kernels, we use $C(r,L)$ to denote the `effective terms' in each kernel and we are going to  show that  \eqref{eq for thm...uniform estimate....order3} holds.  Firstly, we need to denote $\fG_t^{[L+1]}\left(\cdot\right)$ by $\fG_t^{[2]}\left(\cdot\right),$ i.e.,
$$\fG_t^{(2)}(\vx_{\alpha_{1}},\vx_{\alpha_{2}}):=\fG_t^{[L+1]}(\vx_{\alpha_{1}},\vx_{\alpha_{2}})=\left<\vx_{\alpha_1}^{[L]},\vx_{\alpha_2}^{[L]}\right>,$$
then it's natural for us to get that $C(2,L)=\fO(1),$ since there is only one term.

Secondly, by the replacement rule, all the possible terms generated from $\fG_t^{(2)}(\cdot)$ are
\begin{align*}
&\fG_t^{(2)}(\vx_{\alpha_{1}},\vx_{\alpha_{2}})=\left<\vx_{\alpha_1}^{[L]},\vx_{\alpha_2}^{[L]}\right>\to \fG_t^{(3)}(\vx_{\alpha_{1}},\vx_{\alpha_{2}}, \vx_{\beta})\\
 &\fG_t^{(3)}(\vx_{\alpha_{1}},\vx_{\alpha_{2}}, \vx_{\beta})={\frac{c_\sigma}{m}}\underbrace{\left< \mathrm{diag}\left(\mE_{t,\alpha_1}^{[2:L]} \vsigma^{(1)}_{[1]}(\vx_{\alpha_1})\vsigma^{(1)}_{[1]}(\vx_{\beta})\left(\mE_{t,\beta}^{[2:L]}\right)^{\T}\va_t \right)\vone,\vx_{\alpha_2}^{[L]}\right> \left<\vx_{\alpha_1},\vx_\beta  \right>}_{\textrm{I}}\\
&+\sum_{k=2}^{L} \frac{c_{\mathrm{res}}^2}{L^2 {m}}\underbrace{\left< \mathrm{diag}\left(\mE_{t,\alpha_1}^{[(k+1):L]}\vsigma^{(1)}_{[k]}(\vx_{\alpha_1})\vsigma^{(1)}_{[k]}(\vx_{\beta})\left(\mE_{t,\beta}^{[(k+1):L]}\right)^{\T}\va_t  \right) \vone,\vx_{\alpha_2}^{[L]}\right> \left<\vx_{\alpha_1}^{[k-1]},\vx_\beta^{[k-1]}  \right>}_{\mathrm{II}}\\
&+{\frac{c_\sigma}{m}}\left< \mathrm{diag}\left(\mE_{t,\alpha_2}^{[2:L]} \vsigma^{(1)}_{[1]}(\vx_{\alpha_2})\vsigma^{(1)}_{[1]}(\vx_{\beta})\left(\mE_{t,\beta}^{[2:L]}\right)^{\T}\va_t \right)\vone,\vx_{\alpha_1}^{[L]}\right> \left<\vx_{\alpha_2},\vx_\beta  \right>\\
&+\sum_{k=2}^{L} \frac{c_{\mathrm{res}}^2}{L^2 {m}}\left< \mathrm{diag}\left(\mE_{t,\alpha_2}^{[(k+1):L]}\vsigma^{(1)}_{[k]}(\vx_{\alpha_2})\vsigma^{(1)}_{[k]}(\vx_{\beta})\left(\mE_{t,\beta}^{[(k+1):L]}\right)^{\T}\va_t  \right) \vone,\vx_{\alpha_1}^{[L]}\right> \left<\vx_{\alpha_2}^{[k-1]},\vx_\beta^{[k-1]}  \right>.
\end{align*}
Thanks to the $\frac{1}{L^2}$ scaling, we obtain that  
$$C(3,L)=\fO\left( 2\left(1+\frac{L-1}{L^2}\right)\right)=\fO\left(1+\frac{1}{L}\right)$$
Finally for $\fG_t^{(4)}(\cdot),$ by symmetry, we are only going to analyze  terms I  and II.
Since there are at most $(2L+2)$ symbols in term I to be replaced, and by the  replacement rules, each replacement will bring about up to  $(L+1)$ many terms. For term II, for each summand, there are also at most $(2L+2)$ symbols  to be replaced. Since there are $L-1$ summands in II, and  each replacement will bring about up to  $(L+1)$ many terms, then we have 
$$C(4,L)=\fO\left( 2\left((2L+2)(L+1)+ \frac{1}{L^2}(L-1)(2L+2)(L+1) \right)\right)=\fO\left(L^2\right).$$
 Using \eqref{eq for thm...uniform estimate....higher rder2} in Theorem \ref{thm......infinite family}, it holds that for time $0\leq t\leq \sqrt{m}/{\LogLn}^{C'}$
\begin{align*}
    \Norm{\fG_t^{(4)}(\cdot)}_{\infty}\leq C(4,L) \frac{{\LogLn}^C}{m},
\end{align*}
   based on \eqref{eq for thm...neural tangent kernel .... higher order 2}
\begin{align*}
    \Abs{\partial_t\fG_t^{(3)}\left(\vx_{\alpha_1},\vx_{\alpha_2},\vx_{\alpha_3}\right)}&\leq \sup_{1\leq \beta \leq n} \Abs{\fG_t^{(4)}\left(\vx_{\alpha_1},\vx_{\alpha_2},\vx_{\alpha_3},\vx_{\beta}\right)}\sqrt{\frac{\sum_{\beta=1}^n \Abs{f_{\beta}(t)-y_{\beta}}^2}{n}}\\
    &\leq   \Norm{\fG_t^{(4)}(\cdot)}_{\infty} \sqrt{ R_S(\vtheta_0)}\\
    &\leq C(4,L) \frac{{\LogLn}^C}{m},
\end{align*}
then    for any $1\leq \alpha_1, \alpha_2, \alpha_3\leq n,$ with time $0\leq t\leq \sqrt{m}/{\LogLn}^{C'}$
\begin{align*}
    \Abs{\fG_t^{(3)}\left(\vx_{\alpha_1},\vx_{\alpha_2},\vx_{\alpha_3}\right)}&\leq  \Abs{\fG_0^{(3)}\left(\vx_{\alpha_1},\vx_{\alpha_2},\vx_{\alpha_3}\right)}+t\Abs{\partial_t\fG_t^{(3)}\left(\vx_{\alpha_1},\vx_{\alpha_2},\vx_{\alpha_3}\right)}\\
    &\leq  \Norm{\fG_0^{(3)}\left(\cdot\right)}_{\infty}+t \ C(4,L)\frac{{\LogLn}^C}{m}. 
\end{align*}
Finally, we need to make estimate on $\Norm{\fG_0^{(3)}\left(\cdot\right)}_{\infty}.$ We shall take advantage of the $\mathrm{diag}(\cdot)\vone$ structure and     rewrite $\fG_t^{(3)}\left(\cdot\right)$ into
\begin{align*}
 &\fG_t^{(3)}(\vx_{\alpha_{1}},\vx_{\alpha_{2}}, \vx_{\beta})={\frac{c_\sigma}{m}}\left< \mE_{t,\alpha_1}^{[2:L]} \vsigma^{(1)}_{[1]}(\vx_{\alpha_1})\vsigma^{(1)}_{[1]}(\vx_{\beta})\left(\mE_{t,\beta}^{[2:L]}\right)^{\T}\va_t ,\vx_{\alpha_2}^{[L]}\right> \left<\vx_{\alpha_1},\vx_\beta  \right>\\
&+\sum_{k=2}^{L} \frac{c_{\mathrm{res}}^2}{L^2 {m}}\left< \mE_{t,\alpha_1}^{[(k+1):L]}\vsigma^{(1)}_{[k]}(\vx_{\alpha_1})\vsigma^{(1)}_{[k]}(\vx_{\beta})\left(\mE_{t,\beta}^{[(k+1):L]}\right)^{\T}\va_t  ,\vx_{\alpha_2}^{[L]}\right> \left<\vx_{\alpha_1}^{[k-1]},\vx_\beta^{[k-1]}  \right>\\
&+{\frac{c_\sigma}{m}}\left< \mE_{t,\alpha_2}^{[2:L]} \vsigma^{(1)}_{[1]}(\vx_{\alpha_2})\vsigma^{(1)}_{[1]}(\vx_{\beta})\left(\mE_{t,\beta}^{[2:L]}\right)^{\T}\va_t ,\vx_{\alpha_1}^{[L]}\right> \left<\vx_{\alpha_2},\vx_\beta  \right>\\
&+\sum_{k=2}^{L} \frac{c_{\mathrm{res}}^2}{L^2 {m}}\left< \mE_{t,\alpha_2}^{[(k+1):L]}\vsigma^{(1)}_{[k]}(\vx_{\alpha_2})\vsigma^{(1)}_{[k]}(\vx_{\beta})\left(\mE_{t,\beta}^{[(k+1):L]}\right)^{\T}\va_t  ,\vx_{\alpha_1}^{[L]}\right> \left<\vx_{\alpha_2}^{[k-1]},\vx_\beta^{[k-1]}  \right>,
\end{align*}
then at time $t=0,$ wlog, each term in $\fG_0^{(3)}\left(\cdot\right)$ is of the form
\begin{align}
    \frac{c}{m}\left<\mB \va_0, \vx_{\alpha_1}^{[L]} \right>\left<\vx_{\alpha_2}^{[l]},\vx_\beta^{[l]}  \right>, \ 0\leq l \leq L-1, \label{rewrite...1}
\end{align}
where $\mB$ is some specific matrix that changes from term to term, then we can rewrite the inner product into:
\begin{align}
    \frac{c}{m}\left<\va_0, \mB^{\T}\vx_{\alpha_1}^{[L]} \right>\left<\vx_{\alpha_2}^{[l]},\vx_\beta^{[l]}  \right>,  \label{rewrite...2}
\end{align}
we shall focus on the term 
\begin{align}
    \left<\va_0, \mB^{\T}\vx_{\alpha_1}^{[L]} \right>\left<\vx_{\alpha_2}^{[l]},\vx_\beta^{[l]}  \right>, \label{focus on}
\end{align}
note that each entry of $\va_0$ is i.i.d $\fN(0,1),$ also based on Proposition \ref{proposition..A priori spectral property random matrix and a t} and \ref{proposition.... on the output of layes}, with high probability w.r.t random initialization, for time $0\leq t \leq {\LogLn}^{C'}$
\begin{align*}
    \Norm{\mB^{\T}}_{2\to 2}, \vx_{\alpha_1}^{[L]},\vx_{\alpha_2}^{[l]},\vx_\beta^{[l]} \leq c,
\end{align*}
then after taking conditional expectation except for the random variable $\va_0$  
\begin{align}
    \left<\va_0, \mB^{\T}\vx_{\alpha_1}^{[L]} \right>\left<\vx_{\alpha_2}^{[l]},\vx_\beta^{[l]}  \right>\sim\fN\left(0,c\right), \label{distrubuted as }
\end{align}
apply Lemma \ref{lemma..... on the initial of gaussian vectors} directly,  with high probability
\begin{align}
    \frac{c}{m}\left<\va_0, \mB^{\T}\vx_{\alpha_1}^{[L]} \right>\left<\vx_{\alpha_2}^{[l]},\vx_\beta^{[l]}  \right>\leq c\frac{{\LogLn}^C}{m}. \label{rewrite...final...}
\end{align}
consequently
\begin{align}
    \Norm{\fG_0^{(3)}\left(\cdot\right)}_{\infty}\leq C(3,L) \frac{{\LogLn}^C}{m},
\end{align}
then   for any $1\leq \alpha_1, \alpha_2, \alpha_3\leq n,$ with time $0\leq t\leq \sqrt{m}/{\LogLn}^{C'}$
\begin{align*}
    \Abs{\fG_t^{(3)}\left(\vx_{\alpha_1},\vx_{\alpha_2},\vx_{\alpha_3}\right)}& \leq  \Norm{\fG_0^{(3)}\left(\cdot\right)}_{\infty}+tC(4,L)\frac{{\LogLn}^C}{m}\\
    &\leq C(3,L)\frac{{\LogLn}^C}{m}+tC(4,L)\frac{{\LogLn}^C}{m}.
\end{align*}
Similarly, based on \eqref{eq for thm...neural tangent kernel .... higher order 2},  for time $0\leq t\leq \sqrt{m}/{\LogLn}^{C'}$
\begin{align*}
    \Abs{\partial_t\fG_t^{(2)}\left(\vx_{\alpha_1},\vx_{\alpha_2}\right)}&\leq \sup_{1\leq \beta \leq n} \Abs{\fG_t^{(3)}\left(\vx_{\alpha_1},\vx_{\alpha_2},\vx_{\beta}\right)}\sqrt{\frac{\sum_{\beta=1}^n \Abs{f_{\beta}(t)-y_{\beta}}^2}{n}},
\end{align*}
set $\vx_{\beta}=\vx_{\alpha_3}$ 
\begin{align}
    \Abs{\partial_t\fG_t^{(2)}\left(\vx_{\alpha_1},\vx_{\alpha_2}\right)}&\leq \left(C(3,L)\frac{{\LogLn}^C}{m}+tC(4,L)\frac{{\LogLn}^C}{m}\right) \sqrt{\frac{\sum_{\beta=1}^n \Abs{f_{\beta}(t)-y_{\beta}}^2}{n}}\nonumber\\
    &\leq \left(C(3,L) +tC(4,L)\right)\frac{{\LogLn}^C}{m}.\label{eq..D appendx, }
\end{align}
and \eqref{eq..D appendx, } finishes the proof of Theorem \ref{theorem...vary at 1/m}.
\end{proof}

\begin{proof}[Proof of Corollary \ref{corollary....for thm}]
Firstly, based on Proposition \ref{proposition......eigenvalue of order higher}, if $m=\Omega\left( \left(\frac{n}{\lambda_0}\right)^{2+\eps}\right)$, we have with high probability w.r.t random initialization, 
\begin{equation*}
    \lambda_{\min}\left[\fK_0^{(2)}\left(\vx_{\alpha},\vx_{\beta}\right)\right]_{1\leq \alpha,\beta\leq n}>
    \lambda_{\min}\left(\mG^{[L+1]}(0)\right)>\frac{3 \lambda_0}{4},
\end{equation*}
set $\lambda=\frac{3 \lambda_0}{4},$ which finishes the proof of \eqref{least eigenvalue}.

We shall move on to the change of the least eigenvalue of the NTK. Recall  \eqref{eq..D appendx, } in the  proof of Theorem \ref{theorem...vary at 1/m}, for time  $0\leq t\leq \sqrt{m}/{\LogLn}^{C'},$
\begin{align}
       \Abs{\partial_t\fG_t^{(2)}\left(\vx_{\alpha_1},\vx_{\alpha_2}\right)}&\leq \left(C(3,L) +tC(4,L)\right) \frac{{\LogLn}^C}{m},\nonumber
\end{align}
consequently 
\begin{align*}
    \Abs{\fG_t^{(2)}\left(\vx_{\alpha_1},\vx_{\alpha_2}\right)-\fG_0^{(2)}\left(\vx_{\alpha_1},\vx_{\alpha_2}\right)}&\leq  t\left(C(3,L) +tC(4,L)\right)\frac{{\LogLn}^C}{m}.
\end{align*}
The above inequality can be used to derive a bound of the change of the least eigenvalue of the $\fG_t^{(2)}(\cdot) $
\begin{align*}
    \Norm {\fG_t^{(2)}-\fG_0^{(2)}}_{2\to 2} &\leq  \Norm {\fG_t^{(2)}-\fG_0^{(2)}}_{\mathrm{F}} \leq n \Norm {\fG_t^{(2)}-\fG_0^{(2)}}_{\infty}\\
    &\leq  nt\left(C(3,L) +tC(4,L)\right) \frac{{\LogLn}^{C}}{m},
\end{align*}
we   set $t^{*}$ satisfying  
\begin{align*}
   nt^{*}\left(C(3,L) +t^*C(4,L)\right) \frac{{\LogLn}^C}{m}=\frac{\lambda}{2},
\end{align*}
 rewrite the equation above, we have 
\begin{align}
   C(4,L) (t^*)^2 + C(3,L)t^* &= \frac{\lambda m}{2{\LogLn}^C n},\label{quadratic.....0}
\end{align}
solve \eqref{quadratic.....0}, we obtain that 
\begin{align}
    t^*= \frac{-C(3,L)+\sqrt{\left(C(3,L)\right)^2+2 C(4,L)\frac{\lambda m}{{\LogLn}^{C}n}}}{2 C(4,L) },\label{quadratic...1}
\end{align}
  since we are in the regime of over-parametrization, for $m$ large enough, the following holds
\begin{align}
t^*&\geq \frac{1}{2}\sqrt{\frac{\frac{\lambda m}{{\LogLn}^{C}n}}{C(4,L)}}=\frac{1}{2}\sqrt{\frac{\lambda m}{C(4,L){\LogLn}^{C}n}}.\label{a lower bound}
\end{align}
Moreover 
\begin{align*}
    \lambda_{\min}\left[\fK_t^{(2)}\left(\vx_{\alpha},\vx_{\beta}\right)\right]_{1\leq \alpha,\beta\leq n}&\geq   \lambda_{\min}\left[\fG_t^{(2)}\left(\vx_{\alpha},\vx_{\beta}\right)\right]_{1\leq \alpha,\beta\leq n} \\
    &\geq \lambda_{\min}\left[\fG_0^{(2)}\left(\vx_{\alpha},\vx_{\beta}\right)\right]_{1\leq \alpha,\beta\leq n}- \Norm {\fG_t^{(2)}-\fG_0^{(2)}}_{2\to 2},
\end{align*}
then let $\Bar{t}:=\inf\left\{t: \lambda_{\min}\left[\fK_t^{(2)}\left(\vx_{\alpha},\vx_{\beta}\right)\right]_{1\leq \alpha,\beta\leq n}\geq \lambda/2 \right\},$ naturally  
\begin{align}
    t^*\leq \Bar{t},\label{cp}
\end{align}
using \eqref{eq for thm...neural tangent kernel .... order 2}, we have for any $0\leq t \leq \Bar{t},$
\begin{align}
    \partial_t\sum_{\alpha=1}^n\Norm{f_\alpha(t)-y_\alpha}_2^2&\leq\sum_{\alpha,\beta=1}^n-\frac{2}{n}K_t^{(2)}(\vx_\alpha,\vx_\beta)(f_\alpha(t)-y_\alpha)(f_\beta(t)-t_\beta)\\
    &\leq -\frac{\lambda}{ n}\sum_{\alpha=1}^n\Norm{f_\alpha(t)-y_\alpha}_2^2, \label{eq for proof of prop apriori loss...decay of empirical risk..with eigenvalue}
\end{align}
then  
\begin{equation}\label{solution}
    \sum_{\alpha=1}^n\Norm{f_\alpha(t)-y_\alpha}_2^2\leq \exp\left(-\frac{\lambda t}{n}\right)\sum_{\alpha=1}^n\Norm{f_\alpha(0)-y_\alpha}_2^2,
\end{equation}
we can  rewrite \eqref{solution} into
\begin{align}
    R_S(\vtheta_t)\leq \exp\left(-\frac{\lambda t}{n}\right) R_S(\vtheta_0)
\end{align}
set  $R_S(\vtheta_t)=\eps$, it takes time $t\leq \frac{n}{\lambda} \ln (\frac{C'}{\eps})$ for loss  $R_S(\vtheta_t)$ to reach accuracy $\eps,$ hence if the following holds 
\begin{equation}\label{condition}
     t\leq \frac{n}{\lambda} \ln \left(\frac{C'}{\eps}\right)\leq t^*\leq \Bar{t},
\end{equation}
then the width $m$ is required to yield the lower bound for $t^*$ derived in \eqref{a lower bound},
\begin{equation}\label{widtrh2}
    \frac{n}{\lambda}\ln \left(\frac{C'}{\eps}\right)\leq  \frac{1}{2}\sqrt{\frac{\lambda m}{C(4,L){\LogLn}^{C}n}}.
\end{equation}
then we have 
\begin{align*}
 m\geq C(4,L)\left(\frac{n}{\lambda}\right)^3\LogLn^{C}\ln\left(\frac{C'}{\eps}\right)^2,
\end{align*}
since $C(4,L)=\fO\left(L^2\right),$ we conclude that the required width $m$ should be 
\begin{equation}\label{final eq.....}
    m=\Omega\left(\left(\frac{n}{\lambda}\right)^3L^2\LogLn^{C}\ln\left(\frac{C'}{\eps}\right)\right),
\end{equation}
where $\eps$ is the desired training accuracy.
\end{proof}

\end{document}